\documentclass[11pt,twoside]{article}

\usepackage{fullpage}
\usepackage{epsf}
\usepackage{fancyhdr}
\usepackage{graphics}
\usepackage{graphicx}
\usepackage{psfrag}
\usepackage{caption}
\usepackage{subcaption}
\usepackage{color}
\usepackage{amssymb}
\usepackage{amsthm}
\usepackage{amsfonts}
\usepackage{amsmath}
\usepackage{amssymb}
\usepackage{hyperref}
\usepackage{algorithm}
\usepackage{algorithmic}
\usepackage{bbm}
\usepackage{tikz}
\usepackage{caption}

\theoremstyle{definition}

\newtheorem{theos}{Theorem}
\newtheorem{props}{Proposition}
\newtheorem{lems}{Lemma}
\newtheorem{cors}{Corollary}
\newtheorem{defn}{Definition}

\newlength{\widebarargwidth}
\newlength{\widebarargheight}
\newlength{\widebarargdepth}

\makeatletter
\long\def\@makecaption#1#2{
        \vskip 0.8ex
        \setbox\@tempboxa\hbox{\small {\bf #1:} #2}
        \parindent 1.5em  
        \dimen0=\hsize
        \advance\dimen0 by -3em
        \ifdim \wd\@tempboxa >\dimen0
                \hbox to \hsize{
                        \parindent 0em
                        \hfil
                        \parbox{\dimen0}{\def\baselinestretch{0.96}\small
                                {\bf #1.} #2
                                }
                        \hfil}
        \else \hbox to \hsize{\hfil \box\@tempboxa \hfil}
        \fi
        }
\makeatother

\long\def\comment#1{}

\newcommand{\inprod}[2]{\ensuremath{\langle #1, \, #2 \rangle}}

\newcommand{\numobs}{\ensuremath{n}}
\newcommand{\usedim}{\ensuremath{d}}

\newcommand{\Yspace}{\ensuremath{\mathcal{Y}}}

\newcommand{\fone}{\ensuremath{g}}                     

\newcommand{\ftwo}{\ensuremath{h}}                     

\newcommand{\nonsmoothf}{\ensuremath{\varphi}}                 

\newcommand{\range}{\ensuremath{\mathcal{C}}}
\newcommand{\convmajor}{\ensuremath{q}}

\newcommand{\grad}[1]{\ensuremath{\nabla #1}}
\newcommand{\subgrad}[1]{\ensuremath{\partial #1 }}
\newcommand{\M}{\ensuremath{M}}
\newcommand{\sparsity}{s}

\newcommand{\genericconvex}{\ensuremath{p}}
\newcommand{\gradftwo}{\ensuremath{u^k}}
\newcommand{\x}[1]{\ensuremath{x^{#1}}}

\newcommand{\gradg}[1]{\ensuremath{v^{#1}} }

\newcommand{\curvatureconst}[1]{\ensuremath{\mathcal{C}_{#1}}}

\newcommand{\C }[1]{\ensuremath{C_{#1} }}
\newcommand{\fwgap}[1]{\ensuremath{g^{#1}}}

\newcommand{\minfwgap}[1]{\ensuremath{\bar{g}^{#1}}}

\newcommand{\y}[1]{\ensuremath{y_{#1}}}

\newcommand{\frechet }[1]{\ensuremath{\widehat{\partial}{#1}}}

\newcommand{\rate }{\ensuremath{r}} \newcommand{\f}{\ensuremath{f}}

\newcommand{\reanaone}{\ensuremath{h}} 
\newcommand{\reanatwo}{\ensuremath{g}} 
\newcommand{\nbhd}{\ensuremath{V}} 

\newcommand{\genericset}{\ensuremath{S}}
\newcommand{\myconvexset}{\ensuremath{\mathcal{X}}}
\newcommand{\indicator}[1]{\ensuremath{\mathbbm{1}_{#1}}}
\newcommand{\xNormalojGrad}[2]{\ensuremath{\grad f_{#2}( {#1} ) }}

\renewcommand{\j}{\ensuremath{j}}

\newcommand{\CompactSet}{S}

\newcommand{\avg}[1]{\ensuremath{\textup{Avg}\left( #1 \right)}}

\newcommand{\Gavg}[1]{\ensuremath{\textup{GAvg}\left( #1 \right)}}

\newcommand{\limiting}[1]{\ensuremath{\partial_{L}#1 }}

\newcommand{\deltak}[2]{\ensuremath{\Delta^{#1}_{#2}}}

\newcommand{\Bigo}[1]{\ensuremath{\mathcal{O}\left( #1 \right) }}
\newcommand{\LinModMat}{\ensuremath{B}}

\newcommand{\intr}{\operatorname{int}}
\newcommand{\mydefn}{\ensuremath{: \, =}}
\newcommand{\fstar}{\ensuremath{f^*}}
\newcommand{\real}{\ensuremath{\mathbb{R}}}
\newcommand{\Level}{\ensuremath{\mathcal{L}}}
\newcommand{\order}{\ensuremath{\mathcal{O}}}
\newcommand{\Xspace}{\ensuremath{\mathcal{X}}}

\newcommand{\dom}{\ensuremath{\operatorname{dom}}}
\newcommand{\diam}{\ensuremath{\operatorname{diam}}}
\newcommand{\dist}{\ensuremath{\operatorname{dist}}}
\newcommand{\sgn}{\ensuremath{\operatorname{sign}}}

\newcommand{\graph}{\ensuremath{\operatorname{graph}}}
\newcommand{\xstar}{\ensuremath{x^{*}}}

\newcommand{\Intensity}[1]{\ensuremath{I}_{#1}}
\newcommand{\xNormal}[1]{\ensuremath{p_{#1}}}
\newcommand{\yNormal}[1]{\ensuremath{q_{#1}}}
\newcommand{\Dir}{\ensuremath{L}}
\newcommand{\vertex}[1]{\ensuremath{v_{#1}}}
\newcommand{\NORMAL}[1]{\ensuremath{N_{#1}}}
\newcommand{\xCoord}{\ensuremath{x}}
\newcommand{\yCoord}{\ensuremath{y}}
\newcommand{\zCoord}{\ensuremath{z}}
\newcommand{\Poly}{\ensuremath{P}}

\newcommand{\step}[1]{\ensuremath{t}^{#1}}
\newcommand{\imRow}{\ensuremath{r}}
\newcommand{\imCol}{\ensuremath{c}}
\newcommand{\mixFun}{\ensuremath{\zeta}}

\newenvironment{carlist}
 {\begin{list}{$\bullet$}
 {\setlength{\topsep}{0in} \setlength{\partopsep}{0in}
  \setlength{\parsep}{0in} \setlength{\itemsep}{\parskip}
  \setlength{\leftmargin}{0.15in} \setlength{\rightmargin}{0.08in}
  \setlength{\listparindent}{0in} \setlength{\labelwidth}{0.08in}
  \setlength{\labelsep}{0.1in} \setlength{\itemindent}{0in}}}
 {\end{list}}

 \newcommand{\bcar}{\begin{carlist}}
\newcommand{\ecar}{\end{carlist}}

\usepackage{array}
\newcolumntype{P}[1]{>{\centering\arraybackslash}p{#1}}

\newcommand{\ROB}{\ensuremath{\Psi}}

\begin{document}


\begin{center}

{\bf{\LARGE{Convergence guarantees for a class of \\ non-convex and
      non-smooth optimization problems}}}

\vspace*{.2in}

{\large{
\begin{tabular}{ccc}
Koulik Khamaru $^{\dagger}$ & & Martin J. Wainwright
$^{\dagger,\star}$ \\
\end{tabular}
}}

\vspace*{.2in}

\begin{tabular}{c}
  Department of Electrical Engineering and Computer Sciences$^\star$,
  and \\ Department of Statistics$^{\dagger}$, \\ UC Berkeley,
  Berkeley, CA 94720
\end{tabular}

\vspace*{.2in}

\today

\vspace*{.2in}

\begin{abstract}
  We consider the problem of finding critical points of functions that
  are non-convex and non-smooth. Studying a fairly broad class of such
  problems, we analyze the behavior of three gradient-based methods
  (gradient descent, proximal update, and Frank-Wolfe update). For
  each of these methods, we establish rates of convergence for general
  problems, and also prove faster rates for continuous sub-analytic
  functions. We also show that our algorithms can escape strict saddle
  points for a class of non-smooth functions, thereby generalizing
  known results for smooth functions. Our analysis leads to a
  simplification of the popular CCCP algorithm, used for optimizing
  functions that can be written as a difference of two convex
  functions.  Our simplified algorithm retains all the convergence
  properties of CCCP, along with a significantly lower cost per
  iteration. We illustrate our methods and theory via applications to
  the problems of best subset selection, robust estimation, mixture
  density estimation, and shape-from-shading reconstruction.
\end{abstract}

\end{center}


\section{Introduction}
\label{SecIntro}

Non-convex optimization problems arise frequently in statistical
machine learning; examples include the use of non-convex penalties for
enforcing sparsity~\cite{FanLi91,loh2013regularized}, non-convexity in
likelihoods in mixture modeling~\cite{yan2017convergence}, and
non-convexity in neural network training~\cite{li2017convergence}.  Of
course, minimizing a non-convex problem is NP-hard in general, but
problems that arise in machine learning applications are not
constructed in an adversarial manner. Moreover, there have been a
number of recent papers demonstrating that all first (and/or second)
order critical points have desirable properties for certain
statistical problems (e.g. see the
papers~\cite{loh2013regularized,ge2017no}). Given results of this
type, it is often sufficient to find critical points that are
first-order (and possibly second-order) stationary. Accordingly,
recent years have witnessed an explosion of research on different
algorithms for non-convex problems, with the goal of trying to
characterize the nature of their fixed points, and their convergence
properties.

There is a lengthy literature on non-convex optimization, dating back
more than six decades, and rapidly evolving in the present (e.g., see
the books and papers~\cite{Tuy95,hartman-59,Horst00,
  Lanckreit09,Yuille03,Jason-16,bolte2014proximal,
  panageas2016gradient, lipp2016variations,
  cartis2010complexity,attouch2010proximal,gotoh2017dc}). Perhaps the
most straightforward approach to obtaining a first-order critical
point is via gradient descent. Under suitable regularity conditions
and step size choices, it can be shown that gradient descent can be
used to compute first-order critical points.  Moreover, with a random
initialization and additional regularity conditions, gradient descent
converges almost surely to a second-order stationary point
(e.g.,~\cite{Jason-16, panageas2016gradient}).  These results, like
much of the currently available theory for (sub)-gradient methods for
non-convex problems, involve smoothness conditions on the underlying
objectives.  In practice, many machine learning problems have
non-smooth components; examples include the hinge loss in support
vector machines, the rectified linear unit in neural networks, and
various types of matrix regularizers in collaborative filtering and
recommender systems. Accordingly, a natural goal is to develop
subgradient-based techniques that apply to a broader class of
non-convex functions, allowing for non-smoothness.

The main contribution of this paper is to provide precisely such a
set of techniques, along with non-asymptotic guarantees on their
convergence rates. In particular, we study algorithms that can be
used to obtain first-order (and in some cases, also second-order)
optimal solutions to a relatively broad class of non-convex functions,
allowing for non-smoothness in certain portions of the problem. For
each sequence $\{x^k\}_{k \geq 0}$ generated by one of our algorithms,
we provide non-asymptotic bounds on the convergence rate of the
gradient sequence $\{\|\grad \f(x^k)\|_2\}_{k \geq 0}$. Moreover, for
functions that satisfy a form of the Kurdaya-\L ojasiewicz inequality,
we show that our methods achieve faster rates.

Our work has important points of contact with a recent line of papers
on algorithms for non-convex and non-smooth problems, and we discuss a
few of them here. Bolte et al.~\cite{bolte2014proximal} developed a
proximal-type algorithm applicable to objective functions formed as a
sum of smooth (possibly non-convex) and a convex (possibly
non-differentiable) function. Some recent work~\cite{xu2017globally}
extended these ideas and provided analysis for block co-ordinate
descent methods for non-convex functions. In other recent work, Hong
et al.~\cite{hong2016convergence} provided analysis of ADMM method for
non-convex problems. In few recent
papers~\cite{an2017convergence,wen2018proximal} the authors proposed a
proximal-type method for non-convex functions which can be written as
a sum of a smooth function, a concave continuous function and a convex
lower semi-continuous function; we also analyze this class in one of
our results (Theorem~\ref{ThmProx}).

Our results also relate to another interesting sub-area of non-convex
optimization, namely functions that can be represented as a difference
of two convex functions, popularly known as DC functions.  We refer
the reader to the papers~\cite{Tuy95,hartman-59,Lanckreit09,Yuille03}
for more details on DC functions and their properties. One of the most
popular DC optimization algorithms is the Convex Concave Procedure, or
CCCP for short; see the papers~\cite{Yuille03,lipp2016variations} for
further details. This is a double loop algorithm that minimizes a
convex relaxation of the non-convex objective function at each
iteration. While the CCCP algorithm has some attractive convergence
properties~\cite{Lanckreit09}, it can be slow in many situations due
to its double loop structure. One outcome of the analysis in this
paper is a single-loop proximal-method that retains all the
convergence guarantees of CCCP while---as shown in our experimental
results---being much faster to run.


\subsection*{Overview of our results}
\bcar
  \item Our first main result (Theorem~\ref{ThmGradient}) provides
    guarantees for a subgradient algorithm as applied to the
    minimization problem~\eqref{ProbGradient} defined over a closed
    convex set $\range$.  We provide convergence bounds in terms of
    the Euclidean norm of the subgradient and show that our rates are
    unimprovable in general.  We also illustrate some consequences of
    Theorem~\ref{ThmGradient} by deriving a convergence rate for our
    algorithm when applied to non-smooth coercive functions; this
    result has interesting implications for polynomial programming. We
    also provide a simplification of the CCCP algorithm, along with
    convergence guarantees.  In Corollary~\ref{CorSaddlePoint}, we
    argue that our algorithm can escape strict saddle points for a
    large class of non-smooth functions, thereby generalizing known
    results for smooth functions.
  \item Our second main result (Theorem~\ref{ThmProx}) provides
    convergence rates for a proximal-type algorithm for
    problem~\eqref{EqnGenProb} (see below).  In
    Section~\ref{SecSmoothSparse}, we demonstrate how this
    proximal-type algorithm can be used to minimize a smooth convex
    function subject to a sparsity constraint. We demonstrate the
    performance of this algorithm through the example of best subset
    selection.

  \item In Theorem~\ref{ThmFW}, we provide a Frank-Wolfe type
    algorithm for solving optimization problem~\eqref{ProbFW}, and we
    provide a rate of convergence in terms of the associated
    \mbox{Frank-Wolfe gap}.
  \item Finally, in Theorems~\ref{ThmGradSubanal}
    and~\ref{ThmProxSubanal}, we prove that
    Algorithms~\ref{AlgoGradient} and~\ref{AlgoProx}, when applied to
    functions that satisfy a variant of the Kurdaya-\L ojasiewicz
    inequality, have faster convergence rates.  In particular, the
    convergence rate in terms of gradient norm is at least
    $\order(1/k)$ -- whereas the worst case rate for general
    non-convex functions is $\order(\frac{1}{\sqrt{k}})$. We also
    provide examples of functions for which the convergence rate is
    \mbox{$\order(1/k^r)$ with $r > 1$}. In Theorem 6, we characterize
    the class of functions that can be written as a difference of a
    smooth function and a differentiable convex function.
\ecar

\vspace*{0.25cm}

\noindent Section~\ref{SecApplication} is devoted to an illustration
of our methods and theory via applications to the problems of best
subset selection, robust estimation, mixture density estimation and 
shape-from-shading reconstruction.


\paragraph{Notation:}
Given a set $\range \subset \real^{\usedim}$, we use $\intr(\range)$
to denote its interior. We use $\|x\|_2$ , $\|x\|_1$ and $\|x\|_0$ to
denote the Euclidean norm, $\ell_1$-norm and $\ell_0$ norm
respectively, of a vector $x \in \real^\usedim$. We say that a
continuously differentiable function $\fone$ is $\M_{\fone}$-smooth if
the gradient $\grad \fone$ is $\M_{\fone}$-Lipschitz continuous.  In
many examples considered in this paper, the objective function $f$ is
a linear combination of a differentiable function $\fone$ and one or
more convex functions $\ftwo$ and $\nonsmoothf$.  With a slight abuse
of notation, for a function $\f = \fone - \ftwo + \nonsmoothf$, we
refer to a vector of the form $\nabla \fone(x) - u(x) + v(x)$, where
$u(x) \in \partial \ftwo(x)$ and $v(x) \in \partial \nonsmoothf(x)$,
as a gradient of the function $f$ at point $x$ --- and we denote it by
$\grad \f(x)$; here, $\partial \ftwo(\cdot)$ and $\partial
\nonsmoothf(\cdot)$ denote the subgradient sets of the convex
functions $\ftwo$ and $\nonsmoothf$ respectively.  We say a point $x$
is a \textit{critical} point of the function $\f$ if $0 \in \grad
\f(x)$.  For a sequence $\big\{ a^{k} \big\}_{k \geq 0}$, we define
the running arithmetic mean $\avg{a^{k}}$ as $\avg{a^{k}} \mydefn
\frac{1}{k}\sum_{\ell = 0}^{\ell = k+1} a^{\ell}$.  Similarly, for a
non-negative sequence $\big\{ a^{k} \big\}_{k \geq 0}$, we use
$\Gavg{a^{k}} \mydefn ( \prod \limits_{\ell = 0}^{k}
a^{\ell})^{\frac{1}{k+1}}$ to denote the running geometric
mean. Finally, for real-valued sequences $\{a^k\}_{k \geq 0}$ and
$\{b^k\}$, we say $a^k = \Bigo{b^k}$, if there exists a positive
constant $C$, which is independent of $k$, such that $a^k \leq C b^k$
for all $k \geq 0$.  We say $a^k = \Omega(b^k)$ if $a^k = \Bigo{b^k}$
and $b^k = \Bigo{a^k}$.


\section{Problem setup}

In this paper, we study the problem of minimizing a non-convex and
possibly non-smooth function over a closed convex set.  More
precisely, we consider optimization problems of the form
\begin{align}
 \label{EqnGenProb}
 \min_{x \in \range} \Big \{ \underbrace{\fone(x) - \ftwo(x) +
   \nonsmoothf(x)}_{f(x)} \Big \},
\end{align}
where the domain $\range$ is a closed convex set. In all cases, we
assume the function $\f$ is bounded below over domain $\range$, and that the
function $\ftwo$ is continuous and convex. Our aim is to derive
algorithms for problem~\eqref{EqnGenProb} for various types of
functions $\fone$ and $\nonsmoothf$.\\

\vspace*{.2in}

\paragraph{Structural assumption on functions $\fone~\text{and}~ \ftwo$}
\label{StructutalAssumption}

\begin{itemize}
\item[(a)] Theorems~\ref{ThmGradient} and~\ref{ThmGradSubanal} are
  based on the assumption that the function $\fone$ is continuously
  differentiable and smooth, and that the function $\nonsmoothf \equiv
  0$.
\item[(b)] In Theorems~\ref{ThmProx}~and~\ref{ThmProxSubanal}, we
  assume that the function $\fone$ is continuously differentiable and
  smooth, and that the function $\nonsmoothf$ is convex, proper and
  lower semi-continuous.\footnote{Taking the function 
  $\nonsmoothf \equiv 0$ yields
    part (a) as a special case, but it is worthwhile to point out that
    the assumptions in Theorem~\ref{ThmGradient} are weaker 
    than the assumptions of
    Theorem~\ref{ThmProx}. Furthermore, we can prove some interesting
    results about saddle points when the function $\nonsmoothf \equiv 0$; see
    Corollary~\ref{CorSaddlePoint}.}
  \item[(c)] Theorem~\ref{ThmFW} focuses on the case in which the
    function $\fone$ is continuously differentiable, and the function
    $\nonsmoothf \equiv 0$.
\end{itemize}

The class of non-convex functions covered in part (a) includes, as a
special case, the class of differences of convex (DC) functions, for
which the first convex function is smooth and the second convex
function is continuous.  Note that we only put a mild assumption of
continuity on the convex function $\ftwo$, meaning that the difference
function $\fone - \ftwo$ can be non-smooth and non-differentiable in
general. In particular, for any continuously differentiable function
$\ftwo$ and any smooth function $\fone$, the difference function
$\f = \fone - \ftwo$ is non-smooth. 
Furthermore, if we take the function $\ftwo \equiv 0$,
then we recover the class of smooth functions as a special case.


\section{Main results}
\label{SecMainResults}

Our main results are analyses of three algorithms for this class of
non-convex non-smooth problems; in particular, we derive
non-asymptotic bounds on their rates of convergence. The first
algorithm is a (sub)-gradient-type method, and it is mainly suited for
unconstrained optimization; the second algorithm is based on a
proximal operator and can be applied to constrained optimization
problems. The third algorithm is a Frank-Wolfe-type algorithm, which
is also suitable for constrained optimization problems, but it applies
to a more general class of non-convex optimization problems.


\subsection{Gradient-type method}
\label{SecGradMethod}

In this section, we analyze a (sub)-gradient-based method for solving
a certain class of non-convex optimization problems. In particular,
consider a pair of functions $(\fone, \ftwo)$ such that: \\

\noindent {\bf{Assumption GR:}}
\begin{itemize}
\item[(a)] The function $\fone$ is continuously differentiable and
  \mbox{$\M_{\fone}$-smooth.}
\item[(b)] The function $\ftwo$ is continuous and convex.
\item[(c)] There is a closed convex set $\range$ such that the difference
  function $f \mydefn \fone -\ftwo$ is bounded below on the set $\range$.
\end{itemize}
Under these conditions, we then analyze the behavior of a
(sub)-gradient method in application to the following problem
\begin{align}
\label{ProbGradient}
\fstar = \min_{x \in \range} f(x) \; = \; \min_{x \in \range} \big \{
\fone(x) - \ftwo(x) \big \}.
\end{align}
With a slight abuse of notation, we refer to a vector of the form
$\nabla \fone(x) - u(x)$ with $u(x) \in \partial \ftwo(x)$ --- where
$\partial \ftwo(x)$ denote the subgradient 
set of the convex function $\ftwo$ at the point $x$ --- as a
gradient of the function $f$ at the point $x$.

\begin{algorithm}
\caption{ $\;\;\;\;\;$ Subgradient-type method}\label{AlgoGradient}
\begin{algorithmic}[1]
\STATE{Given an initial point $\x{0} \in \intr(\range)$ and step size
  \mbox{$\alpha \in (0, \frac{1}{\M_{\fone}}]$:} }
\FOR{ $k = 0, 1, 2, \ldots$ }
\STATE{Choose subgradient $\gradftwo \in \subgrad\ftwo(\x{k})$.}
\STATE{Update $x^{k+1}=x^{k}-\alpha \big( \grad{\fone{(\x{k})}} -
  \gradftwo \big)$.}
\ENDFOR
\end{algorithmic}
\end{algorithm}
\noindent In our analysis, we assume that the initial vector $\x{0}
\in \intr(\range)$ is chosen such that the associated level set
\begin{align*}
\Level(\f(\x{0})) \mydefn \big \{ x\in \real^d \mid \f(x) \leq
\f(\x{0}) \big \}
\end{align*}
is contained within $\intr(\range)$. This condition is standard in
the analysis of non-convex optimization methods~(e.g., see Nesterov and
Polyak~\cite{nesterov2006cubic}). When $\range =
\real^{\usedim}$, it holds trivially.
With this set-up, we have the following guarantees on the convergence
rate of Algorithm~\ref{AlgoGradient}.
\begin{theos}
\label{ThmGradient}
Under Assumption GR, any sequence $\{x^k \}_{k \geq 0}$ produced by
Algorithm~\ref{AlgoGradient} has the following properties:
\begin{enumerate}
\item[(a)] Any limit point is a critical point of the function $f$, and the
  sequence of function values $\{f(\x{k}) \}_{k \geq 0}$ is strictly
  decreasing and convergent.
\item[(b)] For all $k = 0, 1, 2, \ldots$, we have
  \begin{align}
    \label{EqnAlgoGradientGradBound}
\avg{ \| \grad \f(\x{k}) \|_2^2 } & \leq \frac{2 \big(\f(\x{0})- \fstar
  \big)}{\alpha (k+1)}.
  \end{align}
\end{enumerate}
\end{theos}
\noindent See Appendix~\ref{AppThmGradient} for a proof of this
theorem.


\subsubsection{Comments on convergence rates }

Note that the bound~\eqref{EqnAlgoGradientGradBound} guarantees that
the gradient norm sequence $\min_{j \leq k} \|\nabla f(\x{j})\|_2$
converges to zero at the rate $\order(1/\sqrt{k})$. It is natural to
wonder whether this convergence rate can be improved. Interestingly,
the answer is no, at least for the general class of functions covered
by Theorem~\ref{ThmGradient}.  Indeed, note that the class of
$M$-smooth functions is contained within the class of functions
covered by Theorem~\ref{ThmGradient}. It follows from past work by
Cartis et al.~\cite{cartis2010complexity} that for gradient descent on
$\M{}$-smooth functions, with a step size chosen according to the
Goldstein-Armijo rule, the convergence rate of the gradient sequence
$\{ \|\nabla f(x^k)\|_2 \}_{k \geq 0}$ can be lower bounded---for
appropriate choices of the function $f$---as $\Omega(1/\sqrt{k})$.  It
is not very difficult to see that the same construction also provides
a lower bound of $\Omega(1/\sqrt{k})$ for gradient descent with a
constant step size. We also note that very recently, Carmon et
al.~\cite{carmon2017lower1} proved an even stronger result---namely,
for the class of smooth functions, the rate of convergence of any
algorithm given access to only the function gradients and function
values cannot be better than $\Omega(1/\sqrt{k})$.  Finally, observe
that in the special case $\ftwo \equiv 0$,
Algorithm~\ref{AlgoGradient} reduces to the ordinary gradient descent
with fixed step size $\alpha$. Putting together the pieces, we
conclude that for the class of functions which can be written as a
difference of smooth and a continuous convex function, Algorithm 1 is
\emph{optimal} among all algorithms which has access to the function
gradients (and/or the sub-gradient) and the function values.


\subsection{Consequences for differentiable functions}

In the special case when the function $\ftwo$ is convex and
differentiable, \mbox{Algorithm~\ref{AlgoGradient}} reduces to an
ordinary gradient descent on the difference function $\f = \fone -
\ftwo$. However, note that the step size choice required in
Algorithm~\ref{AlgoGradient} does \emph{not} depend on the smoothness
of the function $\ftwo$; consequently, the algorithm can be applied to
objective functions $f$ that are not smooth. As a simple but concrete
example, suppose that we wish to apply gradient descent to minimize
the function \mbox{$f(x) \mydefn g(x) - \|x\|_2^q$,} where $g$ is any
\mbox{$\mu$-strongly} convex and \mbox{$M_{\fone}$-smooth} function,
and $q \in (1,2)$ is a given parameter.  Classical guarantees on
gradient descent, which require the smoothness of the function $f$,
would not apply here since the function $f$ itself is not
smooth. However, Theorem~\ref{ThmGradient} guarantees that standard
gradient descent would converge for any step size $\alpha \in \big( 0,
\frac{1}{M_\fone} \big]$.

More generally, given an arbitrary continuously differentiable
function $f$, we can define its \emph{effective smoothness constant} as
\begin{align}
\label{EqnSmoothConst}
\M^*_{\f} \mydefn \inf_{h} \big \{ L \; \mid \; \mbox{$(\f +
  \ftwo)$ is $L$-smooth} \big \},
\end{align}
where the infimum ranges over all convex and continuously
differentiable functions $h$. Suppose that this infimum is achieved
by some function $h^*$, then gradient descent on the function $f$ can
be viewed as applying Algorithm~\ref{AlgoGradient} to
the decomposition $f = g^* - h^*$, where the function
$g^* \mydefn f + h^*$ is
guaranteed to be $\M^*_{\f}$-smooth.  To be clear, the algorithm
itself does \emph{not} need to know the decomposition $(g^*, h^*)$,
but the existence of the decomposition ensures the success of a
backtracking procedure.  Putting together the pieces, we arrive at the
following consequence of Theorem~\ref{ThmGradient}:

\begin{cors}
\label{CorBacktracking}
Given a closed convex set $\range$, consider a continuously
differentiable function $\f$ with effective smoothness $\M_{\f}^* <
\infty$ that is bounded below on $\range$.  Then for any sequence
$\{\x{k}\}_{k \geq 0}$ obtained by applying the gradient update with
step size $\alpha \in \big(0, \frac{1}{\M^*_\f} \big)$, we have:
\begin{subequations}
        \begin{align}
        \label{EqnFastGradBound}
          \avg{ \|\grad \f(\x{k}) \|_2^2 } & \leq
          \frac{2\big(f(\x{0})- \fstar \big) }{\alpha (k+1)}.
        \end{align}
Moreover, if we choose step size by backtracking\footnote{A detailed
  description of gradient descent with backtracking is provided in
  Algorithm~\ref{AlgoBacktracking}.} with parameter $\beta \in (0,1)$,
then for all \mbox{$k = 0,1,2,\ldots$,} we have
        \begin{align}
        \label{EqnBacktrackingBdd}
            \avg{ \|\grad \f(\x{k}) \|_2^2 } & \leq
            \frac{2\max\big\{ 1, \M^*_{\f} \big\}
              \big(f(\x{0})- \fstar \big) }{\beta^2 (k+1)}.
        \end{align}
\end{subequations}
\end{cors}
\noindent See Appendix~\ref{AppCorBacktracking} for proof of the above
corollary.\\

Let us reiterate that the advantage of backtracking gradient descent
is that it works without knowledge of the scalar $\M^*_{\f}$.  The
parameter $\beta$ mentioned in equation~\eqref{EqnBacktrackingBdd} is
the backtracking parameter and is a user defined fraction in the
backtracking method (see Algorithm~\ref{AlgoBacktracking} for
details). In particular, substituting $\beta = \frac{1}{\sqrt{2}}$ in
equation~\eqref{EqnBacktrackingBdd} yields
\begin{align*}
\avg{ \|\grad \f(\x{k}) \|_2^2 } \leq \frac{4\max\big\{ 1,
  \M^*_{\f} \big\} \big(f(\x{0}) - \fstar \big)}{(k+1)},
\end{align*}
which differs from the rate obtained in
equation~\eqref{EqnFastGradBound} only by a factor of two, and a
possible multiple of $\M^*_f$.


\subsubsection{Consequences for coercive functions}

As a consequence of Corollary~\ref{CorBacktracking}, we can obtain a rate
of convergence of the backtracking gradient descent algorithm
(Algorithm~\ref{AlgoBacktracking}) for a class of non-smooth coercive
functions. Consider any twice continuously differentiable coercive
function $\f : \real^{\usedim} \mapsto \real$, which is bounded below.
Recall that a function $\f$ is \emph{coercive} if
\begin{align}
  \label{EqnDefnCoercive}
  \f(x^\ell) \rightarrow \infty \quad \mbox{for any sequence
    $\{x^\ell\}_{\ell \geq 0}$ such that} \quad \|x^\ell\|_2
  \rightarrow \infty.
\end{align}
Let ${\Level(\f(\x{0})) \mydefn \big\{ x \in \real^{\usedim} : \f(x)
  \leq \f(\x{0}) \big\}}$ denote the level set of the function $\f$ at
point $\x{0}$.  It can be verified that for any coercive function
$\f$, the set $\Level(\f(\x{0}))$ is bounded above for all $\x{0} \in
\real^{\usedim}$. This property ensures that for any descent algorithm
and any starting point $\x{0}$, the set of iterates $\big\{ \x{k}
\big\}_{k \geq 0}$ obtained from the algorithm remains within a
bounded set---viz. the level set $\Level(\f(\x{0}))$ in this
case. Since the function $\f$ is twice continuously differentiable, we
have that $\f$ is smooth over bounded set $\Level(\f(\x{0}))$; this
fact ensures that $\f$ has a finite effective smoothness constant in
the set $\Level(\f(\x{0}))$, which we denote by
$\M^*_{\f,\x{0}}$. Finally, note that Algorithm~\ref{AlgoBacktracking}
is a descent algorithm; as a result, a simple application of
Corollary~\ref{CorBacktracking} yields the following rate of
convergence for the backtracking gradient descent algorithm
(Algorithm~\ref{AlgoBacktracking}):
\begin{cors}
\label{CorCoerciveGradRate}
Consider the unconstrained minimization problem of a twice
continuously differentiable coercive function $\f$ that is bounded
below on $\real^{\usedim}$.  Then for any initial point $\x{0}$, the
sequence $\{\x{k}\}_{k \geq 0}$ obtained by applying
Algorithm~\ref{AlgoBacktracking} satisfies the following property:
\begin{align}
\label{EqnCoerciveGradRate}
\avg{ \|\grad \f(\x{k}) \|_2^2 } & \leq \frac{2\max\big\{ 1,
  \M^*_{\f,\x{0}} \big\} \big(f(\x{0})- \fstar \big) }{\beta^2 (k+1)}
\quad \text{for all $k = 0,1,2,\ldots$},
\end{align}
where $\beta \in (0, 1)$ is the backtracking parameter.
\end{cors}
%


\paragraph{Implications for polynomial programming:}

Corollary~\ref{CorCoerciveGradRate} has useful implications for
problems that involve minimizing polynomials.  Such problems of
polynomial programming arise in various applications, including phase
retrieval and shape-from-shading~\cite{wang2014efficient}, and we
illustrate our algorithms for the latter application in
Section~\ref{SecSimShading}.  For minimization of a coercive
polynomial, Corollary~\ref{CorCoerciveGradRate} shows that
Algorithm~\ref{AlgoBacktracking} achieves a near-optimal rate.

It is worth noting that any even degree polynomial can be represented
as a difference of convex (DC) function; hence, such problems are
amenable to DC optimization techniques like CCCP, which we discuss at
more length in Section~\ref{SecConnectCCCP}.  However, obtaining a
good DC decomposition, which is crucial to the success of CCCP, is
often a formidable task. In particular, obtaining an optimal
decomposition for a polynomial with degree greater than four is
NP-hard ---the main reason for this phenomenon is that deciding the
convexity of an even degree polynomial with degree greater than four
is NP-hard~\cite{ahmadi2013np,wang2014efficient}.  Even for a fourth
degree polynomial with dimension larger than three, there is no known
algorithm for finding an optimal DC
decomposition~\cite{ahmadi2013complete}.  An advantage of
Algorithm~\ref{AlgoBacktracking} is that it obviates the need to find
a DC decomposition.


\subsubsection{Escaping strict saddle points}

One of the obstacles with gradient-based continuous optimization
method is possible convergence to saddle points.  Here we show that
with a random initialization this undesirable outcome does not occur
for the class of strict saddle points.  Recall that for a twice
differentiable function $\f$, a point $x$ is called a strict saddle
point of the function $\f$ if ${\lambda_{\min}(\nabla^2 \f(x)) < 0}$,
where $\lambda_{\min}(\nabla^2 \f(x))$ denotes the minimum eigenvalue
of the Hessian matrix $\nabla^2 \f(x)$.  The following corollary shows
that such saddle points are \emph{not} troublesome:
\begin{cors}
  \label{CorSaddlePoint}
Suppose that, in addition to the conditions on $(\fone, \ftwo,
\range)$ from Theorem~\ref{ThmGradient}, the functions $(\fone,
\ftwo)$ are twice continuously differentiable.  If
Algorithm~\ref{AlgoGradient} is applied with step size $\alpha \in
\big(0 , \frac{1}{\M_{\fone}} \big)$, then the set of initial points
for which it converges to a strict saddle point has measure zero.
\end{cors}
\noindent See Appendix~\ref{AppCorSaddlePoint} for the proof of this
corollary. \\

\noindent We note that similar guarantees of avoidance of strict
saddlepoints are known when the function $f = \fone - \ftwo$ is twice
continuously differentiable and $\M$-smooth (e.g.,~\cite{Jason-16,
  panageas2016gradient}).  The novelty of
Corollary~\ref{CorSaddlePoint} is that the same guarantee holds
without imposing a smoothness condition on the entire function $f$.


\subsection{Connections to the convex-concave procedure}
\label{SecConnectCCCP}

As a consequence of Algorithm~\ref{AlgoGradient}, we show that 
one can obtain a convergence rate of the Euclidean norm of 
the gradient for CCCP (convex-concave procedure),
which is a heavily used \mbox{algorithm}
in Difference of Convex (DC) optimization problems. Before doing so,
let us provide a brief description of DC functions and the CCCP
algorithm.

\paragraph{DC functions:} 
Given a convex set 
$\range \subseteq\real^{\usedim}$,
we say that a function $\f: \range \mapsto\real$ is DC if there exist
convex functions $\fone$ and $\ftwo$ with domain $\range$ such that
$f= \fone -\ftwo$.  Note that the DC representation $f = \fone -\ftwo$
mentioned in the definition is not unique. In particular, 
for any convex
function $\genericconvex$, we can write $f = ( \fone +
\genericconvex)-( \ftwo + \genericconvex )$. The class of DC functions
includes a large number of non-convex problems encountered in
practice.  Both convex and concave functions are DC in a trivial
sense, and the class of DC functions remains closed under addition and
subtraction. More interestingly, under mild restrictions on the
domain, the class of non-zero DC functions is also closed under
multiplication, division, and composition (see the
papers~\cite{Tuy95,hartman-59}).  The maximum and minimum of a finite
collection of DC functions are also DC functions.

%
\paragraph{Convex-concave procedure:}
An interesting class of problems are those that involve minimizing a
DC function over a closed convex set $\range \subseteq
\real^{\usedim}$, i.e.
\begin{align}
  \label{ProbDC}
\fstar \mydefn \min_{x \in \range} f(x) \; = \; \min_{x \in \range}
\big \{ \fone(x) - \ftwo(x) \big \},
\end{align}
where $\fone$ and $\ftwo$ are proper convex functions. The above
problem has been studied intensively, and there are various methods
for solving it; for instance, see the
papers~\cite{Tuy95,lipp2016variations,Pham-13} and references therein
for details.  One of the most popular algorithms to solve
problem~\eqref{ProbDC} is the \mbox{Convex-concave Procedure} (CCCP),
which was introduced by Yuille and Rangarajan~\cite{Yuille03}.  The
CCCP algorithm is a special case of a Majorization-Minimization
algorithm, which uses the DC structure of the objective function in
problem~\eqref{ProbDC} to construct a convex majorant of the objective
function $\f$ at each step.  We start with a feasible point $x^0 \in
\intr(\range)$.  Let $\x{k}$ denote the iterate at $k^{\text{th}}$
iteration; at the $(k+1)^{\text{th}}$ iteration we construct a convex
majorant $\convmajor(\cdot, \; x^k)$ of the function $f$ via
\begin{align}
\label{EqnCCCPMajor}
f(x) & \leq \; \underbrace{\fone(x) - \ftwo(x^k) -
  \inprod{\gradftwo}{x-x^k}}_{ = : \; \convmajor(x,x^k)},
 \end{align}
where $\gradftwo \in \subgrad \ftwo (x^k)$, 
the subgradient set of the convex function $\ftwo$ at point 
$\x{k}$. The next iterate $x^{k+1}$
is obtained by solving the convex program
 \begin{align}
 \label{EqnCCCP}
 x^{k+1} \in \arg \min_{x \in \range} \convmajor(x,x^k).
\end{align}

The CCCP algorithm has some attractive convergence properties. For
instance, it is a descent algorithm; when the function $\fone$ is
strongly convex differentiable and the function $\ftwo$ is
continuously differentiable, it can be shown~\cite{Lanckreit09} that
any limit point of the sequence $\big \{ x^k \big \}_{k\geq 0}$
obtained from CCCP is stationary.  Under the same assumptions, one can
also verify that \mbox{$ \lim_{k \rightarrow \infty } \| x^k - x^{k+1}
  \|_2 = 0$.}

We now turn to an analysis of CCCP using the techniques that underlie
Theorem~\ref{ThmGradient}. In the next proposition, we derive a rate
of convergence of the gradient sequence and show that all limit
points of the sequence $\big \{ \x{k} \big \}_{ k \geq 0 }$ are
stationary. Earlier analyses of CCCP, including the
papers~\cite{Lanckreit09,Yuille03}, are mainly based on the
assumption of strong convexity of the function $\fone$, whereas in the
next proposition, we only assume that the function $\fone$ is
$\M_{\fone}$-smooth.  When the function $\fone$ is strongly convex,
our analysis recovers the well-known convergence result in past
work~\cite{Lanckreit09}.  In particular, we show that CCCP enjoys the
same rate of convergence as that of Algorithm~\ref{AlgoGradient}.

\begin{props}
\label{PropCCCP}
Under Assumption GR and with the function $\fone$ being convex, the CCCP
sequence~\eqref{EqnCCCP} has the following properties:
\begin{enumerate}
  \item[(a)] Any limit point of the sequence $\big \{ \x{k} \big
    \}_{k \geq 0}$ is a critical point, and the sequence of
    function values $\big \{ \f(\x{k}) \big \}_{ k \geq 0 }$
    is strictly decreasing and convergent.
\item[(b)] Furthermore, for all $k = 1, 2, \ldots$, we have
  \begin{subequations}
\begin{align}
  \label{EqnCCCPBasic}
  \avg{ \|\grad \f(\x{k}) \|_2^2 } & \leq \frac{2 \M_{\fone} \big(
    \f(\x{0})- \fstar \big)}{(k+1)},
\end{align}
and assuming moreover that $\fone$ is $\mu$-strongly convex,
\begin{align}
  \label{EqnCCCPStrong}
 \avg{\| \x{k} - \x{k+1} \|_2^2} & \leq \frac{2 \big( \f(\x{0})-
   \fstar \big)}{\mu(k+1)}.
\end{align}
\end{subequations}
\end{enumerate}
\end{props}
\noindent The proof of this proposition builds on the argument used
for Theorem~\ref{ThmGradient}; see Appendix~\ref{AppPropCCCP} for
details.


\subsubsection{Simplifying CCCP}

Algorithm~\ref{AlgoGradient} provides us an alternative procedure for
minimizing a difference of convex functions when the first convex
function is smooth.  The benefit of Algorithm~\ref{AlgoGradient} over
standard CCCP is that Algorithm~\ref{AlgoGradient} is a single loop
algorithm and is expected to be faster than standard double loop CCCP
algorithm in many situations. Furthermore, Algorithm~\ref{AlgoGradient}
shares convergence guarantees similar to a standard CCCP algorithm.


\subsection{Proximal-type method}

We now turn to a more general class of optimization problems of the form
\begin{align}
\label{EqnProbProx}
\fstar \mydefn \min_{x \in \real^\usedim} f(x) \; = \; \min_{x \in
  \real^\usedim} \Big \{ \big( \fone(x) - \ftwo(x) \big) +
\nonsmoothf(x) \Big \}.
\end{align}
We assume that the functions $\fone, \ftwo$ and $\nonsmoothf$ satisfy
the following conditions:
\paragraph{Assumption PR}
\begin{itemize}
  \item[(a)] The function $f = \fone - \ftwo + \nonsmoothf$ is bounded
    below on $\real^{\usedim}$.
  \item[(b)] The function $\fone$ is continuously differentiable and
    $\M_{\fone}$-smooth; the function $\ftwo$ is continuous and
    convex; and the function $\nonsmoothf$ is proper, convex and
    lower semi-continuous.

\end{itemize}
Typical examples of the function $\nonsmoothf$ include $\nonsmoothf(x)
= \|x\|_1$, or the indicator of a closed convex \mbox{convex set
  $\myconvexset$}.  Since for a general lower semi-continuous function
$\nonsmoothf$, the sum-function $\fone + \nonsmoothf$ is 
neither differentiable nor smooth,
a gradient-based method cannot be applied. 
One way to minimize such functions is via a proximal-type
algorithm, of which the following is an instance.


\begin{algorithm}[H]
\caption{  $\;\;\;\;\;$  Proximal-type algorithm }
\label{AlgoProx}
\begin{algorithmic}[1]
\STATE{Given an initial vector $\x{0} \in \dom(\f)$ and step size
  $\alpha \in \big(0, \frac{1}{\M_{\fone}} \big]$.}
\FOR{$k = 0,1,2, \ldots$}
\STATE{Update $\x{k+1} = \text{prox}_{ 1/ \alpha }^{\nonsmoothf} \Big(
  \x{k} - \alpha \big( \grad{\fone(\x{k})} - \gradftwo \big) \Big)$
  for some $\gradftwo \in \subgrad{\ftwo(\x{k})}$. } \ENDFOR
\end{algorithmic}
\end{algorithm}

\noindent The proximal update in line 3 of Algorithm~\ref{AlgoProx}~
is very easy to compute and often has a closed form solution
\mbox{(see Parikh and Boyd~\cite{parikh2014proximal})}.  Let us now
derive the rate of convergence result of Algorithm~\ref{AlgoProx}.


\begin{theos}
\label{ThmProx}
Under Assumption PR, any sequence $\big \{ \x{k} \big \}_{ k
  \geq 0 }$ obtained from Algorithm~\ref{AlgoProx} has the following
properties:
\begin{enumerate}
  \item[(a)] Any limit point of the sequence $\big \{ \x{k} \big
    \}_{k \geq 0}$ is a critical point, and the sequence of
    function values $\big \{ \f(\x{k}) \big \}_{ k \geq 0 }$
    is strictly decreasing and convergent.
 \item[(b)] For all $k = 1,2,\ldots$, we have
\begin{subequations}
\begin{align}
\label{EqnAlgoProxDiffBound}
  \avg{ \| \x{k} - \x{k-1} \|_2^2 } & \leq \frac{2 \alpha \big(
    f(\x{0})- \fstar \big) }{ (k+1)}.
\end{align}
If moreover the function $\ftwo$ is $\M_{\ftwo}$-smooth, then
\begin{align}
\label{EqnAlgoProxGradBound}
   \avg{ \| \grad \f(\x{k}) \|_2^2 } \leq \frac{2 \alpha C_{\M,\alpha}
     \big(f(\x{0})- \fstar \big)}{ (k+1)},
 \end{align}
 where $C_{\M,\alpha} = \big( \M_{\fone} + \M_{\ftwo} +
 \frac{1}{\alpha} \big)^2$.
\end{subequations}
\end{enumerate}
\end{theos}
\noindent See Appendix~\ref{AppThmProx} for the proof of the theorem.


\paragraph{Comments:}
The proof of Theorem~\ref{ThmProx} reveals that the smoothness
condition on the function $\ftwo$ in Theorem~\ref{ThmProx} can be
replaced by the local smoothness of $\ftwo$, when the sequence $\big
\{ \x{k} \big \}_{k \geq 0}$ is bounded.  Note that the local
smoothness condition is weaker than the global smoothness condition.
For instance, any twice
continuously differentiable function is locally smooth.  The
boundedness assumption on the iterates $\big \{ \x{k} \big
\}_{k \geq 0}$ holds in many situations.  For instance, if the
function $\f$ is coercive~\eqref{EqnDefnCoercive}, then it follows
that the iterates $\big\{ \x{k} \big\}_{k \geq 0}$ remain
bounded. Another instance is when the function $\nonsmoothf$ is the
indicator function of a compact convex set. 
Finally, we point out that when
the function $\ftwo$ is non-smooth but the proximal-function $\nonsmoothf$ is smooth,
the existing proof can be easily modified to obtain a rate of convergence of the gradient-norm $\| \grad \f(\x{k}) \|_2$.


\paragraph{Projected Gradient Descent:}
A special case of the Algorithm~\ref{AlgoProx} is when $\nonsmoothf$
is equal to the indicator function $\indicator{\myconvexset}$ of a
closed convex set $\myconvexset$. Consider the following constrained
optimization problem
\begin{align}\label{EqnProjGradProb}
 \fstar \mydefn \min_{x \in \myconvexset} \;\; \big\{
 \underbrace{\fone(x) - \ftwo(x)}_{f(x)} \big\},
\end{align}
where $\myconvexset$ is a closed convex set, the function $\fone$ is
$\M_{\fone} $-smooth, and the function $\ftwo$ is convex
continuous. Using Algorithm~\ref{AlgoProx}, the update equation in
this case is given by
\begin{align}
\label{EqnProjGrad}  
\x{k+1} = \Pi_{\myconvexset} \big( \x{k} - \alpha (
\grad{\fone}(\x{k}) - u^k ) \big).
\end{align}
In projected-gradient-type methods, we should not expect a rate in
terms of the gradient. In such cases, the projected gradient step may not be
aligned with the gradient direction, or the step size may be
arbitrarily small due to projection. Rather, an appropriate analogue
of the gradient in this case is as follows:
\begin{align}
\label{EqnGrad_proxy}
\xNormalojGrad{\x{k}}{\myconvexset} = \frac{1}{\alpha} \big( \x{k} -
\Pi_{\myconvexset}(\x{k} - \alpha( \grad{\fone{(\x{k})}} - u^k ) )
\big).
\end{align}
The analysis of the projected gradient method using
$\xNormalojGrad{\x{k}}{\myconvexset}$ is standard in the optimization
\mbox{literature}~\cite{bubeck2015convex}.  It is worth pointing out
that the quantity $\xNormalojGrad{\x{k}}{\myconvexset}$ is the
analogue of the gradient in the constrained optimization setup, and
coincides with the gradient in the unconstrained setup. Concretely, we
have $\xNormalojGrad{\x{k}}{\myconvexset} = \grad{\f}(\x{k})$ where
$\f \mydefn \fone - \ftwo$, and $\Xspace = \real^{\usedim}$. Combining
equations~\eqref{EqnProjGrad} and~\eqref{EqnGrad_proxy} and applying
the bound~\eqref{EqnAlgoProxGradBound} from Theorem~\ref{ThmProx}, we
find that
\begin{align*}
  \avg{ \| \xNormalojGrad{\x{k}}{\myconvexset} \|_2^2} \leq \frac{ 2
    \big(f(\x{0})-\fstar\big)}{ \alpha (k+1)}.
\end{align*}


\subsection{Frank-Wolfe type method}
\label{SecFW}

In our analysis of the previous two algorithms, we assumed that the
objective function $\f$ has a smooth component $\fone$, and we
leveraged the smoothness property of $\fone$ to establish convergence
rates.  In many situations, the objective function may not have a
smooth component; consequently, neither the gradient-type algorithm
nor the prox-type algorithm provides any theoretical guarantee.  In
this section, we analyze a Frank-Wolfe-type algorithm for solving such
optimization problems.  In particular, consider an optimization
problem of the form
\begin{equation}
\label{ProbFW}
 \fstar \mydefn \min_{x \in \range} f(x) \; = \min_{x \in \range} \big
 \{ \fone(x) - \ftwo(x) \big \},
\end{equation}
where $\range$ is a closed convex set, and the functions $(\fone,
\ftwo)$ satisfy the following conditions:
\paragraph{Assumption FW:}
\begin{itemize}
\item[(a)] The difference function $\f = \fone - \ftwo$ is bounded
  below over range $\range$.
\item[(b)] The function $\fone$ is continuously differentiable,
  whereas the function $\ftwo$ is convex and continuous.
\end{itemize}

\vspace{8pt}

The analysis of the Frank-Wolfe algorithm for a convex problem is
based on the \emph{curvature constant} $\curvatureconst{f}$ of the
convex objective function with respect to the closed convex set
$\range$.  This curvature constant can be defined for any
differentiable function, which need not be convex~\cite{Julien16}.

Here we define a slight generalization of this notion, applicable to a
non-differentiable function $\f = \fone - \ftwo$ that can be written
as a difference of a differentiable function $\fone$ and a continuous
convex function $\ftwo$ (which may be non-differentiable).  Define the
set
\begin{align*}
S_{\gamma} \mydefn \big\{ x,y\in \range \mid \mbox{there exist
  $\gamma \in (0,1]$ and $u \in \range$ with $y = x + \gamma
(u-x)$} \big\},
\end{align*}
and the curvature constant
\begin{align}
  \label{EqnCurvConst}
\curvatureconst{ \f} = \sup_{ \substack{ x,y \in S_{\gamma} \\ u \in
    \subgrad{\ftwo(x)} } }\frac{2}{\gamma^{2}}\big[ f(y)-f(x)-
  \inprod{y-x}{\grad{\fone(x)} - u} \big].
\end{align}
Note that in the special case $\ftwo \equiv 0$, we recover the
curvature constant of the differentiable function $\fone$ used by
Lacoste-Julien~\cite{Julien16}.  We refer to the scalar
$\curvatureconst{\f}$ as the generalized curvature constant of the
function $\f$ with respect to the closed convex set $\range$.


\begin{algorithm}[]
\caption{ $\;\;\;\;\;$ Frank-Wolfe type method}~\label{AlgoFW}
\begin{algorithmic}[1]
  \STATE{Given initial vector $\x{0} \in \int(\range)$:}
  \FOR{ $k = 1,\ldots, K$}
  \STATE{Choose any $\gradftwo \in \subgrad{\ftwo}(\x{k})$.}
  \STATE{ Compute $s^{k} \mydefn \arg \min_{s \in \range}
    \inprod{s}{\grad{\fone}(\x{k}) - \gradftwo}$.}
\STATE{ Define $d^{k} \mydefn s^{k}-\x{k}$ and $\fwgap{k} \mydefn -
  \inprod{d^{k}}{\grad{\fone}(\x{k}) - \gradftwo}$.} \qquad
\emph{(Frank-Wolfe gap)}
\STATE{ Set $\gamma^{k}=\min \big \{ \frac{\fwgap{k}}{C_0}, 1 \big\}$
  for some $C_0 \geq \curvatureconst{f}$.}
\STATE{Update $\x{k+1} = \x{k} + \gamma^{k} d^{k}$.} \ENDFOR
\end{algorithmic}
\end{algorithm}


Next, we provide an analysis of Algorithm~\ref{AlgoFW} in terms of the
Frank-Wolfe (FW) gap~$\fwgap{k}$ defined Step 5.  We show that the minimum
FW gap~ $\{ \fwgap{k} \}_{k \geq 0}$ defined in Algorithm~\ref{AlgoFW}
converges to zero at the rate $\frac{1}{\sqrt{k+1}}$.

\begin{theos}
\label{ThmFW}
Under Assumption FW, the Frank-Wolfe gap sequence $\{\fwgap{k}\}_{ k
  \geq 0 }$ from Algorithm~\ref{AlgoFW} satisfies the following
property:
\begin{align*}
\min_{0\leq \j \leq k} \fwgap{\j} & \leq \frac{\max\big\{ 2 \big(
  f(\x{0})- \fstar \big), C_0 \big\} }{\sqrt{k+1}} \qquad \mbox{for
  all $k = 0, 1, 2, \ldots$.}
\end{align*}
\end{theos}
\noindent See Appendix~\ref{AppThmFW} for the proof of this theorem.

\vspace{8pt}
\paragraph{Comments:}
The FW gap appearing in Theorem~\ref{ThmFW} is standard in the
analysis of Frank-Wolfe algorithm; note that it is invariant to an
affine transformation of the set $\range$.  Similar convergence
guarantees for the minimum FW-gap are available for differentiable
functions; for instance, see the paper~\cite{Julien16}. The novelty of
the above theorem is that it provides convergence guarantees of
minimum FW-gap for a class of non-differentiable functions.
\paragraph{Upper bound on generalized curvature constant:}
It is worth mentioning that Algorithm~\ref{AlgoFW} only requires 
an upper bound of the generalized curvature constant 
$\curvatureconst{\fone - \ftwo}$. Consequently,
it is interesting 
to obtain an upper bound for the scalar
$\curvatureconst{\fone - \ftwo}$. 
For a  $\M_{\fone}$-smooth function $\fone$, 
one  well-known upper bound of
the curvature constant $\curvatureconst{\fone}$
is  $\M_{\fone} \times 
\big(\diam_{\| \cdot \|_2}(\range) \big)^{2}$; see the works by 
Jaggi~\cite{jaggi-11}.  
A similar upper bound also holds for
the generalized curvature constant defined in equation~\eqref{EqnDefnCurvConst}.
In particular, we prove that for a difference function $\f = \fone - \ftwo$, 
with the function $\ftwo$ being convex continuous, the scalar
$\curvatureconst{\fone - \ftwo}$ is 
always upper bounded by $\curvatureconst{\fone}$, the 
curvature constant of the \mbox{function
$\fone$ (see Lemma~\ref{LemCurvatureUB})}.

\section{Faster rate under KL-inequality}
\label{SecBetterRateSubanal}

In the preceding sections, we have derived rates of convergence for
the gradient norms for various classes of problems. It is natural to
wonder if faster convergence rates are possible when the objective
function is equipped with some additional structure.  Based on
Theorems~\ref{ThmGradient} and~\ref{ThmProx}, we see that both
Algorithms~\ref{AlgoGradient} and~\ref{AlgoProx} ensure that $\|\x{k}
- \x{k+1} \|_2 \rightarrow 0$, meaning that the successive differences
between the iterates converge to zero.  Although we proved that any
limit point of the sequence $\{ \x{k} \}_{k \geq 0}$ has desirable
properties, the condition $\|\x{k} - \x{k+1} \|_2 \rightarrow 0$ is
not sufficient---at least in general---to prove
convergence\footnote{The convergence of the sequence $\big\{ \x{k}
  \big\}_{k \geq 0}$ for Algorithm~\ref{AlgoProx} was studied in the
  papers~\cite{an2017convergence,wen2018proximal}.  We provide the
  proof under a weaker set of assumptions.  } of the sequence $\{
\x{k} \}_{k \geq 0}$. In this section, we provide a sufficient
condition under which Algorithm~\ref{AlgoGradient} and
Algorithm~\ref{AlgoProx} yield convergent sequences of iterates $\{
\x{k} \}_{ k \geq 0 }$, and we establish that the gradient sequences
$\{\| \grad \f(x) \|_2\}_{k \geq 0}$ converge at faster rates.

\subsection{Kurdaya-\L ojasiewicz inequality}

Let us now establish a faster local rate of convergence of
Algorithms~\ref{AlgoGradient} and~\ref{AlgoProx} for functions that
satisfy a form of the Kurdaya-\L ojasiewicz (KL) inequality.  More
precisely, suppose that there exists a constant $\theta \in [0,1)$
  such that the ratio $\frac{(\f(x) - \f(\bar{x}))^\theta}{\| \grad
    \f(x) \|_2}$ is bounded above in a neighborhood of every point
  $\bar{x} \in \dom(\f)$. This type of inequality is known as a
  Kurdaya-\L ojasiewicz inequality, and the exponent $\theta$ is known
  as the Kurdaya-\L ojasiewicz exponent (\emph{KL-exponent}) of the
  function $\f$ at the point $\bar{x}$. These type of inequalities
  were first proved by \L ojasiewicz~~\cite{Lojasiewicz1963propriete}
  for real analytic functions; Kurdaya~\cite{kurdyka1998gradients} and
  Bolte et al.~\cite{bolte2007Lojasiewicz} proved similar inequalities
  for non-smooth functions, and the authors
  also provided examples of many functions that
  satisfy a form of the KL inequality.  See
  Appendix~\ref{AppSubAnalKL} for further details on functions of the
  KL type.

  \paragraph{Assumption KL:}  For
  any point \footnote{It can be shown that such an inequality would
    hold at non-critical point of a continuous function $\f$; see
    Remark 3.2 of Bolte et al.~\cite{bolte2007Lojasiewicz}. 
    Note that the parameter
    $\theta$ and the neighborhood mentioned in Assumption KL above may
 depend on the \mbox{point $\bar{x}$}.}  $\bar{x}
  \in \dom(\f)$, there exists a scalar $\theta \in [0,1)$ such that
    the ratio $\frac{|\f(x) - \f(\bar{x})|^\theta}{\| \grad \f(x)
      \|_2}$ is bounded above in a neighborhood of $\bar{x}$.\\



\subsection{Convergence guarantees}

\begin{theos}
\label{ThmGradSubanal}
Under Assumptions GR \& KL, any bounded sequence $\{ \x{k} \}_{k \geq
  0}$ obtained from Algorithm~\ref{AlgoGradient} satisfy the following
properties:
\begin{itemize}
  \item[(a)] The sequence $\{ \x{k} \}_{k \geq 0}$ converges to a
    critical point $\bar{x}$, and for all $k = 1, 2, \ldots$
  \begin{align*}
  \avg{ \| \grad{f(\x{k})} \|_2 } \leq \frac{c_1}{k},
  \end{align*}
\item[(b)] Suppose that at the point $\bar{x}$, the function $f$ has a
  KL exponent $\bar{\theta} \in \big[ \frac{1}{2}, \frac{r}{2 r
      -1}\big)$ for some $r > 1$.  Then we have
\begin{align*}
   \Gavg{ \| \grad{f(\x{k})} \|_2} \leq  \frac{c_2}{k^r}  \qquad
    \mbox{for all} \quad k = 1,2,\ldots,
 \end{align*} 
where the constants $(c_1, c_2)$ are independent of $k$, but they may
depend on the KL parameters at the point $\bar{x}$.
\end{itemize}
\end{theos}
\noindent See Appendix~\ref{AppThmGradSubanal} for proof of this theorem.
\paragraph{Comments:} It is worth pointing out that
Theorem~\ref{ThmGradSubanal} does \emph{not} require the function
$\ftwo$ to satisfy any
smoothness assumption.  Such conditions are needed for applying
Algorithm~\ref{AlgoProx}, so that Theorem~\ref{ThmGradSubanal} is
based on milder conditions than Theorem~\ref{ThmProxSubanal}.\\


Our next result is to exhibit a faster convergence rate for
Algorithm~\ref{AlgoProx} under the KL assumption:
\begin{theos}
\label{ThmProxSubanal}
Suppose that, in addition to Assumptions PR \& KL, the function
$\ftwo$ in Algorithm~\ref{AlgoProx} is locally smooth.  Then any
bounded sequence $\{ \x{k} \}_{k \geq 0}$ obtained from
Algorithm~\ref{AlgoProx} satisfy the following properties:
 \begin{itemize}
  \item[(a)] The sequence $\{ \x{k} \}_{k \geq
    0}$ converges to a critical point $\bar{x}$, and for all $k =
    1,2\ldots$
  \begin{align*}
  \avg{ \| \grad{f(\x{k})} \|_2 } \leq \frac{c_1}{k}.
\end{align*}
\item[(b)] Given some $r > 1 $, suppose that at the point $\bar{x}$
the function  $\f$ 
 has a KL exponent $\bar{\theta} \in \big[ \frac{1}{2},
  \frac{r}{2 r -1} \big)$.  Then
\begin{align*}
   \Gavg{ \| \grad{f(\x{k})} \|_2} \leq \frac{c_2}{k^r} \qquad
   \mbox{for all} \quad k = 1,2,\ldots,
 \end{align*} 
where the constants $(c_1, c_2)$ are independent of $k$, but they may
depend on the KL parameters at the point $\bar{x}$.
\end{itemize}\end{theos}
\noindent See Appendix~\ref{AppThmProxSubanal} for the proof of this theorem. \\

\noindent \textbf{Comments:} Note that $\min_{1 \leq i \leq k} \|
\grad \f(\x{k}) \|_2$ is upper bounded by the quantities $\avg{\|\grad
  \f(\x{k})\|_2}$ and $\Gavg{\|\grad \f(\x{k})\|_2}$. It thus follows
that the sequence $\{ \| \grad{\f(\x{k})} \|_2 \}_{k \geq 0}$
converges to zero at a rate of at least $1/k$, thereby improving the
rate of convergence of $\|\grad \f(x)\|_2$ obtained in
Theorems~\ref{ThmGradient} and~\ref{ThmProx}. When $\theta <
\frac{1}{2}$, a simple modification of the proof (using $\gamma = 2$)
shows that, Algorithms~\ref{AlgoGradient} and~\ref{AlgoProx} converge
in a finite number of steps.
Finally, we point out that when
the function $\ftwo$ is non-smooth but the proximal-function 
$\nonsmoothf$ is smooth,
the existing proof can be easily modified to obtain a rate of convergence of the gradient-norm $\| \grad \f(\x{k}) \|_2$.
\section{Some illustrative applications}
\label{SecApplication}

In this section, we study four interesting classes of non-convex
problems that fall within the framework of this paper. We also discuss
various consequences of Theorems~\ref{ThmGradient}
--~\ref{ThmProxSubanal} as well as 
\mbox{Corollaries~\ref{CorBacktracking} --~\ref{CorSaddlePoint}} when
applied to these problems.


\subsection{Shape from shading}
\label{SecSimShading}
The problem of shape from shading is to reconstruct the
three-dimensional (3D) shape of an object based on observing a
two-dimensional (2D) image of intensities, along with some information
about the light source direction.  It is assumed that the observed 2D
image intensity is determined by the angle between the light source
direction and the surface normals of the
object~\cite{ecker2010polynomial}.

In more detail, suppose that both the object and its 2D image are
supported on a rectangular grid of size $\imRow \times \imCol$.  We
introduce the shorthand notation $[\imRow] = \{1, 2, \ldots, \imRow\}$
and $[\imCol] = \{1, 2, \ldots, \imCol\}$ for the rows and columns of
this grid.  For each pair $(i,j) \in [\imRow] \times [\imCol]$, we let
$\Intensity{ij} \in \real$ denote the observed intensity at location
$(i,j)$ in the image, and we let $\NORMAL{ij} \in \real^3$ denote the
surface normal at the vertex ${\vertex{ij} \mydefn
  (\xCoord_{ij},\yCoord_{ij},\zCoord_{ij})}$ of the object.  Based on
observing the 2-dimensional image, both the intensity $\Intensity{ij}$
and co-ordinate pair $(\xCoord_{ij},\yCoord_{ij})$ are known for each
pair $(i,j) \in [\imRow] \times [\imCol]$.  The goal of shape from
shading is to estimate the unknown coordinate $\zCoord_{ij}$, which
corresponds to the height of the object at location $(i,j)$.
Knowledge of these $z$-coordinates allows us to generate a 3D
representation of the object, as illustrated in Figure~\ref{FigSFS}.


\paragraph{Lambertian lighting model:}  

In order to reconstruct the $z$-coordinates, we require a model that
relates the observed intensity $\Intensity{ij}$ to the surface normal.
In a Lambertian model, for a given light source direction ${\Dir
  \mydefn (\ell_1, \ell_2,\ell_2)^\top} \in \real^3$, it is assumed
that the surface normal $\NORMAL{ij}$ and intensity $\Intensity{ij}$
are related via the relation
\begin{align}
\label{EqnLambertian}
    \Intensity{ij} = \frac{\inprod{\Dir}{\NORMAL{ij}}}{\| \NORMAL{ij}
      \|_2}.
\end{align}

\begin{figure}[h]
\makebox[\linewidth][c]{%
\begin{subfigure}{0.25\textwidth}\centering
\includegraphics[trim= 0cm 0cm 0cm 0cm, clip, width = \linewidth]
                {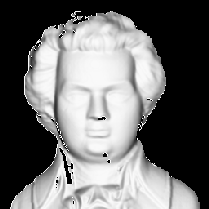} 
                \phantomcaption
\end{subfigure}
\hspace{1cm}
\begin{subfigure}{0.7\textwidth}\centering
\includegraphics[trim= 0cm 0cm 0cm 0cm, clip,width= \linewidth]
                {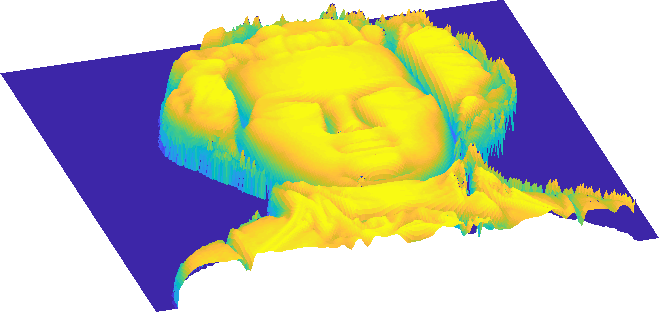} 
                \phantomcaption
\end{subfigure}
}

\makebox[\linewidth][c]{%
\begin{subfigure}{0.25\textwidth}\centering
\vspace*{1cm}
\includegraphics[trim= 0cm 0cm 0cm 0cm, clip,width= \linewidth]
                {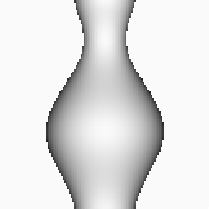} 
                \phantomcaption
\end{subfigure}
\hspace{1cm}
\begin{subfigure}{.7\textwidth}\centering
\vspace*{1cm} \includegraphics[trim= 0cm 0cm 0cm 0cm, clip,width=
  \linewidth] {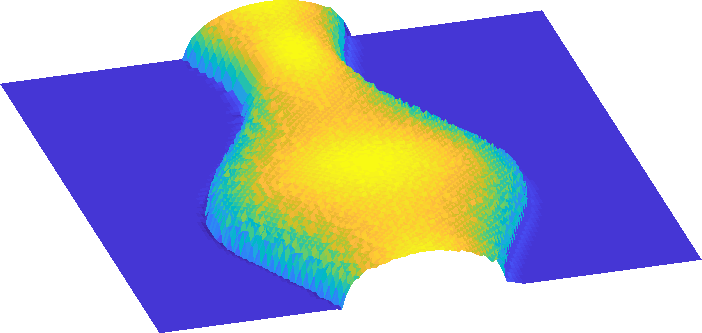} \phantomcaption
\end{subfigure}%
}
\vspace{1cm}
\caption{ Figure shows 3D shape reconstruction of \emph{Mozart} (first
  row) and \emph{Vase} (second row) from corresponding 2D images.  The
  gray-scale images in the left column are the 2D input images; the
  two colored images in the right column are the reconstructed 3D
  shapes.  The 3D shapes are constructed by solving the
  problem~\eqref{EqnSFSProb} using Algorithm~\ref{AlgoBacktracking}.  }
\label{FigSFS}
\end{figure}

In one standard model~\cite{wang2014efficient}, the surface normal
$\NORMAL{ij} \mydefn (\xNormal{ij}, \yNormal{ij}, 1)^\top$ is assumed
to be determiend by the triplet of vertices $(\vertex{ij},
\vertex{i+1,j}, \vertex{i,j+1})$ via the equations
\begin{align*}
  \xNormal{ij} & =  \frac{(\yCoord_{i,j+1} - \yCoord_{i,j})(\zCoord_{i+1,j} 
    - \zCoord_{ij}) - (\yCoord_{i+1,j} - \yCoord_{i,j})(\zCoord_{i,j+1} 
    - \zCoord_{ij})}{(\xCoord_{i,j+1} - \xCoord_{ij})(\yCoord_{i+1,j} - \yCoord_{ij}) -
    (\xCoord_{i+1,j} - \xCoord_{ij})(\yCoord_{i,j+1} - \yCoord_{ij})},  \\
  \yNormal{ij} & =  \frac{(\xCoord_{i,j+1} - \xCoord_{i,j})(\zCoord_{i+1,j} 
    - \zCoord_{ij}) - (\xCoord_{i+1,j} - \xCoord_{i,j})(\zCoord_{i,j+1} 
    - \zCoord_{ij})}{(\xCoord_{i,j+1} - \xCoord_{ij})(\yCoord_{i+1,j} - \yCoord_{ij}) -
    (\xCoord_{i+1,j} - \xCoord_{ij})(\yCoord_{i,j+1} - \yCoord_{ij})}.
\end{align*} 
Squaring both sides of equation~\eqref{EqnLambertian} and substituting
the expression for surface normal $\NORMAL{ij}$ yields the polynomial
equation
\begin{align}
    \big(\xNormal{ij}^2 + \yNormal{ij}^2 +1\big)\Intensity{ij} 
    - (\ell_1\xNormal{ij} + \ell_2\yNormal{ij} + \ell_3)^2 = 0,
\end{align}
which should be satisfied under the assumed model.

\noindent In practice, this equality will not be exactly satisfied, but we can
estimate the $z$-coordinates by solving the following non-convex
optimization problem in the $\imRow \times \imCol$ matrix $\zCoord$
 \mbox{with entries $\big\{ \zCoord_{ij} \mid (i,j) \in [\imRow] \times
[\imCol] \}$:}
\begin{align}
\label{EqnSFSProb}
    \underset{ \zCoord \in \real^{\imRow \times \imCol}}{\min} &
    \;\;\;\big\{ \underbrace{\sum \limits_{i = 1}^{\imRow} \sum
      \limits_{j = 1}^{\imCol} \big( (1 + \xNormal{ij}^2 +
      \yNormal{ij}^2)\Intensity{ij}^2 - (\ell_1\xNormal{ij} +
      \ell_2\yNormal{ij} + \ell_3)^2 \big)^2}_{\Poly(\zCoord)} \big\}.
\end{align}


\paragraph{Some reconstruction experiments:}

In order to illustrate the behavior of our method for this problem, we
considered two synthetic images for simulated experiments. The first
one is a $256 \times 256$ image of
\emph{Mozart}~\cite{zhang1999shape}, and the second one is a $128
\times 128$ image of \emph{Vase}.  The 3D shapes were constructed from
the 2D images by solving optimization problem~\eqref{EqnSFSProb} using
the backtracking gradient descent algorithm~\ref{AlgoBacktracking}.
The
reconstructed surfaces for \emph{Vase} and \emph{Mozart} are provided
in figure~\ref{FigSFS}.  We ran $500$ iterations of
Algorithm~\ref{AlgoBacktracking} for both the images. The runtime for
\emph{Mozart}-example was 87 seconds, whereas the runtime for
\emph{Vase}-example was 39 seconds. The implementation of
Algorithm~\ref{AlgoBacktracking} for Problem~\eqref{EqnSFSProb} is
parallelizable; hence, the runtime can be much lower than our
runtime with a parallel implementation.  It is worth mentioning that
the polynomial $\Poly$ is a fourth-degree polynomial with dimension
$\imRow \times \imCol$; polynomial $\Poly$ is coercive and bounded
below by zero. Consequently, we can apply
Corollary~\ref{CorCoerciveGradRate} to the
problem~\eqref{EqnSFSProb}
which guarantees that average of the squared gradient norm $\avg{\|
  \grad \Poly \|_2^2}$ converges to zero at a rate $\frac{1}{k}$.

One might also consider applying the CCCP method to this problem. In a
recent paper, Wang et al.~\cite{wang2014efficient} provided a DC
decomposition of the polynomial $\Poly$ using a sum of square (SOS)
optimization technique.  However, it is crucial to note that the DC
decomposition of polynomial $\Poly$ obtained from the SOS-optimization
method need not be optimal. In order to see this, note that the
dimension of the polynomial $\Poly$ is much larger than three.  In
particular, the variable $\zCoord_{ij}$ is used in the computation of
surface normals $\NORMAL{ij}, \NORMAL{i,j-1}$ and $\NORMAL{i-1,j}$,
hence is related to variables $(\zCoord_{i,j+1} , \zCoord_{i+1,j}
,\zCoord_{i-1,j}, \zCoord_{i,j-1})$ ---which are again related to the
other variables.  It was shown in the paper~\cite{ahmadi2013complete}
that SOS techniques for deriving a DC decomposition are
sub-optimal for a
fourth-degree polynomial when the dimension of the polynomial is
greater than 3.  Consequently, deriving an optimal DC decomposition
for the polynomial $\Poly$ will be computationally intensive.


\subsection{Robust regression using Tukey's bi-weight}
\label{SecTukey}

Next, we turn to the problem of robust regression with Tukey's
bi-weight penalty function. Suppose that we observe pairs $(\y{i}, z_i)
\in \real \times \real^d$ linked via the noisy linear model
\begin{align*}
   \y{i} = \inprod{z_{i}}{\mu^*} + \varepsilon_i \qquad \mbox{ for $i
     = 1,\ldots,n$.}
 \end{align*}
Here the vector $\mu^{*} \in \real^{\usedim}$ is the unknown parameter
of interest, whereas the variables $\{\varepsilon_i\}_{i=1}^\numobs$
correspond to additive noise.  In robust regression, we obtain an
estimate of the parameter vector $\mu^*$ by computing
\begin{align}
\label{EqnTukeyObj}  
\min_{\mu \in \real^{\usedim} } \underbrace{\big\{ \frac{1}{n}
  \sum_{i=1}^{n} \ROB \big( y_i - \inprod{z_i}{\mu} \big) \big\}}_{ =
  \,: \, f(\mu)}
\end{align}
where $\ROB$ is a known loss function with some robustness
properties. One popular example of the loss function $\ROB$ is Tukey's
bi-weight function, which is given by
\begin{align}
  \label{EqnTukeyFun}
\ROB(t) =
\begin{cases}
  1- (1- (t/\lambda) ^2)^3 & \mbox{ if $|t| \leq \lambda$} \\
  1 & \mbox{otherwise}
\end{cases},
 \end{align}
where $\lambda > 0$ is a tuning parameter. Note that $\ROB$ is a
smooth function, whence the function $f$ in the
objective~\eqref{EqnTukeyObj} is also smooth, implying that
Algorithm~\ref{AlgoGradient} is suitable for the problem. \\

\noindent With this set-up, applying Theorem~\ref{ThmGradient},
Theorem~\ref{ThmGradSubanal} and Corollary~\ref{CorSaddlePoint},
we obtain the following guarantee:
\begin{cors}
  \label{CorTukey}
Given a random initialization, any bounded sequence $\{ \mu^{k} \}_{k
  \geq 0}$ obtained by applying Algorithm~\ref{AlgoGradient} to the
objective~\eqref{EqnTukeyObj} has the following properties:
\begin{itemize}
  \item[(a)] Almost surely with respect to the random initialization,
    the sequence $\{ \mu^{k} \}_{k \geq 0}$ converges to a point
    $\bar{\mu}$ such that $\grad \f(\bar{\mu}) = 0$ and $\nabla^2
    \f(\bar{\mu}) \succeq 0$.
  \item[(b)] There is a universal constant $c_1$ such that
  \begin{align*}
       \avg{ \| \grad \f(\mu^{k}) \|_2 } \leq \frac{c_1}{k} \qquad
       \mbox{for all $k = 1, 2, \ldots$.}
  \end{align*}   
\end{itemize}
\end{cors}
\noindent We provide the proof in Appendix~\ref{AppCorTukey}.


\subsection{Smooth function minimization with sparsity constraints}
\label{SecSmoothSparse}

Moving beyond the robust regression problem, we now discuss another
interesting problem of minimizing a smooth function subject to
sparsity penalty. Consider the following optimization problem
\begin{align}
\label{EqnSmoothSparse}
\min_{ \substack{ x \in \real^{\usedim} \\ \| x \|_0 \leq s }
}\;\; \fone(x),
\end{align}
where $\fone$ is a smooth function, the $\ell_0$-``norm'' $\|x\|_0$
counts the number of non-zero entries in the vector $x$, 
and $s \in \{ 1,\ldots, d\}$ is a
  sparsity parameter.  The constraint set $\big \{ x \in
  \real^{\usedim} : \|x \|_0 \leq s \big \}$ is non-convex, and
  consequently, the optimization problem~\eqref{EqnSmoothSparse} is
  non-convex. However, the constraint set can be expressed as the
  level set of a certain DC function; see Gotoh et
  al.~\cite{gotoh2017dc}.  In particular, let $|x|_{(\usedim)} \geq
  |x|_{(\usedim-1)} \geq \cdots \geq |x|_{(1)}$ denote the values of
  $x \in \real^\usedim$ re-ordered in terms of their absolute
  magnitudes.  In terms of this notation, we have $\| x \|_1 \geq
  \sum_{i = \usedim - s+1 }^{ \usedim } |x|_{(i)}$ for all $x \in
  \real^{\usedim}$, with equality holding if and only if $x$ is
  $s$--sparse.  This fact ensures that
\begin{align}
  \label{EqnSparsityAsDC}
 \Big\{ x \in \real^{\usedim} : \|x \|_0 \leq s \Big\} =
 \big\{ x \in \real^{\usedim} : \| x \|_{1} - \sum_{i =
   \usedim - s+1 }^{ \usedim } |x|_{(i)} \leq 0 \big\}.
\end{align}
Since both of the functions $x \mapsto \| x \|_{1}$ and $x \mapsto
\sum_{i = \usedim - s+1 }^{\usedim} |x|_{(i)}$ are
convex~\cite{boyd2004convex}, this level set formulation is a DC
constraint.
Now using the representation~\eqref{EqnSparsityAsDC}, we can rewrite
problem~\eqref{EqnSmoothSparse} as $\min_{x \in \real^{\usedim}}
\fone(x)$ such that $\| x \|_{1} - \sum_{i = \usedim - s+1 }^{\usedim}
|x|_{(i)} \leq 0$.  For our experiments, it is more convenient to
solve the penalized analogue of the last problem, given by
\begin{align}
\label{EqnSmoothSparse2}
   \min_{x \in \real^{\usedim} } \;\; \big\{ \fone(x) +
   \lambda \Big( \| x \|_{1} - \sum_{i = \usedim-s+1 }^{\usedim}
   |x|_{(i)} \Big) \big\},
 \end{align}
where $\lambda > 0$ is a tuning parameter. The optimization
problem~\eqref{EqnSmoothSparse2} can be solved using
Algorithm~\ref{AlgoProx} with $\fone(x) = \fone(x)$, $\nonsmoothf(x) =
\lambda \| x \|_1$ and $\ftwo(x) = \lambda \sum_{i = \usedim-s+1
}^{\usedim} |x|_{(i)}$.  For the non-smooth component ${\nonsmoothf(x)
= \lambda \|x\|_1}$, there is a closed form expression of the proximal
update in Algorithm~\ref{AlgoProx}, so that the method is especially
efficient in this case.

\subsubsection{Best subset selection}

A special case of problem~\eqref{EqnSmoothSparse2} arises from best
subset selection in linear regression. Suppose that we observe a
vector $y \in \real^{\numobs}$ and a matrix $\LinModMat \in
\real^{\numobs \times \usedim }$ that are linked via the standard
linear model \mbox{$y = \LinModMat \xstar + \varepsilon$.}  Here the
vector $\varepsilon \in \real^\numobs$ corresponds to additive noise,
whereas $\xstar \in \real^\usedim$ is the unknown regression
vector. We wish to estimate the unknown parameter 
vector $\x{*}$ subject to a
sparsity constraint, and we do so by solving the following
optimization problem:
\begin{align}
\label{EqnBestSubset}
\min_{ \substack{ x \in \real^{\usedim} \\ \|x \|_0 \leq s } } \| y -
\LinModMat x \|_2^2.
\end{align}
Here the non-negative integer $s$ is a tuning parameter that controls
maximum number of allowable non-zero entries in the vector $x$.
Following the development leading to the
formulation~\eqref{EqnSmoothSparse2}, let us consider instead the
problem of minimizing
the function
\begin{align}
    \label{EqnLinRegDC}
  \f(x) & \mydefn \| y - \LinModMat x \|_2^2 + \lambda \Big( \| x
  \|_{1} - \sum_{i = \usedim-s+1 }^{\usedim} |x|_{(i)} \Big)
  \big\}.
\end{align}
Note that the function $\f$ can be decomposed as a difference of two convex
functions as follows:
\begin{align}
  \f(x) = \underbrace{\| y - \LinModMat x \|_2^2 + \lambda \| x
    \|_{1}}_{\mbox{convex}} - \underbrace{\lambda \sum_{i =
      \usedim-s+1 }^{\usedim} |x|_{(i)} }_{\mbox{convex}}.
   \label{EqnDCDecomp}
\end{align}
Consequently, problem~\eqref{EqnLinRegDC} is a DC optimization
problem; hence, it is amenable to standard DC optimization techniques
like CCCP.  We can also apply Algorithm~\ref{AlgoProx} on
problem~\eqref{EqnLinRegDC} with ${\fone(x) = \|y - \LinModMat x
  \|^2_{2}}$, $\nonsmoothf(x) = \lambda \| x \|_1$ and $\ftwo(x) =
\lambda \sum_{i = \usedim-s+1 }^{\usedim} |x|_{(i)}$.


\subsubsection{Comparison of Algorithm~\ref{AlgoProx} and CCCP}

Let us compare the performance of our Algorithm 2 (prox-type method)
with the popular convex-concave procedure (CCCP) for minimizing
differences of convex functions.  We apply both algorithms to the best
subset selection problem~\eqref{EqnLinRegDC}.

Let us reiterate that problem~\eqref{EqnLinRegDC} can be written as 
a difference of two convex functions, and one 
can apply CCCP update~\eqref{EqnCCCP} to the
decomposition~\eqref{EqnDCDecomp}. The inner convex optimization
problem in update~\eqref{EqnCCCP} is solved by proximal methods for
minimizing the sum of a smooth convex function and a
$\ell_1$ regularizer.
We also apply Algorithm~\ref{AlgoProx} on problem~\eqref{EqnLinRegDC}
with ${\fone(x) = \| y - \LinModMat x \|_2^2}$, $\ftwo(x) = \lambda
\sum_{i = \usedim-s+1 }^{\usedim} |x|_{(i)}$ and $\nonsmoothf(x) =
\lambda \| x \|_{1}$.

\paragraph{Synthetic data generation:} We generated the rows of the 
 $\numobs \times \usedim$ matrix $\LinModMat$ from a
$\usedim$-dimensional Gaussian distribution with zero mean and an
equicovariance matrix $\Sigma$, where $\Sigma_{ii} = 1$ for all $i$, and
$\Sigma_{ij} = 0.7$ for all $i \neq j$.  The regression vector $\xstar
\in \real^{\usedim}$ (true value) was chosen to be a binary vector
with sparsity $\sparsity$ ($\sparsity \ll \usedim$). The location of
the nonzero entries of the vector $\xstar$ was chosen uniformly
without replacement form the set $\big\{ 1,\ldots, \usedim \big\}$.

\paragraph{Performance measures:} We use the
following two criteria to 
compare the performance of the prox-type method and CCCP.

\begin{itemize}
  \item[(a)]\emph{Total runtime:} Firstly, we compare the algorithms
    in terms of their total runtime.  The runtime was measured in units
    of seconds.

  \item[(b)]\emph{Estimation error:} Secondly, we use average
    estimation error of the algorithms as a measure of
    performance. Let us recall that if 
    $\bar{x} \in \real^{\usedim}$ is the
    estimated value of the unknown regression vector $\xstar$,
    then the average
    estimation error is defined as $\frac{\|\bar{x} -
      \xstar\|_2}{\sqrt{p}\|\bar{x}\|_2}$. Note that the average
    estimation error used here is invariant under scaling.
\end{itemize}

\begin{figure}[p]
\makebox[\linewidth][c]{%
\begin{subfigure}{0.70\textwidth}\centering
\includegraphics[trim= 0cm 6.5cm 3.cm 5.5cm, clip,width= 1\linewidth]
                {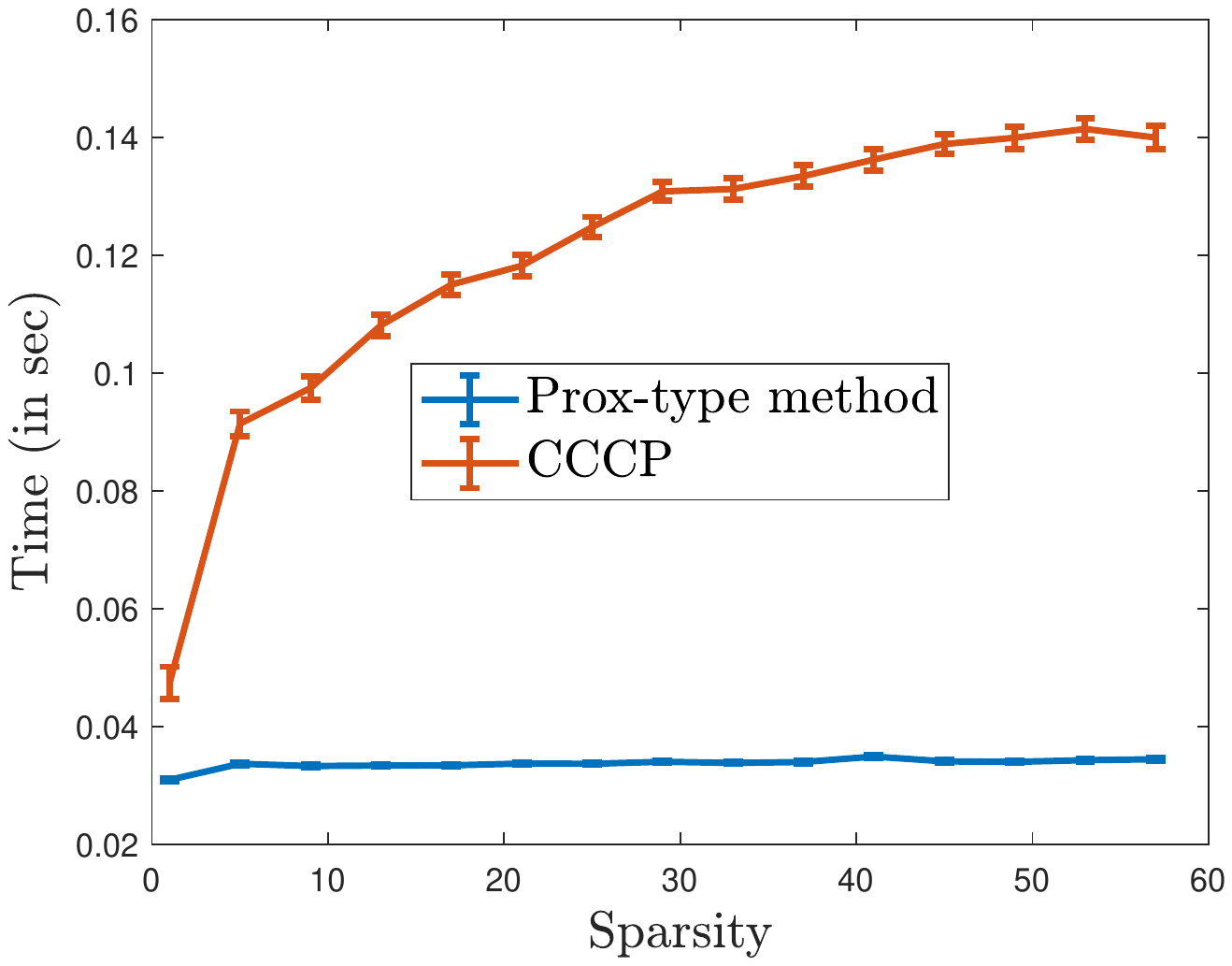} \phantomcaption
\end{subfigure}%
\begin{subfigure}{0.70\textwidth}\centering
\includegraphics[trim= 3cm 6.5cm 0.cm 5.5cm, clip,width = 1\linewidth]
                {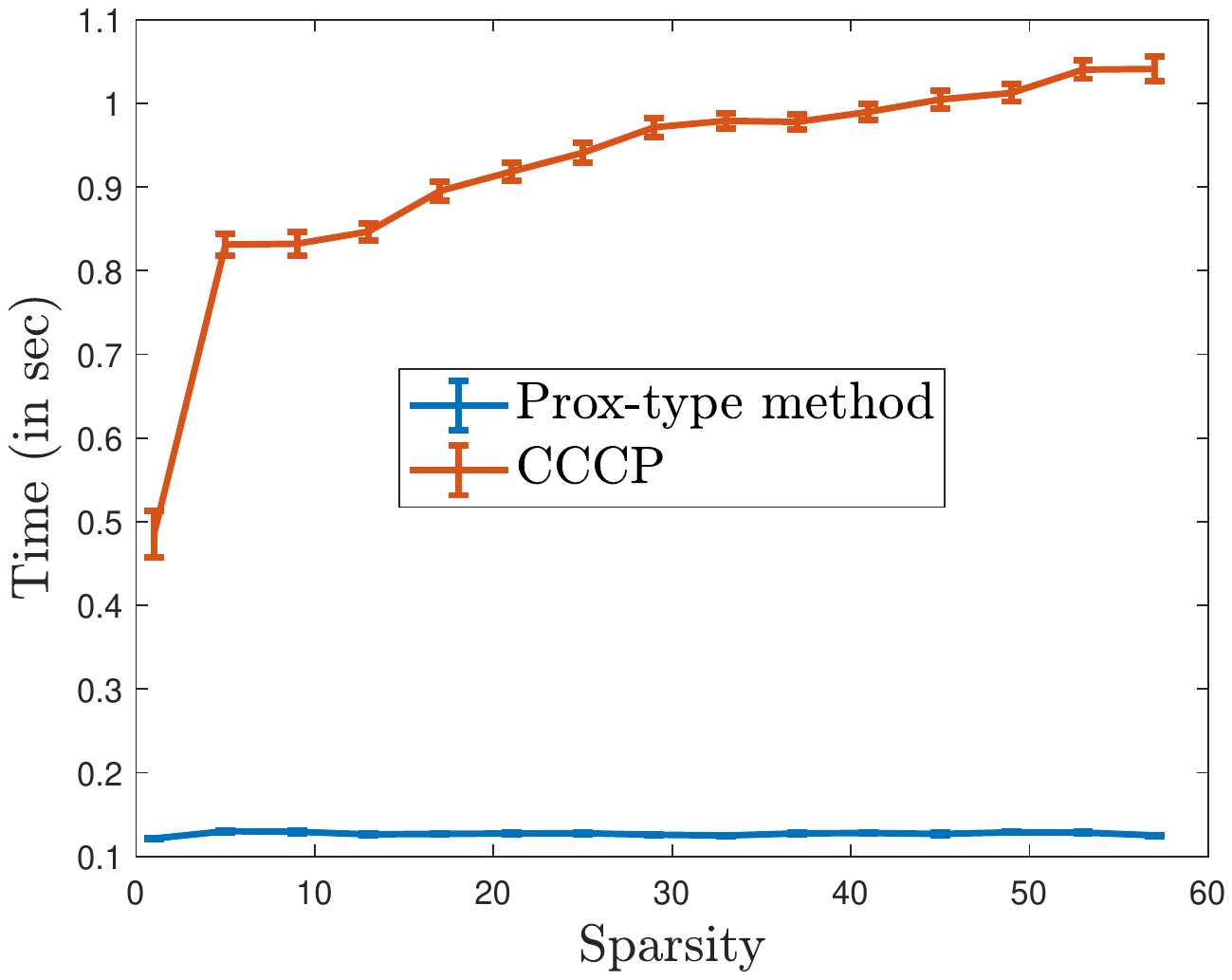} \phantomcaption
\end{subfigure}
}

\makebox[\linewidth][c]{%
\begin{subfigure}{0.70\textwidth}\centering
\includegraphics[trim= 0cm 7.5cm 3.cm 7.5cm, clip,width= 1\linewidth]
                {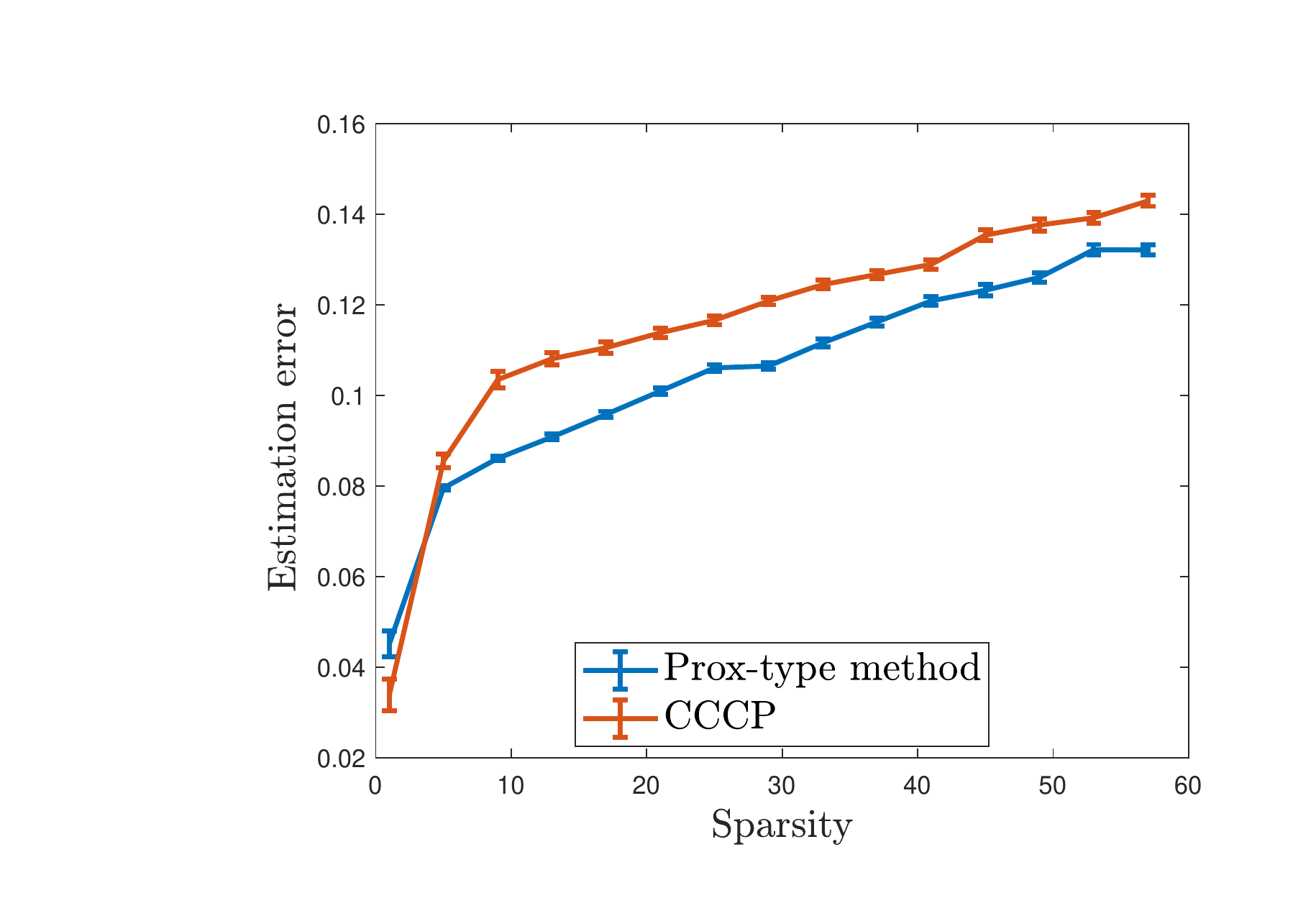}
                \phantomcaption
\end{subfigure}%
\begin{subfigure}{0.70\textwidth}\centering
\includegraphics[trim= 3cm 7.5cm 0.cm 7.5cm, clip,width = 1\linewidth]
                {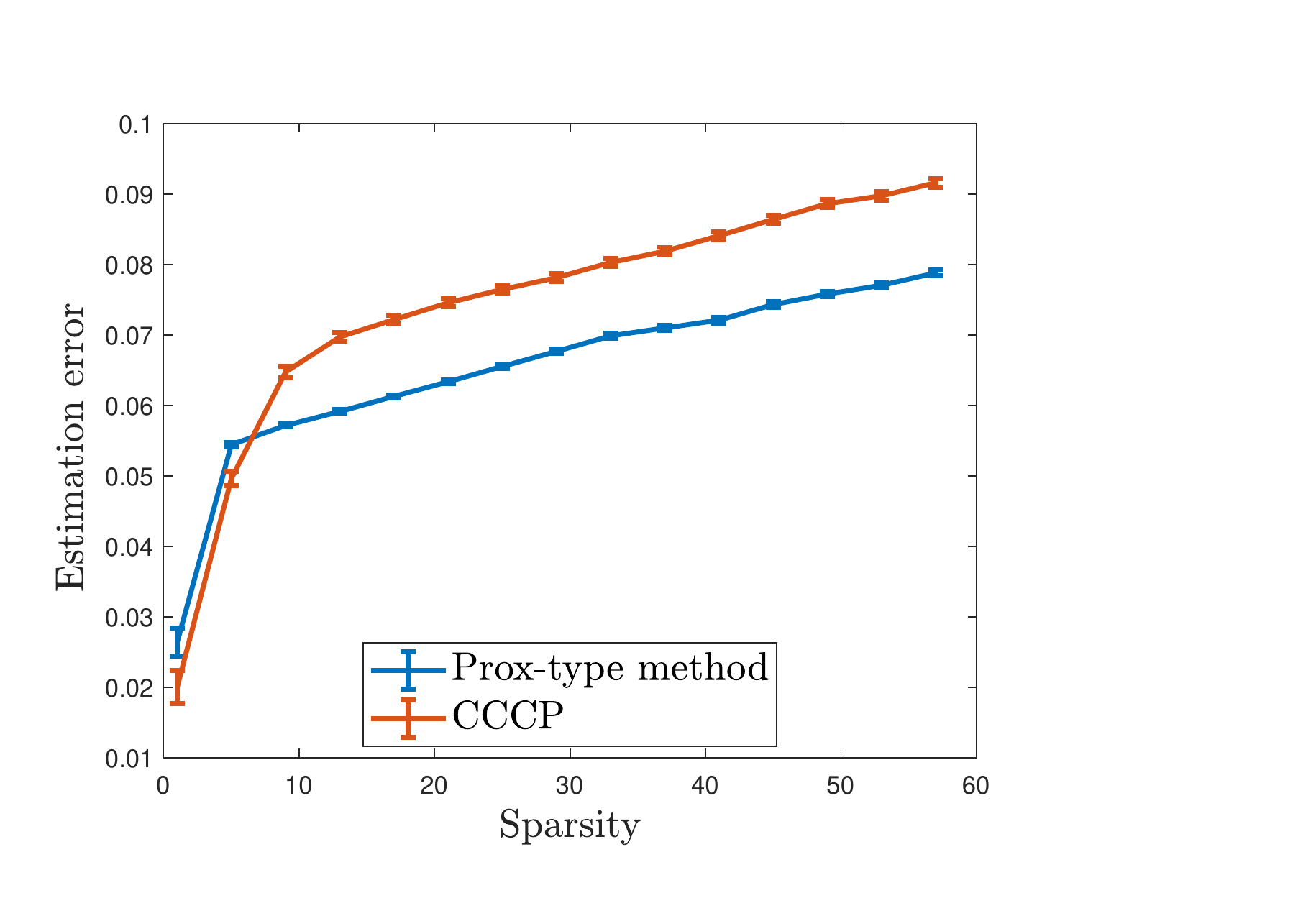}
                \phantomcaption
\end{subfigure}
}
\caption{ Figure showing performances of CCCP and
  Algorithm~\ref{AlgoProx} on best subset selection problem for
  synthetic data for different values of $(n,p)$. In the left column
  of the figure, the value of ${(n,p) = (190,300)}$, and in the right
  column, we show the plot with $(n,p) = (380,600)$. 
  Figures in the first
  row compare the performance in terms of 
  total runtime, whereas in the second
  row, we compare the algorithms in terms of estimation error.
  We see
  that Algorithm~\ref{AlgoProx} outperforms CCCP in terms of
  runtime. The performances of Algorithm~\ref{AlgoProx} and CCCP in
  terms of estimation error are similar for low values of sparsity,
  whereas Algorithm~\ref{AlgoProx} outperforms CCCP 
  when sparsity is moderate
  to large.  We initialized both the algorithms 
  from the same starting
  point. Results shown above are averaged over 
  100 replications, and we
  also provide pointwise error bars in the plots. }
\label{FigBestSubset}
\end{figure}

\paragraph{Comparison results:} 
Figure~\ref{FigBestSubset} shows the performances 
of the prox-type method
and CCCP for synthetic data simulated as above, with problem
parameters $(n,p)= (190,300)$ and $(n,p) = (380,600)$ and different
choices of sparsity $\sparsity$.

For both the algorithms, the tolerance level $\eta$ was set to $\eta =
10^{-8}$, whereas the maximum number of iterations was
$1000$. Figure~\ref{FigBestSubset} suggests that total runtime of the
prox-type method is significantly smaller than the runtime of
CCCP. Furthermore, the
estimation error for the prox-type
method is lower compared to CCCP, which
possibly suggests that prox-type method is finding better local minima
compared to
CCCP for the non-convex optimization problem~\eqref{EqnLinRegDC}.
In all our simulations we used same initializations for both the
algorithms. The simulation results shown in Figure~\ref{FigBestSubset}
are average over 100 replications, and we also provide the pointwise
error bar in the plots.


\subsubsection{Some theoretical guarantees}

Interestingly, it turns out that when applied to
problem~\eqref{EqnLinRegDC}, the convergence behavior of
Algorithm~\ref{AlgoProx} to a given stationary point $\bar{x}$ depends
on the behavior of a certain convex program defined in terms of
$\bar{x}$.  More precisely, for any point $\bar{x} \in
\real^{\usedim}$ with $|\bar{x}|_{(r)} > |\bar{x}|_{(r+1)}$, consider
the following convex relaxation of problem~\eqref{EqnLinRegDC}:
\begin{align}
   \label{EqnBestSubsetSubProb}
   \mathcal{P}(\bar{x}) & \mydefn \min_{ x \in \real^{\usedim} } \big
   \{ \| y - \LinModMat x \|_2^2 + \lambda \| x \|_{1} - \lambda
   \inprod{\grad \ftwo(\bar{x})}{x - \bar{x}} \big \}.
\end{align}
Note that $|\bar{x}|_{(r)} > |\bar{x}|_{(r+1)}$ implies the
differentiability of the function $\ftwo \mydefn
\lambda \sum_{i = \usedim-s+1 }^{\usedim} |x|_{(i)}$
which ensures that the above
problem is well-defined.


\begin{cors}
\label{CorBestSubset}
Let $\big \{ \x{k} \big \}_{k \geq 0}$ be any bounded sequence
obtained by applying Algorithm~\ref{AlgoProx} on
problem~\eqref{EqnLinRegDC}. Suppose there exists a limit point
$\bar{x}$ of the sequence $\big \{ \x{k} \big \}_{k \geq 0}$
satisfying $|\bar{x}|_{(r)} > |\bar{x}|_{(r+1)}$, and the convex
problem~\eqref{EqnBestSubsetSubProb} has unique solution. Then the sequence 
$\big\{ \x{k} \big\}_{k \geq 0}$ converges to the point $\bar{x}$, and
for all $k = 1,2,\ldots$, we have
\begin{align*}
\avg{ \| \grad \f(\x{k}) \|_2} \leq \frac{c_1}{k}, \quad \mbox{and}
\quad \|\x{k} - \bar{x} \|_2 \leq c q^k,
\end{align*}
where $q \in (0,1)$, and $(c, c_1)$ are positive constants independent of $k$.
\end{cors}

\paragraph{Comments on problem~\eqref{EqnBestSubsetSubProb}:}
 It can be shown that when the matrix $\LinModMat$ is of 
 full rank, the objective
 function in problem~\eqref{EqnBestSubsetSubProb} is strictly convex,
 and as a result, the problem~\eqref{EqnBestSubsetSubProb} has unique
 solution. In the proof of Corollary~\ref{CorBestSubset}, we show that
 the point $\bar{x}$ is always a minimizer of the convex
 problem~\eqref{EqnBestSubsetSubProb}, so that the uniqueness
 assumption implies that $\bar{x}$ is in fact the unique solution.

 
\subsection{Mixture density estimation}

As a final example, we consider the problem of estimating a
two-component mixture density, where each of the constituent densities
belong to an exponential family.  The density of an exponential family
(with respect to a fixed base measure, typically counting or Lebesgue)
takes the form
\begin{align}
  \label{EqnExpDensity}
p(y; \eta) & = g(y) \exp \big \{ \eta^\top T(y) - A(\eta) \big \}.
\end{align}
Here the function $T: \Yspace \rightarrow \real^d$ is a vector of
sufficient statistics, whereas the log-partition function
\begin{align*}
A(\eta) & \mydefn \log \Big( \int_{\Yspace} g(y) \exp \{ \eta^\top
T(y) \} dy \Big)
\end{align*}
serves to normalize the density.  The parameter vector $\eta \in
\real^d$ determines the choice of density within the family. See
Table~\ref{TabMixDensity} for some examples of $1$-dimensional
exponential families of this type.  It includes various familiar
examples, such as the Gaussian, Poisson and Beta families.
 \begin{center}
   \begin{table}
\begin{tabular}{ |P{3cm}||P{3cm}|P{5.2cm}|P{3cm}|  }
 \hline Distribution Name & $\eta$ & $A(\eta)$ & \mbox{Twice
   continuously} \mbox{differentiable and} \mbox{sub-analytic}
 \\ \hline Poisson $(\lambda)$ & $\ln(\lambda)$ & $\exp{\eta}$ &
 \checkmark \\
 \hline Geometric  $(p)$  &  $\ln(p)$  &$-\ln( 1 -
 \exp{\eta} )  $& \checkmark \\
\hline Gaussian$  (\mu,\sigma^2)  $&$ 
 \big( \frac{\mu}{\sigma^2}, -\frac{1}{2\sigma^2} \big)^\top$&
  $-\frac{\eta_{1}^2}{4\eta_2} - \frac{1}{2}\ln(-2\eta_2)$  &
 \checkmark \\
\hline Exponential  $(\lambda)$  &  $-\lambda$  &  $-\ln(-\eta)$  &
\checkmark \\
\hline Gamma  $(\alpha,\beta)$  &  $(\alpha -1, \beta)^\top$  &$\ln
\Gamma (\eta_1 + 1 ) - (\eta_1 + 1)\ln(\eta_2 )$& \checkmark
\\
\hline Weibull  $(\lambda, k \footnote{ \text{The shape parameter$  k$
    is known } })  $&$  -\frac{1}{\lambda^k}  $&$  \ln(-\eta) - \ln(k)$
& \checkmark \\
\hline Beta  $(\alpha,\beta)$  &  $(\alpha, \beta)^\top$  &$
\ln\Gamma(\eta_1 ) + \ln \Gamma(\eta_2 ) - \ln \Gamma(\eta_1 + \eta_2
)$& \checkmark \\
\hline
\end{tabular}
\captionof{table}{Table showing the natural parameter $\eta$ 
and the log-partition function $A$ for different densities of exponential
  family, which are twice continuously differentiable and
  sub-analytic. In Appendix~\ref{AppTabMixDensity} we prove 
  the log-partition functions $A$ mentioned in the above
  table are sub-analytic. }
\label{TabMixDensity}
   \end{table}
   \end{center}

 In the problem of mixture density estimation, one is interested in
 densities of the form
\begin{align}
\label{EqnMixDensity}
\mixFun(y; \underbrace{\pi, \eta_0, \eta_1}_{\theta} ) = \pi \, p(y; \eta_0)
+ (1 - \pi) p(y; \eta_1),
\end{align}
where $\pi \in (0,1)$ is an unknown mixing proportion, and $(\eta_0,
\eta_1)$ are the unknown parameters of the two underlying densities.

Given $n$ i.i.d. samples $\{y_i \}_{i=1}^\numobs$ drawn from a mixture
density of the form~\eqref{EqnMixDensity}, a standard goal is to
estimate the unknown parameter vector $\theta \mydefn (\pi, \eta_0,
\eta_1)$.  One way to do so is by computing the maximum likelihood
estimate (MLE), obtained via minimizing the negative log-likelihood of
parameter $\theta$ given by the data.
Frequently, a regularized form
of the MLE is used, say of the form
\begin{align}
  \label{EqnMixDensityLkhd}
  \min_{\theta} \Big \{ \underbrace{- \sum\limits_{i = 1 }^n \log
    \big( \mixFun( y_i; \theta )\big)}_{g(\theta)} \Big \} \qquad \mbox{such
    that $\eta_0, \eta_1 \in \real^\usedim$, $\pi \in [0,1]$, and $\|
    \eta_0 \|_2 \leq R_0$, $\| \eta_1 \|_2 \leq R_1$.}
\end{align}
Here $R_0 > 0$ and $R_1 > 0$ are tuning parameters providing upper
bound on the $\ell_2$-norms of the parameters $\eta_0$ and $\eta_1$
respectively, often chosen by a data-dependent procedure (such as
cross-validation).

By inspection, the objective function $g$ in
problem~\eqref{EqnMixDensityLkhd} is non-convex.  By standard theory
on exponential families, the function $A$ is always infinitely
differentiable on its domain, so that the objective function $g$ is
infinitely differentiable on the convex set
\begin{align*}
\Xspace & = \Big \{ \theta = (\eta_0, \eta_1, \pi) \, \mid \, \eta_j
\in \dom(A), \pi \in [0,1], \|\eta_j\|_2 \leq R_j \quad \mbox{for $j =
  0, 1$} \Big \}.
\end{align*}
Consequently, we may apply Algorithm~\ref{AlgoProx} with $\fone(\cdot)
= - \sum \limits_{i = 1 }^n \log \big( \mixFun( \cdot; y_i )\big)$, $\ftwo
\equiv 0$ and $\nonsmoothf(\cdot) = \mathbbm{1}_{
  \Xspace }(\cdot)$ and $\f = \fone -\ftwo +\nonsmoothf$.
  Interestingly, the log-partition function $A$ is
sub-analytic for many exponential family densities 
(see Table~\ref{TabMixDensity}), which ensures that
the function $\fone$ is also sub-analytic. 
In Appendix~\ref{AppSubanal}, we show
that continuous sub-analytic functions satisfy Assumption KL so that we can
apply Theorem~\ref{ThmProxSubanal} to obtain the following:

\begin{cors}
\label{CorMixDensity}
Any sequence $\{\theta^k \}_{k \geq 0} = \big \{\eta_0^k, \eta_1^k,
\pi^k \big \}_{k \geq 0 }$ obtained by applying
Algorithm~\ref{AlgoProx} to problem~\eqref{EqnMixDensityLkhd}
satisfies the following properties:
\begin{itemize}
  \item[(a)] It converges to a first order stationary point.
\item[(b)] For all $k = 1,2,\ldots$, we have $\avg{ \| \grad
  \f(\theta^k) \|_2 } \leq \frac{c_1}{k}$, where $c_1$ is a universal
  constant independent of $k$.
\end{itemize}
\end{cors}
\noindent See Appendix~\ref{AppCorMixDensity} for the proof of this
corollary.


\section{Discussion}
In this paper, we analyzed the behavior of three gradient-based
algorithms---namely gradient descent, a proximal method, and an
algorithm of the Frank-Wolfe type---for finding critical points of a
class of non-convex non-smooth optimization problems. For each of the
three algorithms, we provided non-asymptotic bounds on the rate of
convergence to a first-order stationary point.  We showed that our
algorithm can escape strict saddle point for a class of non-smooth
functions, thereby generalizing existing results for smooth functions.
As a consequence of our theory, we obtained a simplification of the
popular CCCP algorithm, and the simplified algorithm retains all the
convergence properties of CCCP. Finally, we showed that for a large
subclass of functions, which include continuous sub-analytic functions
as a special case, we can have a significant improvement in the rate
of convergence.

Our work leaves open a number of questions for future research.  For
instance, it would be interesting to characterize the class of
DC-based functions mentioned in problem~\eqref{ProbGradient} when the
convex function $\ftwo$ is non-differentiable.  Indeed, we then obtain
a larger non-class of non-differentiable functions, and we suspect
that Theorem~\ref{ThmNecSuffCond} can be suitably
generalized. Finally, we suspect that the proof techniques used here
can be leveraged in order to establish sharper results for other forms
of non-convex optimization problems.


\subsection*{Acknowledgements}

This work was partially supported by the Office of Naval Research
Grant DOD ONR-N00014 and National Science Foundation Grant NSF-DMS-1612948.



\appendix

\section{Technical background}

In this appendix, we collect some technical background on
subdifferentials and sub-analytic functions.

\subsection{Fr\'{e}chet and limiting subdifferential}
\label{AppFrechetSubdiff}

We first recall the definitions and some useful properties of
sub-differentials, which will be useful in subsequent sections.

\begin{defn}
  Let $f: \mathbb{R}^{\usedim} \mapsto \mathbb{R}$ be a lower
  semicontinuous function.  For any $x \in \dom(f)$, the Fr\'{e}chet
  subgradient of the function $f$ at point $x$ is defined as
  \begin{align*}
       \frechet {f(x)} = \Biggr \{ u \;\Big| \;\;\underset{y \neq x, y
         \rightarrow x }{\lim \inf} \frac{ f(y) - f(x) -\big\langle u,
         y-x \big\rangle }{\| y - x\|_2} \geq 0 \Biggr \}.
     \end{align*}   
\end{defn}

\begin{defn}
   Let  $f: \real^d \mapsto \real$  be a lower semi-continuous
   function. For any  $x \in \dom(f)$,  the limiting subdifferential
   of the function $f$  at point $x$  is defined as
  \begin{align*}
    \limiting f(x) = \Big \{ u \; \Big| \;\; \exists \;\; \x{k}
    \rightarrow x, u^k \rightarrow u \text{ with }f(x^k) \rightarrow
    f(x) \text{ and } u^k \in \frechet{ f(\x{k})} \text{ as } k
    \rightarrow \infty \Big \}.
  \end{align*}
\end{defn}

\paragraph{Properties:}
The following properties of Fr\'{e}chet and limiting sub-differential
are provided in Chapter 8 of the book Rockafeller and
Wets~\cite{rockafellar2009variational}.
\begin{itemize}
\item[(a)] For any proper convex function $\ftwo$, we have $\limiting
  \ftwo(x) = \frechet \ftwo(x)$ for all $x \in \dom(\ftwo)$, and both
  quantities agree with the usual subgradient of the convex function
  $\ftwo$.
\item[(b)] If a function $\fone$ is smooth in a neighborhood of a
  point $x$, then $\limiting \f(x) = \grad \f(x)$.
\item[(c)] Consider a function $\f$ of the form 
$\f = \fone + \nonsmoothf$,
  where the function $\fone$ is smooth in a neighborhood of a point
  $x$, and the function $\nonsmoothf$ is proper convex and
  finite at the point $x$.  Then the limiting sub-differential of the
  function $f$ at the point $x$ is given by ${\limiting{\f(x)} =
    \grad \fone(x) + \subgrad \nonsmoothf(x)}$.
\item[(d)] \emph{(Graph continuity:)} Consider a sequence
  $\big\{ \big( \x{k}, u^k \big) \big\}_{k \geq 1}$ in
  \graph($\partial_{L} f$) such that
the sequnece $\{ (\x{k}, u^k, f(\x{k}) \}_{k \geq 0}$ 
converges to a point $(x, u, \f(x))$.
Then $(x,u) \in \graph(\partial_{L} f)$. Recall that
  \graph($\partial_{L} f$) \mydefn $\big\{ (x,u) 
  \in \real^{\usedim} \times \real~|~u
\in \partial_{L}f(x) \big\}$.
\end{itemize}

.


\subsection{Sub-analytic functions satisfy KL-assumption}
\label{AppSubAnalKL}
In this appendix, we show that continuous sub-analytic functions
satisfy the KL-inequality. We also provide examples of functions which
are sub-analytic.

\paragraph{Comments on limiting sub-differential:}

In order to facilitate our discussion, we mention some simple facts on
limiting subdifferential of a function $\f$, where $\f$ is of the form
$\f = \fone - \ftwo$ (Theorems~\ref{ThmGradient}
and~\ref{ThmGradSubanal}) or $\f = \fone + \nonsmoothf - \ftwo$
(Theorems~\ref{ThmProx} and~\ref{ThmProxSubanal}).  The following
properties are direct consequences of properties of the limiting
subdifferential mentioned in Appendix~\ref{AppFrechetSubdiff}.

\begin{itemize}
\item Suppose the difference function $\f = \fone - \ftwo$ satisfies
  parts (a) and (b) of Assumption GR. Then we have
  \begin{align*}
{\limiting{(-\f)(x)} = \partial \ftwo(x) - \grad \fone(x)}, \quad
\mbox{and moreover} \| \grad \f(x) \|_2 \mydefn \|\grad
\fone(x) - \partial \ftwo(x)\|_2 = \| \limiting{(-\f)(x)} \|_2.
  \end{align*}

  \item Suppose the function $\f = \fone + \nonsmoothf - \ftwo$, where
    the function $\ftwo$ is locally smooth, and the function $\f$
    satisfies Assumption PR part (b).  Then $\limiting{\f(x)} = \grad
    \fone(x) - \grad \ftwo(x) + \partial \nonsmoothf(x)$.
    Consequently, we have that ${\| \grad \f(x) \|_2 =
      \|\limiting{\f(x)}\|_2}$.
\end{itemize}

We prove that continuous sub-analytic functions satisfy
Assumption KL by utilizing  a
previous work by Bolte et al.~\cite{bolte2007Lojasiewicz}.
In order to facilitate 
further discussion, we introduce few notations used in the 
paper~\cite{bolte2007Lojasiewicz}.
We use $m_{\f}(x)$ to denote the $\ell_2$ distance of the set
$\limiting{\f(x)}$ from zero; concretely,  
$m_{\f}(x) := \dist_{\| \cdot \|_2} \big(0, \limiting{\f(x)} \big)$. 
In Theorem 3.1 (for critical points of the function $\f$ ) 
and Remark 3.2 (for non-critical points of the function $\f$), 
Bolte et al. proved the following fact about sub-analytic functions.
\begin{lems}
\label{LemBolte}
 \emph{(Bolte et al.~\cite{bolte2007Lojasiewicz}):} Let $\f:
 \real^{\usedim} \mapsto \real \cup \{+\infty\}$ be a sub-analytic
 function with closed domain, and assume that $\f|_{\dom(\f)}$ is
 continuous.  Then for any $a \in \dom(f)$, there exists an exponent
 $\theta \in [0,1)$ such that, the function $\frac{|\f -
     \f(a)|^{\theta}}{m_{\f}}$ is bounded above in a neighborhood of
   $a$.
\end{lems}
\noindent Using Lemma~\ref{LemBolte}, we now argue that sub-analytic
functions, under the conditions of Theorem~\ref{ThmGradSubanal}
or Theorem~\ref{ThmProxSubanal}, satisfy Assumption KL.
\begin{lems}
\label{LemSubanalKl1}
Any sub-analytic function $\f$ satisfying Assumption GR also satisfies Assumption KL.
 \end{lems} 
\begin{proof}
First, note that the function $\f$ is continuous by Assumption GR;
suppose $\f$ is sub-analytic, then from
properties of sub-analytic functions, we have that the function $-\f$ is
also sub-analytic. Furthermore, the function $-\f$ is continuous in the
closed domain $\range$ ---which by Lemma~\ref{LemBolte} guarantees
that, for any $a \in \range$, there exists $\theta \in [0,1)$ such that
  the ratio $\frac{|-\f - (-\f(a))|^{\theta}}{ m_{(-\f)}}$ is bounded
  above in a neighborhood of the point $a$.  Since $|-\f - (-\f(a))| =
  |\f - \f(a)|$, proving satisfiability of Assumption KL reduces to
  showing that $m_{(-\f)}(x)$ is upper bounded by $\|\grad \f(x)\|_2$.
 To this end, note that from the discussion about limiting
  subdifferential in the paragraph above Lemma~\ref{LemBolte},
  we have
  \begin{align}
    \| \grad \f(x) \|_2 = \| \limiting{(-\f)(x)} \|_2 \;
    \stackrel{(i)}{\geq} \; m_{(-\f)}(x),
  \end{align} 
  where step (i) follows from the definition of $m_{(-\f)}(x)$.
  Putting together the pieces, we conclude that any sub-analytic 
  function $\f$ which satisfies Assumption GR, also satisfies
  Assumption KL.
\end{proof}

\begin{lems}
\label{LemSubanalKl2}
Suppose that, in addition to the conditions on the
functions $(\fone, \ftwo, \nonsmoothf)$ from Theorem~\ref{ThmProx}, 
the function $\f \mydefn \fone - \ftwo + \nonsmoothf$ is continuous and sub-analytic 
in its domain $\dom(\f)$, and the domain
$\dom(\f)$ is closed. Then the function $\f$ satisfies Assumption KL.
 \end{lems} 
\begin{proof}
  Since the function $\f|_{\dom(\f)}$ is continuous and sub-analytic by assumption, from
  Lemma~\ref{LemBolte}, we have that for any $a \in \dom(\f)$ there
  exists a $\theta \in [0,1)$ such that, the ratio $\frac{|\f -
      \f(a)|^{\theta}}{m_{\f}}$ is bounded above in a neighborhood
      of the point $a$. In order to justify
    satisfiability of Assumption KL, it suffices to prove that
    $m_{\f}(x)$ is upper bounded by $\|\grad \f(x)\|_2$. To this end,
    note that the function $\ftwo$ is locally smooth by assumptions of
    Theorem~\ref{ThmProx} part (b). Hence, from the discussion about
    limiting subdifferential in the paragraph above
    Lemma~\ref{LemBolte}, we have
  \begin{align}
    \| \grad \f(x) \|_2 & = \| \limiting{\f(x)} \|_2 \;
    \stackrel{(i)}{\geq} \; m_{\f}(x),
  \end{align} 
where step (i) follows from the definition of $m_{\f}(x)$.  Putting
together the pieces, guarantees that the function $\f$ satisfies
Assumption KL.
\end{proof}


\subsection{Instances of sub-analytic functions}
\label{AppSubanal}

In Appendix~\ref{AppSubAnalKL}, we proved that continuous sub-analytic functions satisfy Assumption KL, and in those cases,---by
Theorems~\ref{ThmGradSubanal} and~\ref{ThmProxSubanal}---we have a
faster rate of convergence of Algorithms~\ref{AlgoGradient}
and~\ref{AlgoProx}. In this appendix, we provide examples of
functions which are sub-analytic. We start by providing definitions of
sub-analytic functions following the definition of Bolte et al.~\cite{bolte2007Lojasiewicz}.

A subset $S \subset \real^{\usedim}$ is called \emph{semi-analytic},
if each point of $\real^{\usedim}$ admits a neighborhood $\nbhd$
such that the set $\genericset \cap \nbhd$ has the form
\begin{align*}
     \genericset \cap \nbhd = \cup_{i=1}^{p} \cap_{j=1}^{q}
     \big \{ x \in \nbhd \mid \reanaone_{ij} = 0, \reanatwo_{ij}
     > 0 \big \},
\end{align*}
where the functions $\reanaone_{ij}, \reanatwo_{ij} : \nbhd \mapsto
\real$ are real-analytic.

A set \genericset~ is called \emph{sub-analytic}, if each point of
$\real^{\usedim}$ admits a neighborhood $\nbhd$ such that
\begin{align*}
\genericset \cap \nbhd = \big \{ x \in \real^{\usedim} : \big( x,y
\big) \in B \big \},
\end{align*}
where $B$ is a bounded semi-analytic subset of $\real^{\usedim} \times
\real^{m}$ for some $m \geq 1$. A function $\f$ is called sub-analytic
if the graph of $\f$, defined by $\graph(\f) \mydefn \big
\{ (x,y) \in \real^{\usedim} \times \real : f(x)=y \big \}$,
is sub-analytic.

The class of sub-analytic functions is quite large. In order to motivate the
reader, we provide few examples here.
The following results can be found in~Bolte et
al.~\cite{bolte2014proximal} and Chapter 6 in the
book~\cite{facchinei2007finite}.
\begin{itemize}
\item[(a)] Any real-valued polynomial or analytic function is
  sub-analytic.
\item[(b)] Any real-valued semi-algebraic or semi-analytic
  function is sub-analytic.
\item[(c)] Indicator function of a semi-algebraic set is sub-analytic.
\item[(d)] Sub-analytic functions are closed under finite linear
  combinations, and the product of two sub-analytic functions is
  sub-analytic.
\item[(e)] Pointwise maximum and minimum of a finite collection
  of sub-analytic functions are sub-analytic.
\item[(f)] \emph{Composition rule:} If $g_1$ and $g_2$ 
are two sub-analytic functions with
  the function $g_1$ being continuous, then
  the composition function $g_2 \circ g_1$ is sub-analytic. In
  fact, the class of continuous sub-analytic functions are
  \emph{closed under algebraic operations}.
 \end{itemize}


\section{Proofs related to Algorithm~\ref{AlgoGradient}}


In this appendix, we collect the proofs of various results related to
the gradient-based Algorithm~\ref{AlgoGradient}, including
Theorem~\ref{ThmGradient}, Corollaries~\ref{CorBacktracking}
and~\ref{CorSaddlePoint}, and
Proposition~\ref{PropCCCP}.


\subsection{Proof of Theorem~\ref{ThmGradient}}
\label{AppThmGradient}

Our proof of this theorem, as well as subsequent ones, depends on the
following descent lemma:
\begin{lems}
\label{LemDescentStep}
Under the conditions of Theorem~\ref{ThmGradient}, we have
\begin{align}
  \label{EqnDescentCondGrad}
  \x{k} \in \intr(\range) \quad \mbox{and} \quad \f(\x{k+1}) 
  \leq \f{(\x{k})} -  \frac{\alpha}{2} \| 
  \grad{\f(\x{k})}\|_2^2 \quad
  \mbox{for all $k = 0, 1, 2, \ldots$.}
 \end{align}
\end{lems}
\noindent See Appendix~\ref{AppLemDescentStep} for the proof of this
lemma. \\


\noindent We now prove Theorem~\ref{ThmGradient} using Lemma~\ref{LemDescentStep}.
\paragraph{Convergence of function values:}
We first prove that the function value sequence $\{ f(\x{k}) \}_{k \geq 0}$
is convergent.  Since $\fstar \mydefn \min \limits_{x \in \range}
f(x)$ is finite by assumption, 
and $\x{k} \in \intr(\range)$ for all $k \geq 0$  by
Lemma~\ref{LemDescentStep}, the sequence $\{ f(\x{k}) \}_{k
  \geq 0}$ is bounded below.  For any non-stationary $\x{k}$,
inequality~\eqref{EqnDescentCondGrad} also ensures that $f(\x{k}) >
f(\x{k+1})$; hence, there must exist some scalar $\bar{f}$ such that $\lim
\limits_{k \rightarrow \infty} f(\x{k}) = \bar{f}$.


\paragraph{Stationarity of limit points:}  Next, we establish that
any limit point of the sequence $\{\x{k}\}_{k \geq 0}$ must be stationary.
Consider a subsequence $\{ \x{k_j}\}_{j \geq 0 }$ of $
\big \{ \x{k} \big \}_{ k \geq 0 }$ such that $\x{k_j}
\rightarrow \bar{x}$, and let $\{ u^{k_j} \}_{j \geq 0}$ be
the associated sequence of subgradients.  It suffices to exhibit a
sub-gradient $\bar{u} \in \subgrad \ftwo(\bar{x})$ such that $\grad
\fone(\bar{x}) - \bar{u} = 0$.

Since the sequence $\{\x{k_j}\}_{j \geq 0}$ converges to $\bar{x}$,
we must have ${\| \grad
\f(\x{k_j}) \|_2 = \|\grad \fone(\x{k_j}) - u^{k_j}\|_2 \rightarrow 0}$;
The function $\fone$ is continuously differentiable
by assumption, and we have
${\grad{\fone(\x{k_j})} \rightarrow \grad \fone(\bar{x})}$. Combining
these we find that $u^{k_j} \rightarrow \grad
\fone(\bar{x})$. Furthermore, by continuity of the function
$\fone$, we have
$\fone(\x{k_j}) \rightarrow \fone(\bar{x})$. Putting together the pieces
we have established above that 
${\big( \x{k_j}, u^{k_j}, \fone(\x{k_j}) \big) \rightarrow \big( \bar{x},
\bar{u}, \fone(\bar{x}) \big)}$, where $\bar{u} \mydefn \grad \fone(\bar{x})$.
Consequently, the graph continuity of limiting-sub-differentials (see
Appendix~\ref{AppFrechetSubdiff}) guarantees that $\bar{u} =
\grad{\fone (\bar{x})} \in \subgrad{\ftwo(\bar{x})}$.  Overall, we
conclude that ${\nabla \f(\bar{x}) \mydefn 
\grad \fone(x) - \bar{u} = 0}$,
so that $\bar{x}$ is a stationary point as claimed.


\paragraph{Establishing the bound~\eqref{EqnAlgoGradientGradBound}:}

Finally, we prove the claimed bound~\eqref{EqnAlgoGradientGradBound}
on the averaged squared gradient.  Recalling that $\fstar \mydefn \min
\limits_{x \in \range} f(x)$ is finite, we have
\begin{align*}
f(\x{0}) - \fstar \geq f(\x{0}) - f(x^{k+1}) & =
\sum_{j=0}^{k}f(x^{j}) - f(x^{j+1})\nonumber \\
& \stackrel{(i)}{\geq} \frac{\alpha}{2}\sum_{j=0}^{k}\|
\grad{\f}(\x{k}) \|_2^2 \\ & = \frac{\alpha(k+1)}{2} \avg{ \|\grad
  \f{(\x{k})} \|_2^2 },
\end{align*}
where step (i) follows from equation~\eqref{EqnDescentCondGrad}.
Rearranging yields the claimed bound~\eqref{EqnAlgoGradientGradBound}
on the averaged squared gradient.

\vspace{8pt}


\subsubsection{Proof of Lemma~\ref{LemDescentStep}}
\label{AppLemDescentStep}

Recall that by assumption, the function $\fone$ is continuously
differentiable and $\M_{\fone}$-smooth, and the function $\ftwo$ is
convex. As a consequence, for any vector $\x{k} \in \range$ and
subgradient $u^k \in \subgrad{\ftwo(\x{k})}$, we have
\begin{subequations}
  \begin{align}
    \label{EqnSmoothUB} 
    \fone(x) & \leq \fone(\x{k}) + \inprod{\grad \fone(\x{k})}{ x -
      \x{k}} + \frac{\M_{\fone}}{2} \| x - \x{k} \|_2^2 \\
\label{EqnCvxUB}
 \ftwo(x) & \geq \ftwo(\x{k}) + \inprod{\gradftwo}{x - \x{k}}.
\end{align}
\end{subequations}
Combining inequalities~\eqref{EqnSmoothUB} and~\eqref{EqnCvxUB} yield
\begin{align}
 \label{EqnCvxMajorGrad}
\f(x) = \fone(x) - \ftwo(x) \leq f(\x{k}) + \inprod{\nabla
  \fone(\x{k}) - \gradftwo}{x - \x{k}} + \frac{\M_{\fone}}{2}\| x -
\x{k} \|_2^2.
\end{align}
Substituting $x = \x{k+1} \mydefn \x{k} - \alpha \big(
\grad{\fone}(\x{k}) - \gradftwo \big)$ in
equation~\eqref{EqnCvxMajorGrad} and simplifying yields
\begin{align*}
\f(\x{k}) - \f(\x{k+1}) \geq \big(\frac{1}{\alpha} -
\frac{\M_{\fone}}{2} \big) \|\x{k+1} - x^{k}\|_2^2 & = \alpha \big( 1 -
\frac{\alpha \M_{\fone}}{2} \big) \|\grad{\fone}(\x{k}) - \gradftwo
\|_2^{2} \\
& \stackrel{(i)}{\geq} \frac{\alpha}{2} \|\grad{\f}(\x{k}) \|_2^{2},
\end{align*}
where inequality (i) follows from the upper bound $\alpha \leq
\frac{1}{\M_{\fone}}$. This proves the second part of the stated
lemma.  As for the claim that the sequence remains in the interior of
the set $\range$, note that 
${\f(\x{k+1}) \leq \f(\x{k}) \leq \f(\x{0})}$, which
ensures that \mbox{$\x{k+1} \in \Level(\f(\x0)) \subset
  \intr(\range)$,} as claimed.

\subsection{Proof of Corollary~\ref{CorBacktracking}}
\label{AppCorBacktracking}
The first part of the proof builds on a simple 
application of Theorem~\ref{ThmGradient} and the definition 
of effective smoothness constant $\M^*_{\f}$. The second part 
of the proof utilizes a relation between the backtracking step size 
and the effective smoothness constant. For sake of completeness, we 
first describe the gradient descent backtracking algorithm.
\begin{algorithm}
\caption{ $\;\;\;\;\;$ Gradient descent with backtracking}
\label{AlgoBacktracking}
\begin{algorithmic}[1]
\STATE{Given an initial point $\x{0} \in \intr(\range)$ and parameter
  $\beta \in (0, 1)$:}
\FOR{ $k = 0, 1, 2, \ldots$ }
\STATE{Choose the smallest nonnegative integer $i_k$ such that the
  step size $\step{k} \mydefn \beta^{i_k}$ satisfies:
\vspace*{-0.5\baselineskip}
\begin{align}
\label{EqnBacktrackDescent}
    \f \big (\x{k} - \step{k} \grad \f(\x{k}) \big) & \leq \f(\x{k}) -
    \frac{\step{k}}{2} \| \grad \f(\x{k}) \|^2.
\end{align}
\vspace*{-0.5\baselineskip}
}
\STATE{Update $x^{k+1}=x^{k}- \step{k}\grad{\f(\x{k})}$.}
\ENDFOR
\end{algorithmic}
\end{algorithm}

\paragraph{Establishing the bound in~\eqref{EqnFastGradBound}:}
For any step size $\alpha$ in the interval $\big(0,
\frac{1}{\M_{\f^*}} \big)$, the definition of the effective smoothness
constant $\M_{\f^*}$ ensures the following property. There exists a
$\M_{\fone}$-smooth function $\fone$ and a convex-differentiable
function $\ftwo$ with $\f = \fone - \ftwo$, and the scalar
$\M_{\fone}$ satisfies ${\alpha < \frac{1}{\M_{\fone}} \leq
  \frac{1}{\M_{\f^*}}}$.  Since the function $\f$ is differentiable,
applying Algorithm~\ref{AlgoGradient} on the function $\f$ with the
decomposition $\f = \fone - \ftwo$ is equivalent to applying gradient
descent on $\f$. Furthermore, the step size $\alpha$ satisfies the
upperbound $\alpha \leq \frac{1}{\M_{\fone}}$, and applying the
bound~\eqref{EqnAlgoGradientGradBound} from Theorem~\ref{ThmGradient}
yields:
\begin{align}
\label{EqnGradBoundFast}
\avg{ \| \grad \f(\x{k}) \|_2^2 } & \leq \frac{2 \big(\f(\x{0})-
  \fstar \big)}{\alpha (k+1)}.
\end{align}

\paragraph{Establishing the backtracking bound~\eqref{EqnBacktrackingBdd}:}
For any fraction $\beta \in (0, 1)$, the definition of the effective
smoothness constant $\M_{\f^*}$ guarantees the following.  There
exists a $\M_{\fone}$-smooth function $\fone$ and a convex and
differentiable function $\ftwo$ with $\f = \fone - \ftwo$, and the
scalar $\M_{\fone}$ satisfies ${\beta \M_{\fone} \leq \M_{\f^*} \leq
  \M_{\fone}}$.  Comparing the descent step~\eqref{EqnDescentCondGrad}
from Lemma~\ref{LemDescentStep} and step~\eqref{EqnBacktrackDescent}
in Algorithm~\ref{AlgoBacktracking}, we conclude that the step size
$\step{k}$ satisfies the lower bound $\step{k} \geq \min \big\{ 1,
\frac{\beta}{\M_{\fone}} \big\} \geq \min \big\{ 1,
\frac{\beta^2}{\M^*_{\f}} \big\}$.  Applying the descent
step~\eqref{EqnBacktrackDescent} in Algorithm~\ref{AlgoBacktracking}
repeatedly and then utilizing the last lower bound on step size
$\step{k}$, we find that for all $k = 0,1,2\ldots$
\begin{align*}
    \f(\x{0}) - \f(\x{k+1}) \geq \sum_{i = 0}^{k} \frac{\step{k}}{2}
    \| \grad \f(\x{k}) \|^2 \geq \min \Big \{ \frac{1}{2},
    \frac{\beta^2 }{2\M^*_{\f}} \Big \} \sum_{i = 0}^{k} \| \grad
    \f(\x{k}) \|^2.
\end{align*}
Rearranging the last inequality yields:
\begin{align}
    \avg{\| \grad \f(\x{k}) \|^2} & \leq \frac{2\max \big\{
    1, \frac{\M^*_{\f}}{\beta^2} \big\} \big( \f(\x{0}) 
    - \f(\x{k+1}) \big)}
    {(k+1)} \nonumber \\
    & \stackrel{(i)}{\leq}
    \frac{2\max \big\{ 1, \M^*_{\f} \big\}
    \big( \f(\x{0}) - \fstar \big)}{\beta^2 (k+1)},
\end{align}
where step (i) follows since $\beta \in (0, 1)$, along with the lower
bound $\f(\x{k+1}) \geq \fstar$.


\subsection{Proof of Corollary~\ref{CorSaddlePoint}}
\label{AppCorSaddlePoint}

Based on Theorem 4 of Lee et al.~\cite{Jason-16}, it suffices to show
that the gradient map ${G(x) \mydefn x - \alpha \grad{f}(x)}$ is a
diffeomorphism for any step size $\alpha \in \big(0,
\frac{1}{\M_{\fone}} \big)$.  Recall that a map $G: \real^{\usedim}
\mapsto \real^{\usedim}$ is a diffeomorphism if the map $G$ is a bijection,
and both the maps $G$ and $G^{-1}$ are continuously differentiable.

\paragraph{Injectivity:}
We first prove that $G$ is an injective map. Consider a pair of vectors $x, y$
such that $G(x) = G(y)$; our aim is to prove that $x = y$.  The
condition $G(x) = G(y)$ is equivalent to ${x - y = \alpha \big(
\grad{f(x)} - \grad{f(y)} \big)}$, and we have that
\begin{align*}
\| x - y \|_2^2 & = \alpha \inprod{x - y}{\grad{f(x)} - \grad{f(y)}}
\\
& = \alpha \inprod{ x - y}{\grad{\fone(x)} - \grad{\fone(y)}} - \alpha
\inprod{x-y}{\grad{\ftwo(x)} - \grad{\ftwo(y)}} \\
& \stackrel{(i)}{\leq} \alpha \M_{\fone} \| x- y\|_2^2- \alpha
\inprod{x-y}{\grad{\ftwo(x)} - \grad{\ftwo(y)}} \\
& \stackrel{(ii)}{\leq} \alpha \M_{\fone} \| x- y\|_2^2.
\end{align*}
Here inequality (i) follows because the gradient $\grad{\fone}$ is
$\M_\fone$-Lipschitz by assumption; inequality (ii) follows from
the convexity of the function $\ftwo$, which implies the monotonicity
of the gradient $\grad{\ftwo}$.  Finally, since the step size 
$\alpha < \frac{1}{M_\fone}$ by
assumption, the inequality 
$\| x - y \|_2^2 \leq \alpha\M_{\fone} \| x - y \|_2^2 $ 
can hold only when $x = y$.
\paragraph{Surjectivity:}
For any fixed vector $y \in \real^\usedim$, consider the following
problem
\begin{align}
 \label{EqnOntoProof}  
  \arg \min_{x \in \real^\usedim } \; \Big \{ \frac{1}{2} \| x - y
  \|_2^2- \alpha \fone(x) + \alpha \ftwo(x) \Big \}.
\end{align}
Observe that for any step size $\alpha \in \big(0,
\frac{1}{\M_{\fone}} \big)$ and any fixed vector $y \in
\real^{\usedim}$, the map $x \mapsto \frac{1}{2} \| x - y \|_2^2-
\alpha \fone(x)$ is strongly convex, whence the map $x \mapsto
       {\frac{1}{2} \| x - y \|_2^2- \alpha \fone(x) + \alpha
         \ftwo(x)}$ is also strongly convex. Consequently, the convex
       problem~\eqref{EqnOntoProof} has a unique minimizer, and we
       denote it by $x_y$. In order to prove surjectivity of the map
       $G$, it suffices to show the point $x_y$ is mapped to the point
       $y$. Recalling the KKT conditions of the
       problem~\eqref{EqnOntoProof}, we have that
\begin{align*}
y = x_y - \alpha \nabla f(x_y) = G(x_y),
\end{align*}
 which completes the proof of surjectivity
of the map $G$. \\

Combining the injectivitivty and the surjectivity of
the map $G$, we conclude
that the inverse map $G^{-1}$ exists.
Next, let $DG(\cdot)$ denote the
Jacobian of the map $G$, then ${DG(x) = \text{I} - \alpha \nabla
^{2} \fone(x) +\alpha \nabla ^{2} \ftwo(x)}$. 
Since the function $\fone$ is $\M_{\fone}$-smooth, and the map
$G$ is continuously differentiable, standard application  
of the inverse-function theorem guarantees that for 
all step size $\alpha < \frac{1}{
\M_{g}}$, the inverse map $G^{-1}$ 
is continuously differentiable. 
Putting together the pieces, we conclude
that map $G^{-1}$ exists, and both the maps
$(G,G^{-1})$ are continuously differentiable.
Overall, we have established that the map $G$ is a
diffeomorphism, as claimed.


\subsection{Proof of Proposition~\ref{PropCCCP}}
\label{AppPropCCCP}

The CCCP update at step $(k+1)$ is given by $\x{k+1} = \arg \min
\limits_{x \in \range} \convmajor(x,\x{k})$, where
\begin{align}
  \label{EqnDefnH}
\convmajor(x,\x{k}) & \mydefn \fone(x) - \ftwo(\x{k}) - \inprod{\grad
  \ftwo(\x{k})}{x - \x{k}}.
\end{align}
Observe that step $(k+1)$ of Algorithm~\ref{AlgoGradient} is equivalent to
a gradient descent update with step size $\alpha$ on the map $x \mapsto
\convmajor(x, \x{k})$. Accordingly, if we define $y^{k+1} = \x{k} - \alpha
\nabla \convmajor(x, \x{k})$, then we have \mbox{$\convmajor(y^{k+1},\x{k}) \geq
  \convmajor(\x{k+1},\x{k})$}; moreover
 \begin{align}
\f(\x{k}) - \f(\x{k+1}) & \stackrel{(i)}{\geq} \convmajor(\x{k},\x{k}) -
\convmajor(\x{k+1},\x{k}) \nonumber \\
& \stackrel{(ii)}{\geq} \convmajor(\x{k},\x{k}) - \convmajor(y^{k+1},\x{k}) \nonumber \\
& \stackrel{(iii)}{\geq} \frac{1}{2\M_{\fone}} \| \grad \f(\x{k})
\|_2^2. \label{EqnCCCPDescentStep}
\end{align}
 Here inequality (i) follows from the equality $\convmajor(\x{k},\x{k}) =
 \f(\x{k})$ combined with the lower bound $\convmajor(x,\x{k}) \geq \f(x)$.
 Inequality (ii) follows since $\convmajor(y^{k+1},\x{k}) \geq
 \convmajor(\x{k+1},\x{k})$, and inequality (iii) follows from
 Lemma~\ref{LemDescentStep} with step size $\alpha =
 \frac{1}{\M_{\fone}}$. Note that equation~\eqref{EqnCCCPDescentStep}
 guarantees that the function value sequence $\{ \f(\x{k}) \}_{k \geq 0}$
 is decreasing. Since the
 function $\f$ is bounded below, we have that the
 sequence $\{ \f(\x{k}) \}_{k \geq 0}$ converges. 
 In order to prove that all limit points
 of the sequence $\big \{ \x{k} \big \}_{ k \geq 0 }$ are
 critical points, we follow the corresponding argument in proof of
 Theorem~\ref{ThmGradient}. This completes the proof of part (a) in
 Proposition~\ref{PropCCCP}.

Turning to part (b), unwrapping the recursive lower
bound~\eqref{EqnCCCPDescentStep} and re-arranging yields
inequality~\eqref{EqnCCCPBasic}. Finally, we turn to the proof of
inequality~\eqref{EqnCCCPStrong} under the additional strong convexity
condition. Under this condition, the map $x \mapsto \convmajor(x,\x{k})$ in
equation~\eqref{EqnDefnH} is $\mu$-strongly convex, so that
\begin{align}
\f(\x{k}) - \f(\x{k+1}) \geq \convmajor(\x{k},\x{k}) - \convmajor(\x{k+1},\x{k})
\stackrel{(i)}{\geq} \frac{\mu}{2}\| \x{k} - \x{k+1} \|_2^2,
  \end{align}
where inequality (i) follows from the strong convexity of 
the map $x \mapsto \convmajor(x,\x{k})$
and the fact that ${\grad{\convmajor(\x{k+1},\x{k})} = 0}$. Using this
equation repeatedly, we find that
\begin{align*}
\f(\x{0}) - \fstar \geq \f(\x{0}) - \f(\x{k+1}) & = \sum_{ j = 0 }^{k}
\big \{ \f(\x{j}) - \f(\x{j+1}) \big \} \\
& \geq \frac{\mu}{2} \sum_{ j = 0 }^{k} \| \x{j} - \x{j+1} \|_2^2 \\
& = \frac{\mu (k+1)}{2} \avg{ \| \x{k} - \x{k+1} \|_2^2 }.
\end{align*}
Rearranging the last inequality yields the
bound~\eqref{EqnCCCPStrong}.  Finally, let us reiterate that bounds
similar to~\eqref{EqnCCCPStrong} are known in the literature; see the
paper~\cite{Lanckreit09} for example. We provide the proof of
bound~\eqref{EqnCCCPStrong} for completeness.


\section{Proof of Theorem~\ref{ThmProx}}
\label{AppThmProx}


This proof shares some important steps with Theorem~\ref{ThmGradient},
but it requires a more refined argument due to the presence of a
non-smooth and non-continuous function $\nonsmoothf$.  We start by
stating an auxiliary lemma that underlies the proof of
Theorem~\ref{ThmProx}. In the proof, the subgradients of the convex
functions $\ftwo$ and $\nonsmoothf$ at a point $\x{k}$ are denoted by
$u^k$ and $v^k$, respectively.
\begin{lems}
\label{LemProxDescentStep}
Under the conditions of Theorem~\ref{ThmProx}, we have
   \begin{subequations}
   \begin{align}
       \x{k+1} = \x{k} - \alpha(\grad\fone(\x{k}) + v^{k+1} - u^{k} ),
       \qquad \mbox{and}
       \label{EqnProxUpdate}\\
       \f(\x{k}) - \f{(\x{k+1})} \geq \frac{1}{2\alpha} \| \x{k} - \x{k+1} \|_2^2
     \label{EqnDescentCondProx},
   \end{align} 
   \end{subequations}
valid for all $k = 0, 1, 2, \ldots$.  Furthermore, for any convergent
subsequence $\big\{ \x{k_j} \big\}_{j \geq 0}$ of the sequence $\big\{
\x{k} \big\}_{k \geq 0}$ with $\x{k_j} \rightarrow \bar{x}$, we have
   \begin{align*}
   \underset{ j \rightarrow \infty }{ \text{lim} } \;\;
   \nonsmoothf(\x{k_j+1}) = \nonsmoothf(\bar{x}).
   \end{align*}  
 \end{lems}
\noindent See Appendix~\ref{AppLemProxDescentStep} for the proof 
of this lemma.\\

\noindent We now prove Theorem~\ref{ThmProx} using
Lemma~\ref{LemProxDescentStep}.
 \paragraph{Convergence of function value:}
 We first prove that the sequence of function values
 $\{ f(\x{k}) \}_{k \geq 0}$
is convergent. Since $\fstar
\mydefn \min \limits_{x \in \mathbb{R}^{\usedim}}
f(x)$ is finite by assumption, the sequence $\{ f(\x{k}) \}_{k
\geq 0}$ is bounded below. If $\x{k} = \x{k+1}$ for some $k$, 
the convergence of
the sequence $\big\{ f(\x{k}) \big\}_{k \geq 0}$ is trivial. Hence, we
may assume without loss of generality that 
$\x{k} \neq \x{k+1}$ for all $k = 0,1,2,..$. In that case,
inequality~\eqref{EqnDescentCondProx}  ensures that $f(\x{k}) >
f(\x{k+1})$, and consequently, there must 
exist some scalar $\bar{f}$ such that $\lim
\limits_{k \rightarrow \infty} f(\x{k}) = \bar{f}$.


\paragraph{Stationarity of limit points:} Next, we 
establish that any limit point 
of the sequence $\big\{ \x{k} \big\}_{k\geq 0}$ must be
stationary. Consider a subsequence $\big\{ \x{k_j}
\big\}_{j\geq 0}$ such that $\x{k_j} \rightarrow \bar{x}$. Let
$\big\{ v^{k_j} \big\}_{j \geq 0}$ and $\big\{
u^{k_j} \big\}_{j \geq 0}$ be the associated sequence of
subgradients. It suffices to exhibit subgradients $\bar{v}\in \partial
\nonsmoothf(\bar{x})$ and $\bar{u}\in \partial \ftwo(\bar{x})$ such
that, $\grad \fone(\bar{x}) + \bar{v} - \bar{u} = 0$.\\

\noindent \underline{Step 1:} \emph{Existence of subgradient
  $\bar{u}$}: Since the sequence $\big\{ \x{k_j}
\big\}_{j \geq 0}$ is convergent, we may assume that the
sequence $\big\{ \x{k_j} \big\}_{j \geq 0}$ is bounded,
and it lies in a compact set $\CompactSet$. The function $\ftwo$ is
convex continuous, and we have that 
$\ftwo(\x{k_j}) \rightarrow \ftwo(\bar{x})$, 
and the subgradient sequence
$\big\{ u^{k_j} \big\}_{j \geq 0}$ is bounded; see
example 9.14 in the book~\cite{rockafellar2009variational}.  Passing
to a subsequence if necessary, we may assume that 
the sequence $\big\{
u^{k_j} \big\}_{j \geq 0}$ converges to $\bar{u}$. 
Putting together these pieces,
we conclude that
$(\x{k_j}, u^{k_j}, \ftwo(\x{k_j})) \rightarrow 
(\bar{x}, \bar{u}, \ftwo(\bar{x}))$ as $j \rightarrow
\infty$; consequently,
the graph continuity of limiting sub-differentials guarantees
that $\bar{u} \in \partial \ftwo(\bar{x})$ (see
Appendix~\ref{AppFrechetSubdiff} for graph continuity). \\

\noindent \underline{Step 2:} \emph{Existence of subgradient
  $\bar{v}$}: In order to complete the proof, it suffices to show that
the vector \mbox{$\bar{v} \mydefn -\grad \fone(\bar{x}) + \bar{u}$}
belongs to the subgradient set $\partial \nonsmoothf (\bar{x})$.
Since the norm of successive difference ${\| \x{k_j} - \x{k_j+1}
  \|_2}$ converges to zero, Lemma~\ref{LemProxDescentStep} yields
  $\| \grad \fone(\x{k_j}) + v^{k_j+1} - u^{k_j} \|_2
\rightarrow 0$, and $\x{k_j+1} \rightarrow \bar{x}$. Furthermore,
continuity of the gradient $\grad \fone$ yields $\grad \fone(\x{k_j})
\rightarrow \grad\fone(\bar{x})$, and step 1 above guarantees $u^{k_j}
\rightarrow \bar{u}$.  Combining these two facts with $\| \grad
\fone(\x{k_j}) + v^{k_j+1} - u^{k_j} \|_2 \rightarrow 0$, we obtain
$v^{k_j+1} \rightarrow \bar{v} := -\grad \fone(\bar{x}) +
\bar{u}$, and by Lemma~\ref{LemProxDescentStep}, we have
$\nonsmoothf(\x{k_j+1}) \rightarrow \nonsmoothf(\bar{x})$.
Putting together the pieces, we conclude that 
$(\x{k_j+1}, v^{k_j+1}, \nonsmoothf(\x{k_j+1})) \rightarrow
(\bar{x}, \bar{v}, \nonsmoothf(\bar{x}))$. Consequently,
the graph continuity of limiting subdifferentials
guarantees that $\bar{v} \in \partial \nonsmoothf(\bar{x})$ (see
Appendix~\ref{AppFrechetSubdiff} for graph continuity). \\

\noindent Finally, the subgradients $\bar{u} \in \partial \ftwo(\bar{x})$ and
$\bar{v} \in \partial \nonsmoothf(\bar{x})$ obtained from steps 1 and
and 2 respectively satisfy the relation \mbox{$\grad \fone(\bar{x}) + \bar{v}
  - \bar{u} = 0$}, which establishes the claimed stationarity of $\bar{x}$.


\paragraph{Establishing the bound~\eqref{EqnAlgoProxDiffBound}:}

Next, we establish the claimed bound~\eqref{EqnAlgoProxDiffBound} on
the averaged squared successive difference. Recalling that $\fstar
\mydefn \min \limits_{x \in \mathbb{R}^{\usedim} } f(x)$ is finite, we
have
\begin{align}
f(\x{0}) - \fstar \geq f(\x{0}) - f(x^{k+1}) & =
\sum_{j=0}^{k} f(x^{j}) - f(x^{j+1}) \nonumber \\
& \stackrel{(i)}{\geq} \frac{1}{2\alpha} \sum_{j=0}^{k} \|
\x{j} - \x{j+1} \|_2^2 \nonumber  \\ 
& = \frac{(k+1)}{2\alpha} 
\avg{ \|\x{k} - \x{k+1} \|_2^2 },
\label{EqnProxDiffBoundStep}
\end{align}
where step (i) follows from equation~\eqref{EqnDescentCondProx}.
Rearranging the last inequality
yields the claimed bound~\eqref{EqnAlgoProxDiffBound}
on the averaged squared successive difference.


\paragraph{Establishing the bound~\eqref{EqnAlgoProxGradBound}:}
 In order to establish the bound~\eqref{EqnAlgoProxGradBound} on
the averaged squared gradient, we start
by establishing the following upper bound on the gradient-norm
$\| \grad \f(\x{k+1}) \|_2$:
\begin{align}
\label{EqnProxGradUB1}
  \| \grad \f(\x{k+1}) \|_2 \leq 
  \big(\M_{\fone} + \M_{\ftwo} + 
  \frac{1}{\alpha}\big) \| \x{k} - \x{k+1} \|_2. 
\end{align}
Recall that the function $\ftwo$ is $\M_{\ftwo}$
smooth by assumption, and  we have
\begin{align*}
\| \grad \fone(\x{k+1}) - \grad \ftwo(\x{k+1}) + v^{k+1} \|_2 &
\stackrel{(i)}{=} \| \grad \fone(\x{k+1}) - \grad \ftwo(\x{k+1}) +
\big( \grad \ftwo(\x{k}) - \grad \fone(\x{k}) + \frac{1}{\alpha} \big(
\x{k} - \x{k+1} \big) \big) \|_2 \\ 
& \stackrel{(ii)}{\leq} \| \grad{\fone} (\x{k}) -
\grad \fone(\x{k+1}) \|_2 + \| \grad{\ftwo} (\x{k}) - \grad
\ftwo(\x{k+1}) \|_2 + \frac{1}{\alpha} \| \x{k} - \x{k+1} \|_2 \\ &
\stackrel{(iii)}{\leq} \big( \M_{\fone} + \M_{\ftwo} + \frac{1}{\alpha}
\big) \| \x{k} - \x{k+1} \|_2.
\end{align*}
Here step (i) follows from the update equation of $\x{k+1}$ in
 Lemma~\ref{LemProxDescentStep} 
 and from differentianility of the function $\fone$; 
step (ii) follows from triangle inequality, and
step (iii) follows from the smoothness of the functions  $\fone$ and
$\ftwo$. Putting together
 the bounds~\eqref{EqnProxGradUB1} 
 and~\eqref{EqnProxDiffBoundStep}, 
 we obtain
 the desired bound~\eqref{EqnAlgoProxGradBound}.


\subsection{Proof of Lemma~\ref{LemProxDescentStep}}
\label{AppLemProxDescentStep}

Here we prove the claims of Lemma~\ref{LemProxDescentStep}.

\paragraph{Establishing update equation~\eqref{EqnProxUpdate}:}
 Recalling the convex majorant defined in
equation~\eqref{EqnCvxMajorGrad}, we define a convex majorant
$\convmajor(\cdot,\x{k})$ of the function $\f$ as follows:
\begin{align}
\label{EqnCvxMjrProx}
\convmajor(x,\x{k} ) = \fone(\x{k}) - \ftwo(\x{k}) +
\inprod{\grad{\fone(\x{k}) } - \gradftwo}{x - \x{k}} +
\frac{1}{2\alpha}\| x - \x{k} \|_2^2 + \nonsmoothf(x),
\end{align}
where subgradient $\gradftwo \in \partial \ftwo(\x{k})$,
and the step size $\alpha$ 
satisfies  $0 < \alpha \leq
\frac{1}{\M_{\fone}}$. Observe that minimizer
of the convex function 
$ x \mapsto \convmajor(x,\x{k} )$ over $x \in \real^{\usedim}$ 
is same as 
$\text{prox}_{1/\alpha}^{\nonsmoothf} \big( \x{k}
 - \alpha( \grad \fone(\x{k}) - u^{k} ) \big) $,
which implies that $\x{k+1}$ is a minimizer of
the convex function
$x \mapsto\convmajor(x,\x{k} )$ over $x \in \real^{\usedim}$.
Consequently, the optimality condition of 
$\x{k+1}$ guarantees that there exists subgradient 
$\gradg{k+1} \in \subgrad{g(\x{k+1})}$ satisfying the following equation: 
\begin{align}
\label{EqnProxOptCond}
\grad{\fone(\x{k})} - u^{k} + v^{k+1} + \frac{1}{\alpha} \big(
\x{k+1} - \x{k} \big) = 0.
\end{align}
Rewriting the above equation yields the update 
equation~\eqref{EqnProxUpdate}.
\paragraph{Establishing the descent step~\eqref{EqnDescentCondProx}:} Note that
\begin{align}
\f(\x{k}) - \convmajor(\x{k+1},\x{k}) & \stackrel{(i)}{ \geq}
\fone(\x{k}) - \ftwo(\x{k}) + \nonsmoothf(\x{k+1}) +
\inprod{\gradg{k+1}} {\x{k} - \x{k+1}} - \convmajor(\x{k+1},\x{k})
\nonumber \\
& \stackrel{(ii)}{\geq} \inprod{\grad{\fone(\x{k})} - \gradftwo +
  \gradg{k+1}}{\x{k} - \x{k+1}} - \frac{1}{2\alpha}\| \x{k} - \x{k+1}
\|_2^2 \nonumber \\
& \stackrel{(iii)}{\geq} \frac{1}{2\alpha}\| \x{k}- \x{k+1} \|_2^2.
\end{align}
Here step (i) \mbox{follows}
from the convexity of the function $\nonsmoothf$; step (ii) follows by
substituting $\convmajor(\x{k+1},\x{k})$ from equation~\eqref{EqnCvxMjrProx}.
 In step (iii), we  use the
relation \mbox{$\grad{\fone(\x{k})} - \gradftwo + \gradg{k+1} =
  \frac{1}{\alpha} \big( \x{k} - \x{k+1} \big)$}, which follows from
equation~\eqref{EqnProxOptCond}. 
Finally, recall that the function
$ x \mapsto \convmajor(x,\x{k})$ is a
majorant for the function $\f$, and we deduce that
\begin{align}
  \f(\x{k}) - \f(\x{k+1}) &\geq f(\x{k}) - q(\x{k+1},\x{k}) \nonumber
  \\
\label{EqnProxDescentStep}  
  & \geq \frac{1}{2\alpha}\| \x{k} - \x{k+1} \|_2^2.
\end{align}
\paragraph{Limit of the sequence 
$\big\{ \nonsmoothf(\x{k_j+1}) \big\}_{j \geq 0}$:}
Consider any convergent subsequence $\big\{ \x{k_j} \big\}_{j \geq 0}$ 
of the sequence $\big\{ \x{k} \big\}_{k \geq 0}$  with 
${\x{k_j} \rightarrow \bar{x}}$. 
Recall that $\f^{*} = \inf_{x \in \real^{\usedim}} \f(x)$
is finite by assumption; combining this with 
 step~\eqref{EqnDescentCondProx} in
Lemma~\ref{LemProxDescentStep}, we have that
$\| \x{k} - \x{k+1} \|_2 \rightarrow 0$, 
and that $\x{k_j + 1} \rightarrow \bar{x}$. 
The function  $\nonsmoothf$  is lower semi-continuous, and we have
 \begin{align}
 \label{EqnProxLimInf}
 \underset{ j \rightarrow \infty }{\text{lim inf}} \;\;
 \nonsmoothf(\x{k_j+1}) \geq \nonsmoothf(\bar{x}).
\end{align}
Since we already proved $\x{k_j+1}$ is a minimizer of the  
convex function $x \mapsto \convmajor(x,\x{k_j})$,
we have ${\convmajor(\x{k_j+1},\x{k_j} ) \leq \convmajor(\bar{x},\x{k_j})}$.
Unwrapping the last inequality and taking $\text{lim sup}$ yields 
\begin{align}
  \underset{ j \rightarrow \infty }{ \text{lim sup} }\;\;
  \nonsmoothf(\x{k_j+1}) & \stackrel{(i)}{\leq} \nonsmoothf(\bar{x}) +
  \underset{ j \rightarrow \infty }{ \text{lim sup} }\;\; \Big(
  \inprod{\bar{x} - \x{k_j}}{\grad \fone(\x{k_j}) - u^{k_j}} +
  \frac{1}{ 2 \alpha} \| \x{k_j} - \bar{x} \|_2^2 \Big) \nonumber \\ &
  \stackrel{(ii)}{=} \nonsmoothf(\bar{x}) \label{EqnProxLimSup}.
\end{align}
Here step (i) holds since $\|\x{k_j} - \x{k_j+1}\|_2 \rightarrow
0$, and the sequence $\big\{ \grad{\fone(\x{k_j})} \big\} - u^{k_j} \big
\}_{j \geq 0}$ is bounded---which we prove shortly; step (ii) above follows from $\x{k_j} 
\rightarrow \bar{x}$ and boundedness of the sequence $\big \{ \grad{\fone(\x{k_j})} -
u^{k_j} \big \}_{j \geq 0}$. Combining
equations~\eqref{EqnProxLimInf} and~\eqref{EqnProxLimSup} we obtain
the claimed result.
\paragraph{Boundedness of the sequence $\big\{ \grad
 \fone(\x{k_j}) - u^{k_j} \big\}_{ j \geq 0}$:} 
In order to prove the boundedness of the sequence 
$\big\{ \grad
 \fone(\x{k_j}) - u^{k_j} \big\}_{ j \geq 0}$,
 it suffices to show that the gradient sequence 
 $\big\{ \grad
 \fone(\x{k_j}) \big\}_{ j \geq 0 }$ and the 
 sub-gradient sequence $\big\{ u^{k_j} 
 \big\}_{j \geq 0}$ are bounded. Recall that
 $\x{k_j} \rightarrow \bar{x}$, 
 and we have that the sequence $\{\x{k_j}\}_{j \geq 0}$ is bounded. 
Consequently, from the smoothness of
the function $\fone$, we find that the sequence $\big\{ \grad
 \fone(\x{k_j}) \big\}_{ j \geq 0 }$ is bounded. 
Finally, note that
the function $\ftwo$ is convex continuous, and
we already argued that the sequence 
$\{\x{k_j}\}_{j \geq 0}$ is bounded.
Combining this with example 9.14
in the book~\cite{rockafellar2009variational}, we conclude that the
subgradient sequence $\big\{ u^{k_j}
\big\}_{j \geq 0}$ bounded. 


\section{Proofs related to Algorithm~\ref{AlgoFW}}

In this appendix, we provide the proof of Theorem~\ref{ThmFW}, which
applies to the Frank-Wolfe based method (Algorithm~\ref{AlgoFW}).
We also provide an upper bound on the
generalized curvature constant $\curvatureconst{\f}$, which is stated
in Lemma~\ref{LemCurvatureUB}.

\subsection{Proof of Theorem~\ref{ThmFW}}
\label{AppThmFW}

Let $\x{\gamma} \mydefn \x{k} + \gamma d^{k}$, where the
difference $d^{k}$ is defined as $d^{k} 
\mydefn s^{k} - \x{k}$, and the vector $s^{k}$ is
the Frank-Wolfe direction defined in \mbox{Algorithm~\ref{AlgoFW}}.
Unpacking the
definition~\eqref{EqnCurvConst} of the generalized curvature constant
$\curvatureconst{\f}$, we find that for any scalar 
$\gamma \in (0,1)$ and subgradient $\gradftwo
\in \subgrad{\ftwo(\x{k})}$, we have the following:
\begin{align}
f(\x{\gamma}) & \leq f(\x{k}) + \gamma \inprod{\grad{\fone}(\x{k}) -
  \gradftwo}{d^{k}} + \frac{\gamma ^{2}}{2} \curvatureconst{\f}
\nonumber \\
\label{EqnFWDescentStep}
& \stackrel{(i)}{\leq} f(\x{k}) - \gamma \fwgap{k} +
\frac{\gamma^{2}}{2} C_0.
\end{align}
Here inequality (i) is obtained by substituting $\fwgap{k} =
\inprod{d^{k}}{\gradftwo - \grad{\fone}(\x{k})}$ and using $C_0 \geq
\curvatureconst{\f}$. Substituting $\gamma = \gamma^{k} \mydefn \min \big \{
\frac{\fwgap{k}}{C_0},1 \big \}$ in
equation~\eqref{EqnFWDescentStep} yields
\begin{align}
f(x^{k+1}) \leq f(x^{k}) - \min \big\{ \frac{(\fwgap{k})^2}{2
  C_0}, \fwgap{k} - \frac{C_0}{2}\mathbbm{1}_{ \big \{ \fwgap{k}
  > C_0 \big \} } \big\}, \label{EqnFWDescentStep2}
\end{align}
where  $\x{k+1} = \x{k} + \gamma^{k} d^k$.   Let $\minfwgap{k} \mydefn
\min_{0\leq \j \leq k} \fwgap{\j}$ denote the minimum FW gap up to iteration $k$,
then repeated application of equation~\eqref{EqnFWDescentStep2} yields
\begin{align}
\f(\x{0}) - \f(\x{k+1}) & \geq \sum_{\j=0}^{k} \min \big\{
\frac{(\fwgap{\j})^2}{2C_0}, \fwgap{\j} - \frac{C_0}{2}\mathbbm{1}_{
  \big \{ \fwgap{\j} > C_0 \big \} } \big\}
\nonumber \\ & \geq (k+1) \min \big\{ \frac{( \minfwgap{k}
  )^2}{2C_0}, \minfwgap{k} - \frac{C_0}{2} \mathbbm{1}_{ \big \{
  \minfwgap{k} > C_0 \big \} } \big\} \label{EqnFWDescentStep3}.
\end{align}
Rewriting the last equation yields the following upper bound
\begin{align*}
\min \big\{ \frac{( \minfwgap{k} )^2}{2C_0},\minfwgap{k} -
\frac{C_0}{2} \mathbbm{1}_{ \big \{ \minfwgap{k} > C_0
  \big \} } \big\} \stackrel{(i)}{\leq} \frac{ \f(\x{0}) - \fstar }{k+1},
\end{align*}
where step (i) follows from the lower bound $\f(\x{k+1}) \geq
\fstar \mydefn \min_{x \in \range}\;\f(x)$.
Considering the cases where $\minfwgap{k} \leq C_0$ and $\minfwgap{k}>
C_0$ separately, it can be shown following Lacoste-Julien~\cite{Julien16} that
\begin{align*}
\minfwgap{k}\leq
\begin{cases}
\frac{2 (\f(\x{0}) - \fstar)}{\sqrt{k+1}}\;\;\;\;\;\;\;\;
\text{for}\;\; k+1 \leq \frac{2 (\f(\x{0}) -
  \fstar)}{C_0}\\ \;\;\\ \sqrt{\frac{2C(\f(\x{0}) - \fstar)}{k+1}}\;\;
\text{ \; otherwise }.\\
\end{cases}
\end{align*}
Finally, note that $\sqrt{2C_0(\f(\x{0}) - \fstar)} \leq
\max\{2(\f(\x{0}) - \fstar), C_0 \}$ and we conclude that
\begin{align*}
  \minfwgap{k} \leq\frac{\max\big\{ 2 \big( \f(\x{0}) - \fstar
    \big),C_0 \big\} }{\sqrt{k+1}}.
\end{align*}


\subsection{Upper bound on generalized curvature constant}

In this section, we provide an upper bound on the generalized
curvature constant $\curvatureconst{\f}$, where the function $\f$ is a
difference of a differentiable function $\fone$ and a continuous
function $\ftwo$. For better readability, we use
$\curvatureconst{\fone - \ftwo}$ instead of $\curvatureconst{\f}$ in
the following lemma.
\begin{lems}
  \label{LemCurvatureUB}
Suppose that the function $\fone$ is continuously differentiable and
function $\ftwo$ is convex, then we have $\curvatureconst{\fone -
  \ftwo} \leq \curvatureconst{\fone}$.  Furthermore, if the function
$\fone$ is $\M_{\fone}$-smooth, and the function $\ftwo$ is a $\mu$
strongly convex function with $0 \leq \mu < \M$, then
\begin{align}
\label{EqnCurvUB}
  \curvatureconst{\fone - \ftwo} \leq \big( \M - \mu\big)
  \times\big(\diam_{\| \cdot \|_2}(\range) \big)^{2},
\end{align}
  where $\diam_{\| \cdot \|_2 }$ denote the diameter of the set
  $\range$, measured in $\ell_2$ norm.
  \end{lems}
\paragraph{Comments:} The first upper bound on
$\curvatureconst{\fone - \ftwo}$ in Lemma~\ref{LemCurvatureUB} posits
that the curvature constant of the difference function $\fone - \ftwo$ is
upper bounded by curvature constant of the function $\fone$, whenever
the second function~$\ftwo$ is convex. Let us try to understand
an implication of
this result through an example. One of the well-known upper bound of
curvature constant for $\M_{\fone}$-smooth 
function $\fone$ is $\M_{\fone}
\times\big(\diam_{\| \cdot \|_2}(\range) \big)^{2}$; see the paper by
Jaggi~\cite{jaggi-11}.  Now consider continuously differentiable
functions $\fone$ and $\ftwo$ such that the function $\fone$ is
$\M_{\fone}$-smooth and the function
$\ftwo$ is non-smooth and convex. It
can be verified that the difference function $\fone - \ftwo$
is \emph{not} smooth in this case; consequently, the earlier
bound on curvature constant
$\curvatureconst{\fone - \ftwo}$ is $\infty$, whereas
Lemma~\ref{LemCurvatureUB} ensures that
\begin{align*}
\curvatureconst{\fone - \ftwo} \leq \curvatureconst{\fone } \leq
\M_{\fone} \times\big(\diam_{\| \cdot \|_2}(\range) \big)^{2}.
\end{align*}

\paragraph{Proof of the upper bound
  $\curvatureconst{\fone - \ftwo} \leq \curvatureconst{\fone}$:}
Unwrapping the definition of $\curvatureconst{\fone - \ftwo}$,
we have
\begin{align}
\curvatureconst{\fone - \ftwo} & = \sup_{ \substack{ x,y \in
    c_{\gamma} \\ u \in \subgrad{\ftwo(x)} } }\frac{2}{\gamma^{2}}
\big[ f(y) - f(x)- \inprod{y - x}{\grad{\fone(x)} - u}\big] \nonumber \\
\label{EqnDefnCurvConst} 
& = \sup_{ \substack{ x,y \in c_{\gamma} \\ u \in \subgrad{\ftwo(x)} }
}\frac{2}{\gamma^{2}}\big[ \fone(y) - \fone(x)-
  \inprod{y - x}{\grad{\fone(x)}} - \Delta_{\ftwo}(y,x,u) \big] \\
& \stackrel{(i)}{\leq} \underbrace{\sup_{ \substack{ x,y \in
      c_{\gamma} } }\frac{2}{\gamma^{2}}\big[ f(y) - f(x) -
    \inprod{y - x}{\grad{\fone(x)}} \big]}_{\curvatureconst{\fone}},
\nonumber
\end{align}
where $\Delta_{\ftwo}(y,x,u) \mydefn \ftwo(y) - \ftwo(x)-
\inprod{y - x}{u}$. Here inequality (i) follows by noting
that, for any pair of points $x,y \in
\range$, and for any convex function $\ftwo$ with $u \in \partial
\ftwo(x)$, we have $\Delta_{\ftwo}(y,x,u) \geq 0$ .


\paragraph{Proof of upper bound~\eqref{EqnCurvUB}:}

Suppose in addition, the function $\fone$ is $\M_{\fone}$-smooth, and
the function $\ftwo$ is $\mu$-strongly convex with $\mu \geq 0$. Then
we have ${\Delta_{\ftwo}(y,x,u) \geq \frac{\mu}{2}\| x - y\|_2^2}$,
and equation~\eqref{EqnDefnCurvConst} yields
\begin{align*}
\curvatureconst{\fone - \ftwo} & \leq \sup_{ \substack{ x,y \in
    c_{\gamma} } }\frac{2}{\gamma^{2}}\big[ \fone(y) - \fone(x) -
  \inprod{y - x}{\grad{\fone(x)}} - \frac{\mu}{2}\| x - y\|_2^2 \big] \\
& \stackrel{(i)}{\leq} \sup_{ \substack{ x,y \in c_{\gamma} }
}\frac{2}{\gamma^{2}}\bigg[ \frac{\M_{\fone} - \mu}{2}\| x - y\|_2^2
  \bigg],
\end{align*}
where step (i) follows since the function $\fone$ is 
$\M_{\fone}$-smooth.
Substituting $y - x = \gamma s$ with $s \in \range$, we obtain
the claimed upper bound
\begin{align*}
\curvatureconst{\fone - \ftwo} \leq (\M_{\fone} - \mu) \times
\big(\diam_{\| \cdot \|_2}(\range) \big)^{2}.
\end{align*}


\section{Proofs of faster rates under Assumption KL}

In this appendix, we prove our results on improved convergence rates
for functions which satisfy Assumption KL---as stated in
Theorems~\ref{ThmGradSubanal} and~\ref{ThmProxSubanal}.  We begin by
stating an auxiliary lemma that underlies the proofs of
Theorems~\ref{ThmGradSubanal} and~\ref{ThmProxSubanal}.

\begin{lems}
\label{LemUnifKL_Ineq}
Under assumptions of either Theorem~\ref{ThmGradSubanal} or
Theorem~\ref{ThmProxSubanal}, there exists constants $\theta \in
[0,1)$, $\C{} > 0$ and positive integer $k_1$ such that for all $k
  \geq k_1$, we have
\begin{align*}
|\f(\x{k}) - \bar{f}|^{\theta} \leq \C{}\| \grad \f(\x{k}) \|_2,
\end{align*} 
where $f(\x{k}) \downarrow \bar{f}$. Furthermore, if $\x{k}
\rightarrow \bar{x}$, then the parameters $(\theta,C)$, obtained from
KL-inequality of the function $\f$ at the point $\bar{x}$, satisfy the
above inequality.
\end{lems}
\noindent See Appendix~\ref{EqnUnifKL} for the proof of this lemma.


\subsection{Proof of Theorem~\ref{ThmGradSubanal}}
\label{AppThmGradSubanal}

Now we prove Theorem~\ref{ThmGradSubanal} using
Lemma~\ref{LemUnifKL_Ineq}.


\paragraph{Convergence of the sequence $\big\{ \x{k} \big\}_{k \geq 0}:$}

We demonstrate the convergence of the sequence $\big\{ \x{k} \big\}_{k
  \geq 0}$ by proving that the sequence has finite length property;
  more precisely, we show that
 ${\sum_{k = 0}^{\infty} \| \x{k} - \x{k+1} \|_2 <
  \infty}$. First, note that for any scalar $0 \leq \theta < 1$, the
function $t \mapsto t^{1 - \gamma\theta}$ is concave for $0 < \gamma
< \frac{1}{\theta} $; consequently, for iteration $k \geq k_1$ we
have
\begin{align}
    \big( f(\x{k}) - \bar{f}\big)^{1 - \gamma\theta} - \big( f(\x{k+1})
    - \bar{f}\big)^{1 - \gamma\theta} &\stackrel{}{\geq} \big(
    1 - \gamma\theta \big) \big( f(\x{k}) - \bar{f} \big)^{-\gamma
      \theta} \big( f(\x{k}) - f(\x{k+1}) \big) \nonumber \\ &
    \stackrel{(i)}{\geq} \big( 1 - \gamma\theta \big) \big( | f(\x{k})
    - \bar{f} | \big)^{-\gamma\theta} \times \frac{1}{2 \alpha}\|
    \x{k} - \x{k+1} \|_2^2\nonumber \\ & \stackrel{(ii)}{\geq}
    \frac{(1 - \gamma\theta)}{ \C{} \| \grad{\f(\x{k})} \|_2^{\gamma} }
    \times \frac{1}{2\alpha}\| \x{k} - \x{k+1} \|_2^2 \nonumber \\ &
    \stackrel{(iii)}{=} \frac{(1 - \gamma\theta)}{2 \C{}
      \alpha^{1 - \gamma}} \| \x{k} - \x{k+1} \|_2^{2-\gamma}
    \label{EqnGradDescentStep2}.
  \end{align}
Here inequality (i) follows from the descent property in
equation~\eqref{EqnDescentCondGrad} and from the fact that $f(\x{k})
\downarrow \bar{f}$. Inequality (ii) follows from
Lemma~\ref{LemUnifKL_Ineq}, and equality (iii) follows from the
relation ${\x{k} -\x{k+1} = \alpha \big( \grad{\fone(\x{k})} -
  \gradftwo \big) = \alpha \grad \f(\x{k})}$.  Substituting $\gamma =
1$ and summing both side of inequality~\eqref{EqnGradDescentStep2}
from index $k =k_1$ to $k = \infty$, we obtain
\begin{align*}
 \big( f \big( \x{k_1} \big) - \bar{f} \big)^{1 - \theta } & = \sum_{k
   = k_1 }^{ \infty } \big( f \big( \x{k} \big) - \bar{f} \big)^{1 -
   \theta } - \big( f \big( \x{k+1} \big) - \bar{f} \big)^{1 - \theta
 } \nonumber \\ & \geq \sum_{k = k_1 }^{ \infty } \frac{(1- \theta)}{2
   \C{} } \| \x{k} - \x{k+1} \|_2,
\end{align*}
which proves the finite length property of the sequence $\big\{ \x{k}
\big\}_{k \geq 0}$.  Consequently, we are guaranteed to have a vector
$\bar{x}$ such that $\x{k} \rightarrow \bar{x}$ as $k \rightarrow
\infty$. \\

\paragraph{Rate of convergence of $\avg{\| \grad \f(\x{k}) \|_2}$:}
Rewriting equation~\eqref{EqnGradDescentStep2}, we have the following:
\begin{align}
  C_{\gamma} & \mydefn \sum_{\ell = 0 }^{ k_1 } \frac{(1-
    \gamma\theta)}{2 \C{} \alpha^{1 - \gamma } } \| \x{\ell} -
  \x{\ell+1} \|_2^{2-\gamma} + \big( f \big( \x{k_1} \big) - \bar{f}
  \big)^{(1 - \gamma\theta) } \nonumber \\ & \stackrel{(i)}{\geq}
  \sum_{ \ell = 0 }^{ k-1 } \frac{(1- \gamma\theta)}{2 \C{} \alpha^{1
      -\gamma } } \| \x{\ell} - \x{\ell+1} \|_2^{2-\gamma}
  \nonumber\\ & = \frac{k (1 - \gamma\theta) }{2C \alpha^{1 - \gamma }}
  \avg{ \| \x{k} - \x{k+1} \|_2^{2-\gamma} } \label{EqnKeyEqn},
\end{align}
where step (i) above follows from
equation~\eqref{EqnGradDescentStep2}, and ${\avg{\| \x{k} - \x{k+1}
    \|_2^{2-\gamma}}:= \frac{1}{k}\sum_{\ell = 0}^{k-1} \| \x{\ell} -
  \x{\ell+1} \|_2^{2-\gamma}}$ denote the running arithmetic
average. Since $0 \leq \theta < 1$, we can take $\gamma = 1$ in
equation~\eqref{EqnKeyEqn}, and we obtain the following rate:
\begin{align*}
\avg{\| \grad \f(\x{k}) \|_2} = \frac{1}{\alpha} \avg{\| \x{k} -
  \x{k+1} \|_2} \leq \frac{c_{1}}{k},
\end{align*}
where $c_{1} = \frac{2 \C{} C_{\gamma}}{\alpha (1 -
  \theta)}$. Finally, note that the last equality holds trivially for
iteration $k \leq k_1$ with the given choice of the constant $c_1$.

\paragraph{Rate of convergence of \Gavg{\| \grad \f(\x{k}) \|_2}: }

Since we proved that the sequence $\big\{ \x{k} \big\}_{k \geq 0}$ is
convergent to the point $\bar{x}$, we have that the parameter
$\theta$ in
Lemma~\ref{LemUnifKL_Ineq} can be taken to be the KL-exponent of the
function $\f$ at point $\bar{x}$. Suppose $\frac{1}{2} \leq \theta <
\frac{ \rate }{ 2\rate - 1 } $, then substituting $\gamma =
\frac{2\rate-1}{\rate}$ in equation~\eqref{EqnKeyEqn} yields,
\begin{align*}
\Gavg{ \| \grad \f(\x{k}) \|_2 } & = \frac{1}{\alpha} \Gavg{ \| \x{k}
  - \x{k+1} \|_2 } \\
& \stackrel{(i)}{\leq} \frac{1}{\alpha} \big\{ \avg{ \| \x{k} -
  \x{k+1} \|_2^{\frac{1}{r}} } \big\}^{r} \nonumber \\
& \stackrel{(ii)}{\leq} \frac{c_2}{k^r},
\end{align*}
where $c_2 = \frac{1}{\alpha} \big(\frac{2 \C{} C_{\gamma}
  \alpha^{1 - \gamma\theta} }{1 - \gamma\theta} \big)^r $ with $\gamma
= \frac{2r -1}{r}$, and $\Gavg{\| \x{k} - \x{k+1} \|_2^{2-\gamma}}
\mydefn \prod_{\ell = 0}^{k - 1} \big(\| \x{\ell} - \x{\ell+1}
\|_2\big)^{\frac{1}{k}}$, the geometric average of the sequence
$\big\{ \| \x{\ell} - \x{\ell+1} \|_2 \big\}_{l = 0}^{k-1}$. Here step
(i) above follows from arithmetic-geometric mean (AM/GM) inequality;
step (ii) follows from the bound in equation~\eqref{EqnKeyEqn} 
and from the fact that $\gamma = \frac{2r - 1}{r}$.  
Finally, note that the last equality holds
trivially for iteration $k \leq k_1$ with the given choice of constant
$c_2$.


\subsection{Proof of Theorem~\ref{ThmProxSubanal}}
\label{AppThmProxSubanal}

The proof of Theorem~\ref{ThmProxSubanal} builds on the techniques
used in the proof of Theorem~\ref{ThmGradSubanal} but requires additional
technical care due to the presence of possibly non-continuous function
$\nonsmoothf$.


\paragraph{Convergence of the sequence $\big\{ \x{k} \big\}_{k \geq 0}$:}

The proof of Theorem~\ref{ThmProxSubanal} has two steps. First, we
prove a descent condition similar to
equation~\eqref{EqnGradDescentStep2}. We then leverage this descent
condition and weighted AM-GM inequality to obtain the desired result.

\paragraph{Step 1:}
Following the proof of Theorem~\ref{ThmGradSubanal}, we prove the
convergence of the sequence $\big\{ \x{k} \big\}_{k \geq 0}$ by
showing that the sequence $\big\{ \x{k} \big\}_{k \geq 0}$ has finite
length property.  First, note that for scalars $0 \leq \theta < 1$ and
$0 < \gamma < \frac{1}{\theta}$, the function $t \mapsto t^{1 - \gamma
  \theta}$ is concave. Consequently, 
  for iteration $k \geq k_1$, from Lemma~\ref{LemUnifKL_Ineq} we have
\begin{align}
\big( f(\x{k}) - \bar{f}\big)^{1 - \gamma\theta} - \big( f(\x{k+1}) -
\bar{f}\big)^{1 - \gamma\theta} &\stackrel{}{\geq} \big( 1 - \gamma\theta
\big) \big( f(\x{k}) - \bar{f} \big)^{-\gamma\theta} \big( f(\x{k})
- f(\x{k+1}) \big) \nonumber \\
& \stackrel{(i)}{\geq} \big( 1 - \gamma\theta \big) \big( | f(\x{k}) -
\bar{f} | \big)^{-\gamma\theta} \times \frac{1}{2 \alpha}\| \x{k} -
\x{k+1} \|_2^2 \nonumber \\
& \stackrel{(ii)}{\geq} \frac{(1 - \gamma\theta)}{ \C{} \|
  \grad{\f(\x{k})} \|_2^{\gamma} } \times \frac{1}{2\alpha}\| \x{k} -
\x{k+1} \|_2^2.
\label{EqnLojaDescent}
\end{align}
Here step (i) follows from the descent property in
equation~\eqref{EqnProxDescentStep} and from the fact that \mbox{
  $f(\x{k}) \downarrow \bar{f}$ }; step (ii) follows from
Lemma~\ref{LemUnifKL_Ineq}. The function $\ftwo$ is locally smooth by
assumption; as a result, we have that the difference function $\fone -
\ftwo$ is locally smooth.  We also assumed  that the sequence $\big\{ \x{k}
\big\}_{k \geq 0}$ is bounded (lies in a compact set $\CompactSet$);
consequently, we may assume that the difference function $\fone -
\ftwo$ is smooth in the compact set $\CompactSet$ with a smoothness
parameter $\M_{\fone - \ftwo}$(say). Borrowing the argument of
Theorem~\ref{ThmProx} part(b), it follows that:
\begin{align}
\label{EqnProxGradUB}  
\| \grad{\fone(\x{k})} - \grad \ftwo(\x{k}) + v^{k} \|_2 \leq \big(
\M_{\fone - \ftwo} + \frac{1}{\alpha} \big) \| \x{k} - \x{k-1} \|_2.
\end{align}   
Combining the last inequality with inequality~\eqref{EqnLojaDescent}
yields the following descent property
\begin{align}
  \label{EqnProxDescent_KL}  
 \big( f(\x{k}) - \bar{f}\big)^{1 - \gamma\theta} - \big( f(\x{k+1}) -
 \bar{f}\big)^{1 - \gamma\theta} \geq \frac{(1 - \gamma\theta)}{ 2\alpha
   \C{} \big( \M_{\fone - \ftwo} + \frac{1}{\alpha} \big)^\gamma }
 \times \frac{ \| \x{k} - \x{k+1} \|_2^2 }{ \| \x{k} - \x{k-1}
   \|_2^{\gamma} }.
\end{align}

\paragraph{Step 2:}

We now leverage the descent condition obtained from step 1 to prove
finite length property of the sequence $\big\{ \x{k} \big\}_{k \geq
  0}$.  In order to facilitate further discussion, we use
$\deltak{k}{\gamma}$ to denote the following:
\begin{align*}
\deltak{k}{\gamma} \mydefn C_3 \Big( \big( f(\x{k}) -
\bar{f}\big)^{1 - \gamma\theta} - \big( f(\x{k+1}) -
\bar{f}\big)^{1 - \gamma\theta} \Big),
\end{align*} 
where the constant $C_3 \mydefn \frac{ 2\alpha \C{} \big( \M_{\fone -
    \ftwo} + \frac{1}{\alpha} \big)^\gamma }{(1 - \gamma\theta)}$.
With this notation, we can rewrite the
equation~\eqref{EqnProxDescent_KL} as
\begin{align}
\label{EqnProxKeyEqn2}  
\deltak{k}{\gamma} \| \x{k-1} - \x{k}\|_2^{\gamma} \geq \| \x{k} -
\x{k+1} \|_2^2.
\end{align}
Combining equation~\eqref{EqnProxKeyEqn2} with the weighted AM-GM
inequality, we obtain
\begin{align}
\big(1 + \frac{ \gamma }{2-\gamma} \big) \times \sum_{j = k_1+1}^{k}
\| \x{j} - \x{j+1} \|_2^{2 - \gamma} & \stackrel{(i)}{\leq} \big(1 +
\frac{ \gamma }{2-\gamma} \big) \times \sum_{k=k_1+1}^{k} \Big( \sqrt{
  \deltak{j}{\gamma} \; \| \x{j-1} - \x{j} \|_2^{\gamma} } \Big)^{
  \frac{2 -\gamma }{ 2 } } \nonumber \\
& \stackrel{(ii)}{\leq} \sum_{j = k_1+1}^{k} \big( \deltak{j}{\gamma}
  + \frac{\gamma}{2 - \gamma} \| \x{j-1} - \x{j} \|_2^{2 - \gamma} \big)
  \nonumber \\
\label{EqnProxKeyEqn3}  
& \stackrel{(iii)}{\leq} C_3 \big( f(\x{k_1}) - \bar{f}\big)^{1 - \gamma
  \theta} + \sum_{j = k_1+1}^{k} \frac{\gamma}{2 - \gamma} \| \x{j-1}
- \x{j} \|_2^{2 - \gamma}.
\end{align}
Here step (i) follows from equation~\eqref{EqnProxKeyEqn2}, and step
(ii) is implied by applying weighted AM-GM inequality as follows:
\begin{align*}
\frac{\deltak{j}{\gamma} + \frac{\gamma}{2 - \gamma} \| \x{j-1} -
  \x{j} \|_2^{2 - \gamma} }{ 1+ \frac{\gamma}{2 - \gamma} } \geq \Big(
\deltak{j}{\gamma}\|\x{j-1} - \x{j} \|_2^{\gamma}
\Big)^{\frac{1}{1+\frac{\gamma}{2 - \gamma}}}.
\end{align*}
Step (iii) in equation~\eqref{EqnProxKeyEqn3} follows from the
following observation
\begin{align*}
  \sum_{j = k_1}^{k} \deltak{j}{\gamma} & = C_3 \sum_{j = k_1}^{k} \big( f(\x{j}) -
  \bar{f}\big)^{1 - \gamma\theta} - \big( f(\x{j+1}) -
  \bar{f}\big)^{1 - \gamma\theta} \\
  & \leq C_3 \big( f(\x{k_1}) - \bar{f}\big)^{1 - \gamma\theta}.
\end{align*}
Rewriting inequality~\eqref{EqnProxKeyEqn3}, we have for all $k \geq
k_1 + 2$
\begin{align}
  \label{EqnSumFinite}
\sum_{ j = k_1 + 1}^{k-1} \| \x{j} - \x{j+1} \|_2^{2 - \gamma} & \leq
C_3 \big( f(\x{k_1}) - \bar{f}\big)^{1 - \gamma\theta} +
\frac{\gamma}{2 - \gamma} \| \x{k_1} - \x{k_1+1} \|_2^{2 - \gamma} - \big(1
+ \frac{\gamma}{2 - \gamma}\big) \| \x{k} - \x{k+1} \|_2^{2 - \gamma}
\nonumber \\ & \leq C_3 \big( f(\x{k_1}) - \bar{f}\big)^{1 - \gamma
  \theta} + \frac{\gamma}{2 - \gamma} \| \x{k_1} - \x{k_1+1}
\|_2^{2 - \gamma} < \infty.
\end{align}
Finally, by substituting $\gamma = 1$ and letting $k \rightarrow
\infty$ in the last equation, we deduce the finite length property of
the sequence $\big\{ \x{k} \big\}_{k \geq 0}$.

\paragraph{Rate of convergence of $\avg{\| \grad \f(\x{k}) \|_2}$
  and $\Gavg{\| \grad \f(\x{k}) \|_2}$:} The proof of this part
follows from the corresponding proof in Theorem~\ref{ThmGradSubanal}
and using the inequality~\eqref{EqnSumFinite} and upper
bound~\eqref{EqnProxGradUB}.


\subsection{ Proof of Lemma~\ref{LemUnifKL_Ineq}}
\label{EqnUnifKL}
Since the sequence $\big\{ \x{k} \big\} _{k \geq 0}$ is bounded by
  assumption, without loss of generality, we may assume that the set of limit points of the sequence $\big\{ \x{k} \big\} _{k \geq 0}$ --- which we denote by
  $\mathcal{\bar{X}}$ --- is a compact set. From Theorem~\ref{ThmGradient} (respectively
Theorem~\ref{ThmProx}), we have that all the limit points of the sequence $\big\{ \x{k} \big\} _{k \geq 0}$ are critical points of the function $\f$; furthermore, since
${\f(\x{k}) \downarrow \bar{f}}$, we also have 
that the function $\f$ is constant
on the set of limit points $\mathcal{\bar{X}}$, 
and the function value on $\mathcal{\bar{X}}$ equals $\bar{f}$.
Combining this with
Assumption KL, we have for all $z \in \mathcal{\bar{X}}$, there
exists constants $\theta(z) \in [0,1) $, $ \;r_{z} > 0$ and $\C{}(z)
  >0$ such that, ${\mid \f(x) - \bar{f} \mid^{\theta(z)} \le \C{}(z)
  \times \| \grad \f(x) \|_2}$ for all $x \in B(z,r_{z})$. 
  Now, consider the open cover ${\{B(z,r_{z}):z\in \mathcal{\bar{X}} \}}$ of the set $\mathcal{\bar{X}}$. From compactness of the set $\mathcal{\bar{X}}$, we are guaranteed to have a finite subcover;
  more precisely, there exists $\{ z_{1},\ldots z_{p}\} 
  \subseteq \mathcal{\bar{X}}$ such that
  $\mathcal{\bar{X}}
  \subseteq\bigcup_{i=1}^{p}B(z_{i},r_{z_{i}})$. Define constants
  $\theta \mydefn \max\{\theta(z_{i}):1\leq i\leq p\}$, $\C{} \mydefn
  \max\{\C{}(z_{i}):1\leq i\leq p\} $, and ${r \mydefn \min\big\{
  \frac{r_{z_{i}}}{2}:1\leq i\leq p\big\}}$. 
  Utilizing the result $\|\x{k} -\x{k+1}\|_2 \rightarrow 0$ from 
  Theorem~\ref{ThmGradient} (respectively Theorem~\ref{ThmProx}),
  one can show that, there exists positive integer $k_1$
  such that for all $k \geq k_1$ we have $\| \x{k} - \x{k+1} \|_2< \frac{r}{2}$, and $x^{k}\in\bigcup_{i=1}^{p}B(z_{i},r_{z_{i}})$.
 Putting together these pieces, we conclude that
  for all $k\geq k_{1}$
\begin{align*}
x^{k} \in \bigcup_{i=1}^{p}B(z_{i},r_{z_{i}}), \;\;\;\; \mbox{and} \quad
\mid f(x^{k}) - \bar{f} \mid^{\theta}\leq \C{} \| \grad \f \|_2,
\end{align*}
which proves the first part of claimed lemma. Now suppose the sequence
$\big\{ \x{k} \big\}_{k \geq 0}$ converges to a point
$\bar{x}$, then we have that the set of limit points
$\bar{\mathcal{X}} = \{ \bar{x} \}$, is a singleton set.  The rest of
the proof is immediate by repeating the argument so far, with the additional information that $\mathcal{\bar{X}} = \{\bar{x}\}$.


\section{Proofs of Corollaries}

In this appendix, we collect the proofs of Corollaries~\ref{CorTukey},
\ref{CorBestSubset} and~\ref{CorMixDensity} from
Section~\ref{SecApplication}.


\subsection{Proof of Corollary~\ref{CorTukey}}
\label{AppCorTukey}

First, note that in order to apply Theorem~\ref{ThmGradient} and
Theorem~\ref{ThmGradSubanal} to Corollary~\ref{CorTukey}, it is enough
to show that the function $\mu \mapsto \f(\mu)$ is $\M_{\f}$-smooth
(in this example, function $\ftwo \equiv 0$, and hence $\f \equiv
\fone$), and the function $\f$ satisfies Assumption KL. We verify that
Assumption KL is satisfied by proving that the objective function $\f$
in problem~\eqref{EqnTukeyObj} is continuous sub-analytic (see
Appendix~\ref{AppSubAnalKL}). For proving sub-analyticity, we heavily
use the properties mentioned in Appendix~\ref{AppSubanal}. In the
following proof, we assume without loss of generality that $\lambda = 1$.


\paragraph{The function $\f$ is continuous
  sub-analytic:}
First, we show that the function $\ROB$ is sub-analytic. We begin by
observing that $\ROB$ is piecewise polynomial.  Polynomials are
analytic functions and intervals are semi-analytic sets.
Since piecewise analytic functions with semi-analytic pieces are
semi-analytic (hence sub-analytic), we conclude that the function
$\ROB$ is sub-analytic. Now, the function $\mu \mapsto y_i -
\inprod{z_i}{\mu}$ is linear, and hence continuous
sub-analytic. Furthermore, since
continuous sub-analytic functions are
closed under composition, 
we have that the function
$\mu \mapsto \ROB \big(y_i - \inprod{z_i}{\mu} \big)$ 
is sub-analytic. Finally, note
that sub-analytic functions are closed under linear combination, and
we conclude that the function $\f$ is sub-analytic. The continuity of
the function $\f$ is immediate by inspection.

\paragraph{The function $\f$ is smooth:} Since the vectors
$\big \{(z_i, y_i)\big \}_{i=1}^{n}$ are fixed, it suffices to prove that
the function $\ROB$ is smooth. A straightforward calculation shows
that $\ROB$ is continuously differentiable and smooth; in particular,
it has a smoothness parameter $36$ when $\lambda = 1$.\\

\noindent Putting together the pieces, we conclude that
Theorem~\ref{ThmGradient} and Theorem~\ref{ThmGradSubanal} are
applicable for problem~\eqref{EqnTukeyObj}.  Convergence of the
sequence $\big \{ \mu^k \big \}_{k \geq 0 }$ to
a point $\bar{\mu}$
and the convergence rate of gradient norms follows from
Theorem~\ref{ThmGradSubanal}, and the stationary condition $\grad
\f(\bar{\mu}) = 0$ follows from Theorem~\ref{ThmGradient}.

\paragraph{Escaping strict saddle points:} 
Note that the functions $( \fone, \ftwo)$ are twice
continuously differentiable, and the
function $\fone$ is smooth.
Consequently, from Corollary~\ref{CorSaddlePoint},
it follows that with random initializations,
Algorithm~\ref{AlgoGradient} avoids strict 
saddle points almost surely.


\subsection{Proof of Corollary~\ref{CorBestSubset}}
\label{AppCorBestSubset}

We begin by providing a high-level outline of the proof.  First, note
that from Theorem~\ref{ThmProx}, we have the successive difference $\|
\x{k} - \x{k+1} \|_2 \rightarrow 0 $, and as a result, the set of limit point of
the sequnece $\big \{ \x{k} \big \}_{k \geq 0 } $---call it
$\mathcal{\bar{X}}$---is a connected
set~\cite{ostrowski2016solution}.
We prove that the connected-set $\mathcal{\bar{X}}$ is singleton by
showing that the set $\mathcal{\bar{X}}$ has an isolated point --- this 
also proves that sequence $\{ \x{k} \}_{k \geq 0}$ is convergent.
Next, we show that the objective-function $\f$, in the 
problem~\eqref{EqnLinRegDC},
satisfies Assumption KL with exponent $\theta = \frac{1}{2}$. Finally,
we show that condition $|\bar{x}|_{(r)} > |\bar{x}|_{(r+1)} \geq 0$
implies that function $x \mapsto \ftwo(x) \mydefn \sum_{i=d-s+1}^{d}
|x|_{(i)}$ is smooth in a neighborhood of point $\bar{x}$, and we use the
proof techniques of Theorem~\ref{ThmProxSubanal} to establish the
convergence rate of the gradient sequence. In order to obtain the rate
of convergence of the sequence $\big \{ \x{k} \big \}_{ k \geq 0 }$,
we use ideas similar to those in the paper~\cite{Jason-16}.


\paragraph{Convergence of the sequence $\{ \x{k} \}_{k \geq 0}$:}
For
notational convenience, let us use \mbox{$\fone(x) \mydefn \| y -
  \LinModMat x \|_2^2 $}, $\nonsmoothf(x) \mydefn \lambda \|x\|_1$,
and $\ftwo(x) \mydefn \lambda \sum_{i=d-s+1}^{d} |x|_{(i)}$. Since
 the point $\bar{x}$ satisfies the condition $|\bar{x}|_{(r)} >
|\bar{x}|_{(r+1)} \geq 0$ by assumption, there must exist a neighborhood
$B(\bar{x},r)$ such that the function $\ftwo$ is differentiable in the
neighborhood $B(\bar{x},r)$, and all points $x \in B(\bar{x},r)$ satisfy
$\sgn(x_{(i)}) = \sgn(\bar{x}_{(i)})$ for $1 \leq i \leq r$. We show that,
in a neighborhood of the point $\bar{x}$, it is the 
only critical point, thereby proving that
the point $\bar{x}$ is an isolated critical point.
To this end consider the  
convex sub-problem mentioned in Corollary~\ref{CorBestSubset}
\begin{align}
\label{ProbCvxSubprob}
\mathcal{P}(\bar{x}) \mydefn \min_{ x \in \real^{\usedim} } \fone(x) + \lambda \nonsmoothf(x) -
\lambda \inprod{\grad \ftwo(\bar{x})}{x}.
\end{align}
For any point $\xstar$ such that $\xstar \in B(\bar{x},r) \cap \mathcal{\bar{X}}$,
from Theorem~\ref{ThmProx}, we know that
\begin{align}
\label{EqnxAndxstarAreSoln}
\grad \fone(\bar{x}) + \lambda \bar{u} - \lambda \grad \ftwo(\bar{x}) = 0 \;\;\;
\mbox{and} \;\;\; \grad \fone(\x{*}) + \lambda u^{*} - \lambda \grad \ftwo(\x{*}) =
0,
\end{align}
where subgradients $u^* \in \subgrad \nonsmoothf(\x{*})$ and $\bar{u} \in
\subgrad \nonsmoothf(\bar{x}) $. Next, note that from the choice of neighborhood 
$B(\bar{x},r)$, it follows that for all $x \in B(\bar{x},r)$ we have
$\grad \ftwo(\x{}) = \grad \ftwo(\bar{x})$, and in particular, we deduce $\grad
\ftwo(\x{*}) = \grad \ftwo(\bar{\x{}})$. Combining this relation with 
equation~\eqref{EqnxAndxstarAreSoln} yields:
\begin{align*}
\grad \fone(\bar{x}) + \lambda \bar{u} - \lambda \grad \ftwo(\bar{x}) = 0 \;\;\;
\mbox{and} \;\;\; \grad \fone(\x{*}) + \lambda u^{*} - \lambda \grad \ftwo(\bar{x}) =
0,
\end{align*}
which implies both the points  $\x{*}$ and $\bar{x}$ are zero sub-gradient points
of convex problem~\eqref{ProbCvxSubprob}; this contradicts the assumption that
problem~\eqref{ProbCvxSubprob} has an unique solution. Hence, we conclude
that $\x{*} = \bar{x}$, and the point $\bar{x}$ is an isolated 
critical point of the sequence
$\big \{ \x{k} \big \}_{ k \geq 0 }$, 
and $\mathcal{\bar{X}}$. Putting together the pieces,
we conclude that
$\x{k} \rightarrow \bar{x}$.
\paragraph{Smoothness of function $\ftwo$ in a neighborhood of
  $\bar{x}$:}
We already argued above that for all $x \in B(\bar{x},r)$, the
function $\ftwo$ is differentiable and $\grad \ftwo(\x{}) =
\grad \ftwo(\bar{x})$. Consequently,  we have that in the neighborhood 
$B(\bar{x},r)$, the function $\ftwo$ is smooth with a
smoothness parameter
$\M_{\ftwo} = 0$.
\paragraph{The function $\f$ satisfies Assumption KL 
with exponent $\theta = \frac{1}{2}$:}
Recently, in the paper~\cite{li2016calculus} (Corollaries 5.1 and
5.2), the authors showed that if the functions $f_1,f_2,\ldots,f_T$
satisfy the KL-inequality with an exponent $\theta = \frac{1}{2}$, then
the function $f \mydefn \min \big \{ f_1, f_2, \ldots, f_T \big
\}$ also satisfies KL-inequality with the exponent $\theta = \frac{1}{2}$. 
Interestingly, the function $\f$ can be represented as is minimum
of finitely many functions as follows:
\begin{align}
\f(x) = \min_{a \in \mathcal{A}} \big \{ \| y - \LinModMat x
\|_2^2 + \lambda \|x\|_1 - \lambda a^\top x \big \},
\end{align}
where $\mathcal{A} \mydefn \big \{ a \in \big \{ -1,0,1 \big
\}^{\usedim} : \sum_{i=1}^{\usedim} |a_i| = r \big \} $. Note that
the set $\mathcal{A}$ has cardinality at most $3^{\usedim}$. It
is known that functions of the form $x \mapsto \frac{1}{2} x^\top A x
+ P(x) + b^\top x$ satisfy the KL-inequality with exponent $\theta =
\frac{1}{2}$, where $P$ is a proper closed polyhedral function, and 
$A$ is a positive semi-definite matrix; see
Corollaries 5.1 and 5.2 in the paper~\cite{li2016calculus}.
Putting together
these two observations, 
we conclude that the function $\f$
satisfies KL-assumption with KL-exponent
$\theta = \frac{1}{2}$.   
\paragraph{Combining the pieces:}
Since we proved $\x{k} \rightarrow \bar{x}$, we have that
for a suitable choice of $k_1$, the
tail sequence $\big \{ \x{k} \big \}_{ k \geq k_1 }$ 
lies in the neighborhood
$B(\bar{x},r)$.
Now, the function $\f$ satisfies Assumption KL
with exponent $\theta = \frac{1}{2}$,  and the function
$\ftwo$ is smooth in the neighborhood $B(\bar{x},r)$;
hence, following the argument in proof of 
Theorem~\ref{ThmProxSubanal} part(b), we conlcude that:
\begin{align*}
  \avg{\| \grad f (\x{k}) \|_2 } \leq  \frac{c_{1}}{k}.
\end{align*}
\paragraph{Rate of convergence of sequence $\big\{ \x{k}
  \big\}_{k \geq 0}$:} The KL-exponent for the
function $\f$ is $\theta = \frac{1}{2}$, and we may use $\gamma = 1$
in equation~\eqref{EqnSumFinite} which yields
\begin{align}
  \label{EqnSuccDiffBound}
   \sum_{\ell = k_1+1}^{\infty} \| \x{\ell} - \x{\ell+1} \|_2 & \leq
    \| \x{k_1} - \x{k_1+1} \|_2 + C_3 \big( \f(\x{k_1}) -
   \bar{\f} \big)^{\frac{1}{2}},
\end{align}
for some constant $C_3$. From Lemma~\ref{LemUnifKL_Ineq} and
equation~\eqref{EqnProxGradUB1}, we have
\begin{align}
\label{EqnFuncValBound}
\big( \f(\x{k_1}) - \bar{\f} \big)^{\frac{1}{2}} & \leq \C{} \|
\grad{\f(\x{k_1})} \|_2 \leq \C{} (\M + \M_{\ftwo} + \frac{1}{\alpha} )
\| \x{k_1} - \x{k_1-1} \|_2.
\end{align}
Combining equations~\eqref{EqnSuccDiffBound}
and~\eqref{EqnFuncValBound} we have
\begin{align}
\label{EqnFinalRecursion}  
\sum_{\ell = k_1}^{\infty} \| \x{\ell} - \x{\ell+1} \|_2 & \leq 2 \|
\x{k_1} - \x{k_1+1} \|_2 + C_3 \big( \f(\x{k_1}) - \bar{\f}
\big)^{\frac{1}{2}} \nonumber  \\
& \stackrel{(i)}{\leq} 2 \|
\x{k_1} - \x{k_1+1} \|_2  + \C{} C_3 (\M + \M_{\ftwo} + \frac{1}{\alpha} )
\| \x{k_1} - \x{k_1-1} \|_2 \nonumber \\
& \stackrel{(ii)}{\leq} \bar{\C{}} \| \x{k_1} - \x{k_1-1} \|_2,
\end{align}
where $\bar{\C{}}$ is a constant depending on $\M, \M_{\ftwo}, \alpha$, $C_3$ 
and $\C{}$, and step (i) above follows from equation~\eqref{EqnFuncValBound}. 
We justify step (ii) shortly, 
but let us first derive the linear rate of convergence of the sequence
$\big\{ \x{k} \big\}_{k \geq 0}$ using the derivation
in equation~\eqref{EqnFinalRecursion}. Denote 
$e_k = \sum_{\ell = k}^{\infty} \| \x{\ell} -
\x{\ell+1} \|_2$. Then equation~\eqref{EqnFinalRecursion} provides the
following recursion
\begin{align*}
e_{k_1} \leq \bar{\C{}}( e_{k_1-1} - e_{k_1} ).
\end{align*}
 Simple inspection of proof of Theorem~\ref{ThmProxSubanal} 
 and derivations so far
ensure that we can derive
the equations~\eqref{EqnSuccDiffBound} and~\eqref{EqnFuncValBound}
for all $k \geq k_1$; this provides us a recursion relation as above
with $k_1$ replaced by $k$. Furthermore, 
by choosing a larger value of the
constant $\bar{C}$ if necessary,
we may conclude that for all $k \geq 1$  we have
\begin{align*}
  e_{k} \leq \bar{\C{}}( e_{k-1} - e_{k}).  
\end{align*}
Rearranging the above inequality yields $e_{k} \leq
\frac{\bar{\C{}}}{\bar{\C{}}+1} e_{k-1}$, which guarantees
that the sequence $\big \{ e_{k} \big \}_{ k \geq 0 }$ converges
to zero at a linear rate. Finally,
observe that
$\| \x{k} - \x{*} \|_2 \leq \sum_{\ell=k}^{\infty} \| \x{\ell} -
\x{\ell+1} \|_2 = e_{k}$,
and
the linear rate of convergence of
the sequence $\big\{ \|\x{k} - \x{*} \|_2\big\}_{k \geq 0}$
to zero follows.



\paragraph{Justification for step (ii) in
  equation~\eqref{EqnFinalRecursion}:} Note that it suffices to show
that the object $\|\x{k_1} - \x{k_1+1} \|_2$ is upper bounded by a constant
multiple of $\| \x{k_1} - \x{k_1-1} \|_2 $, where the constant depends
only on $\M, \M_{\ftwo}, \alpha$ and $\C{}$. Recalling the decent
property proved in equation~\eqref{EqnProxDescentStep} we have:
\begin{align}
  \label{EqnUpperBound}
  \big( \f(\x{k_1}) - \bar{\f} \big)^{\frac{1}{2}} \geq \big( \f(\x{k_1}) - \f(\x{k_1+1}) \big)^{\frac{1}{2}}
   \geq \frac{1}{\sqrt{2\alpha}} \|\x{k_1} - \x{k_1+1} \|_{2}.
\end{align}
Combining equations~\eqref{EqnUpperBound} and~\eqref{EqnFuncValBound} we obtain the following upper and lower bound of 
$\big( \f(\x{k_1}) - \bar{\f} \big)^{\frac{1}{2}}$:
\begin{align*}
   \frac{1}{\sqrt{2\alpha}} \|\x{k_1} - \x{k_1+1} \|_2 \leq \big( \f(\x{k_1}) - \bar{\f} \big)^{\frac{1}{2}}  \leq 
    \C{} (\M + \M_{\ftwo} + 1/\alpha)\| \x{k_1} - \x{k_1-1} \|_2.
\end{align*}  
Rearranging the last equality proves the desired upper bound.
Finally, we reiterate that the above justification also 
hold for any iterate $k$ with $k\geq k_1$. 


\subsection{Proof of Corollary~\ref{CorMixDensity}}
\label{AppCorMixDensity}
The proof of this corollary is based on application 
of Theorems~\ref{ThmProx} and~\ref{ThmProxSubanal}.
We verify the assumptions of Theorems~\ref{ThmProx}
and~\ref{ThmProxSubanal}
with $\fone(\theta) = - \sum \limits_{i = 1 }^n
\log \big( \mixFun( y_i; \theta)\big)$, $\ftwo \equiv 0$, $\nonsmoothf =
\mathbbm{1}_{\Xspace}$ 
and function $\f \mydefn \fone - \nonsmoothf +
\ftwo$. Note that the
domain $\dom(\f) = \mathcal{X}$ is compact, which guarantees
that the iterate sequence $\{ \theta^{k} \}_{k \geq 0}$
obtained from Algorithm~\ref{AlgoProx} is
bounded. The function $\ftwo \equiv 0$ is smooth. 
The log-partition function $A$ is twice continuously 
differentiable by assumption, which
guarantees that the function $\fone$ is also twice
continuously differentiable, whence smooth in the compact domain
$\Xspace$. Finally, we verify that the function $\f$ satisfies Assumption
KL by proving that $\f$ is continuous sub-analytic in
its domain $\Xspace$, and the domain $\mathcal{X}$ is closed;
see Lemma~\ref{LemSubanalKl2}.  Clearly,
$\dom(\f) = \Xspace$ is closed, and the function $\f$ is continuous in
$\dom(\f)$. Finally, we show that the
functions $(\fone, \nonsmoothf)$ are
sub-analytic, and  invoking the property (d) of
sub-analytic functions form Appendix~\ref{AppSubanal},
we conclude that the 
function $\f \mydefn \fone + \nonsmoothf$
is sub-analytic.

\paragraph{The function $\nonsmoothf$  is
  sub-analytic:} 
  Here, we use a simple result by Attouch
et al.~\cite{attouch2010proximal}, which states that 
the indicator function of
a semi-algebraic set is a semi-algebraic 
function (hence a sub-analytic function). In order to show that the set
$\Xspace$ is semi-algebraic, we note the following
representation of the set $\Xspace$
\begin{align}
  \label{EqnMixDensitySetIsAlgebraic}
  \Xspace = \big \{ \sum\limits_{ i = 1}^{ \usedim} \theta_i^2 > R_1^2 \big
  \}^{c} \bigcap \big \{ \sum\limits_{ i = \usedim + 1 }^{2 \usedim}
  \theta_i^2 > R_2^2 \big \}^{c} \bigcap \big \{ \theta_{2\usedim+1} > 1 \big
  \}^c \bigcap \big \{ - \theta_{2\usedim+1} > 0 \big \}^c.
\end{align}
Each of the four sets in
representation~\eqref{EqnMixDensitySetIsAlgebraic} are semi-algebraic
by definition, and semi-algebraic sets are closed under finite
intersection and complements; see the book by
Coste~\cite{coste2002introduction}. 
Putting together these two observations, we
conclude that the set $\Xspace$ is semi-algebraic, 
and that $\mathbbm{1}_{\Xspace}$ is a sub-analytic function.

\paragraph{The function $\fone$ is sub-analytic:}
The log-partition function $A$ is sub-analytic by assumption.
For a fixed vector $y$, the
map $\eta \mapsto \eta^\top T(y)$ is linear, and hence sub-analytic.
Since sub-analytic
functions are closed under a finite linear combination, we conclude
that the map $\eta
\mapsto \eta^\top T(y) - A(\eta)$ is sub-analytic. Continuous
sub-analytic functions are closed under multiplication and
composition; since the $\exp(\cdot)$ function is continuous sub-analytic,
we have for every fixed vector $y$ the following map
\begin{align*}
 (\eta_0, \eta_1, p) \mapsto \mixFun(y; \eta_0, \eta_1, p) \mydefn p\exp(
\eta_0^\top T(y) - A(\eta_0) ) + (1-p)\exp( \eta_1^\top T(y) -
A(\eta_1) )
  \end{align*}
is sub-analytic. Furthermore, the $\log(\cdot)$ 
function analytic on the interval $(0, \infty)$, and
using the composition rule for continuous sub-analytic
functions, we obtain that the map 
${ \theta \mapsto \log(\mixFun(y_i;\theta)) }$ is
sub-analytic, where $\theta \mydefn (\eta_0,\eta_1,p)$. 
Finally, the target function $\fone$ is a linear
combination of sub-analytic 
functions $\log(\mixFun(y_i;\theta))$, and we conlcude that
the map $\theta \mapsto \fone(\theta)$ is sub-analytic.
\paragraph{Combining the pieces:} Putting together the pieces, we
conclude that the function $\f$ is
sub-analytic, with the function $\f$ being  continuous in $\dom(\f)$,
whereas $\dom(\f)$ is closed;
furthermore, the functions $\fone$ and $\ftwo$ are smooth.
This allows us to apply
Theorem~\ref{ThmProx} and Theorem~\ref{ThmProxSubanal} and the
corollary follows.

\paragraph{Sub-analyticity of the log-partititon functions $A$ in Table 1:}
  \label{AppTabMixDensity}
The sub-analyticity of the log-partition function $A$ mentioned in
Table 1 follows from the following two observations. First, note that
the functions $\exp, \text{ln}$ and $\Gamma$ are continuous and
analytic (hence sub-analytic).  Given two continuous sub-analytic
functions $g_1$ and $g_2$, the composition function $g_2 \circ g_1$ is also
continuous sub-analytic. Secondly, any linear combination of
sub-analytic functions is also sub-analytic function. See
Appendix~\ref{AppSubanal} for properties of sub-analytic functions.


\section{Characterizing ``smooth - convex" function class}
In Theorem~\ref{ThmGradient} and Theorem~\ref{ThmProx} we discussed a
class of non-smooth non-convex functions, where a gradient or a prox-type
algorithm provides satisfactory convergence to a critical point. One
possible deficiency of the theory discussed so far is that, in
Algorithm~\ref{AlgoGradient} (respectively Algorithm~\ref{AlgoProx}),
we need to specify a decomposition of the objective function $\f$ as
a difference of a smooth and a convex function (respectively, smooth +
convex - convex). Consequently, it is natural to wonder if we can
characterize the class of functions which has a decomposition needed
in Algorithms~\ref{AlgoGradient} and~\ref{AlgoProx}. Furthermore, if
a function has this a decomposition, how can we obtain such a
decomposition easily. It is worth pointing out that for the case of
Algorithm~\ref{AlgoProx}, the convex function $\nonsmoothf$ is known
in many cases. For instance, in the case of constrained optimization,
the function $\nonsmoothf$ is the indicator of the constraint
set; in many statistical estimation problems, $\nonsmoothf$ is a penalty
function on the parameters; a well-known example of such penalty
function is the $\ell_1$ penalty, which is used to obtain sparse
solutions. Hence, for all practical purposes, 
the task of characterizing
the function class mentioned in Theorems~\ref{ThmGradient} and
~\ref{AlgoProx} reduces to characterizing functions which can be
decomposed as a difference of a smooth function($\fone$) and a convex
function ($\ftwo$). In the next theorem, we characterize the class 
of of continuously differentiable functions 
that can  be written as a difference of a smooth function and a 
convex function.
\begin{theos}
\label{ThmNecSuffCond}
Given any continuously differentiable function $\f: \real^{\usedim} \mapsto \real$,  
the following two properties are equivalent:
\begin{enumerate}
    \item[(a)] There exists a $\M$-smooth function $\fone$, and 
    a convex continuously differentiable function $\ftwo$ such that:
    \begin{align*}
       \f(x) = \fone(x) - \ftwo(x) \quad \text{for all}\;\; x \in \real^{\usedim}.
    \end{align*}
    \item[(b)] The gradient of the function $\f$ satisfies the following inequality:
    \begin{align*}
           \big\langle \grad \f(x) - \grad \f(y), x - y  \big\rangle  
           \leq \M \| x - y \|^2 \quad \text{for all} \;\; x,y \in \real^{\usedim}. 
    \end{align*}       
  \end{enumerate}  
\end{theos}

\begin{proof}
We establish the 
equivalence by proving the circle of 
implications $(a) \Longrightarrow (b) \Longrightarrow (a)$.

\paragraph{Implication (a) $\Longrightarrow$ (b):}
For any $\M$-smooth function $\fone$, we have the following:
\begin{align}
   \big\langle \grad \fone(x) - \fone(y), x - y \big\rangle 
   &\leq \| \grad \fone(x) - \fone(y) \|_2 
   \times \| x- y\|_2 \nonumber \\
   & \stackrel{(i)}{\leq} \M\|x - y\|_2^2, 
   \quad \text{for all} \;\; x,y \in \real^{\usedim}, 
   \label{EqnMSmooth}
 \end{align} 
 where step (i) follows since the gradient $\grad \fone$ is $M$ Lipschitz.
 Next note that the gradient of a differentiable convex 
 function is a monotone
 operator, and we have that for all $x,y \in \real^{\usedim}$:
 \begin{align}
    \big\langle\grad \ftwo(x) - \grad \ftwo(y), x -y\big\rangle
    \geq 0.
    \label{EqnMonotone}
  \end{align} 
Subtracting equation~\eqref{EqnMonotone} 
from equation~\eqref{EqnMSmooth}, we obtain the desired 
upper bound in part (b).

\paragraph{Implication (b) $\Longrightarrow$ (a):} We prove this implication 
by finding a $\M$--smooth function $\fone$ and a convex differentiable
function $\ftwo$ such that $\f = \fone - \ftwo$. To this end, we fix
any $x_0 \in \real^{\usedim}$ and consider the following two
functions:
\begin{subequations}
  \begin{align}
       \fone(x) & \mydefn f(x_0) + \big\langle \grad \f(x_0), x - x_0\big\rangle
        + \frac{\M}{2}\|x - x_0\|_2^2 
        \label{EqnDefFone} \\
        \ftwo(x) & \mydefn \fone(x) - \f(x).
    \label{EqnDefFtwo}
  \end{align}
\end{subequations}
The function $\fone$ in definition~\eqref{EqnDefFone} 
is $\M$-smooth by inspection. Since 
both the functions $\f$ and $\fone$ are continuously differentiable,
the function $\ftwo$ is continuously differentiable by construction.
In order to complete the proof, it suffices to show that
the function $\ftwo$ is convex. To this end, the first order Taylor series expansion of the function $\ftwo$ yields
\begin{align}
     \ftwo(x) &= \ftwo(y) + \big\langle \grad \ftwo(y + t(x-y)), x - y\big\rangle 
     \quad \text{for some}~t~\in [0,1] \nonumber \\
      & = \ftwo(y) + \big\langle \grad \ftwo(y),x - y\big\rangle
      + \big\langle \grad \ftwo(y + t(x - y)) - \grad \ftwo(y), x - y\big\rangle.
      \label{EqnTaylor}
\end{align} 
Expanding the term
${\big\langle \grad \ftwo(y + t(x - y)) - \grad \ftwo(y), 
x - y\big\rangle}$ above yields,
\begin{align*}
  \big\langle \grad \ftwo(y + t(x - y)) - \grad \ftwo(y), x -
  y\big\rangle & \stackrel{(i)}{=} \M \|x - y\|_2^2 - \frac{\big\langle
  \grad \f(y + t(x - y)) - \grad \f(y), t(x - y)\big\rangle}{t} \\
& \stackrel{(ii)}{\geq} \M  \|x - y\|_2^2 - \M t \|x - y\|_2^2 \\
& \stackrel{(iii)}{\geq} 0.
\end{align*}
Here step (i) follows by substituting the expression of the function
$h$; step (ii) follows from the gradient inequality of part (b), and
step (iii) follows from the inequality 
$0 \leq t \leq 1$. Since the vectors $x,y \in \real^{\usedim}$ 
were arbitrary, the 
inequality $\big\langle \grad \ftwo(y + t(x - y)) - \grad \ftwo(y), 
x - y\big\rangle \geq 0$  combined with equation~\eqref{EqnTaylor} proves the convexity of the function $\ftwo$, thereby proving the claimed
result in \mbox{part (a)}.

\end{proof}

\paragraph{Comments:}
It would be interesting to characterize the class of
DC-based functions mentioned in problem~\eqref{ProbGradient} when the
convex function $\ftwo$ is non-differentiable. Indeed,
we obtain a larger and more interesting non-differentiable class of functions.
It would interesting to see whether Theorem~\ref{ThmNecSuffCond} can be
suitably generalized in this setting.




\bibliography{koulikpaper}

\begin{thebibliography}{10}

\bibitem{ahmadi2013np}
A.~A. Ahmadi, A.~Olshevsky, P.~A. Parrilo, and J.~N. Tsitsiklis.
\newblock Np-hardness of deciding convexity of quartic polynomials and related
  problems.
\newblock {\em Mathematical Programming}, 137(1-2):453--476, 2013.

\bibitem{ahmadi2013complete}
A.~A. Ahmadi and P.~A. Parrilo.
\newblock A complete characterization of the gap between convexity and
  sos-convexity.
\newblock {\em SIAM Journal on Optimization}, 23(2):811--833, 2013.

\bibitem{an2017convergence}
N.~T. An and N.~M. Nam.
\newblock Convergence analysis of a proximal point algorithm for minimizing
  differences of functions.
\newblock {\em Optimization}, 66(1):129--147, 2017.

\bibitem{attouch2010proximal}
H.~Attouch, J.~Bolte, P.~Redont, and A.~Soubeyran.
\newblock Proximal alternating minimization and projection methods for
  nonconvex problems: An approach based on the {K}urdyka-{L}ojasiewicz
  inequality.
\newblock {\em Mathematics of Operations Research}, 35(2):438--457, 2010.

\bibitem{bolte2007Lojasiewicz}
J.~Bolte, A.~Daniilidis, and A.~Lewis.
\newblock The {L}ojasiewicz inequality for nonsmooth subanalytic functions with
  applications to subgradient dynamical systems.
\newblock {\em SIAM Journal on Optimization}, 17(4):1205--1223, 2007.

\bibitem{bolte2014proximal}
J.~Bolte, S.~Sabach, and M.~Teboulle.
\newblock Proximal alternating linearized minimization for nonconvex and
  nonsmooth problems.
\newblock {\em Mathematical Programming}, 146(1-2):459--494, 2014.

\bibitem{boyd2004convex}
S.~Boyd and L.~Vandenberghe.
\newblock {\em Convex optimization}.
\newblock Cambridge university press, 2004.

\bibitem{bubeck2015convex}
S.~Bubeck et~al.
\newblock Convex optimization: {A}lgorithms and complexity.
\newblock {\em Foundations and Trends{\textregistered} in Machine Learning},
  8(3-4):231--357, 2015.

\bibitem{carmon2017lower1}
Y.~Carmon, J.~C. Duchi, O.~Hinder, and A.~Sidford.
\newblock Lower bounds for finding stationary points i.
\newblock {\em arXiv preprint arXiv:1710.11606}, 2017.

\bibitem{cartis2010complexity}
C.~Cartis, N.~I. Gould, and P.~L. Toint.
\newblock On the complexity of steepest descent, {N}ewton's and regularized
  {N}ewton's methods for nonconvex unconstrained optimization problems.
\newblock {\em Siam journal on optimization}, 20(6):2833--2852, 2010.

\bibitem{coste2002introduction}
M.~Coste.
\newblock An introduction to semialgebraic geometry.
\newblock {\em RAAG network school}, 145:30, 2002.

\bibitem{ecker2010polynomial}
A.~Ecker and A.~D. Jepson.
\newblock Polynomial shape from shading.
\newblock In {\em Computer Vision and Pattern Recognition (CVPR), 2010 IEEE
  Conference on}, pages 145--152. IEEE, 2010.

\bibitem{facchinei2007finite}
F.~Facchinei and J.-S. Pang.
\newblock {\em Finite-dimensional variational inequalities and complementarity
  problems}.
\newblock Springer Science \& Business Media, 2007.

\bibitem{FanLi91}
J.~Fan and R.~Li.
\newblock Variable selection via non-concave penalized likelihood and its
  oracle properties.
\newblock {\em Jour. Amer. Stat. Ass.}, 96(456):1348--1360, December 2001.

\bibitem{ge2017no}
R.~Ge, C.~Jin, and Y.~Zheng.
\newblock No spurious local minima in nonconvex low rank problems: {A} unified
  geometric analysis.
\newblock {\em arXiv preprint arXiv:1704.00708}, 2017.

\bibitem{gotoh2017dc}
J.-Y. Gotoh, A.~Takeda, and K.~Tono.
\newblock {DC} formulations and algorithms for sparse optimization problems.
\newblock {\em Mathematical Programming}, pages 1--36, 2017.

\bibitem{hartman-59}
P.~Hartman.
\newblock On functions representable as a difference of convex functions.
\newblock {\em Pacific Journal of Mathematics}, 9(3):707--713, 1959.

\bibitem{hong2016convergence}
M.~Hong, Z.-Q. Luo, and M.~Razaviyayn.
\newblock Convergence analysis of alternating direction method of multipliers
  for a family of nonconvex problems.
\newblock {\em SIAM Journal on Optimization}, 26(1):337--364, 2016.

\bibitem{Horst00}
R.~Horst, P.~M. Pardalos, and N.~V. Thoai.
\newblock {\em Introduction to {G}lobal {O}ptimization}, volume~9 of {\em
  Nonconvex {O}ptimization and its {A}pplications}.
\newblock Elsevier, 2000.

\bibitem{jaggi-11}
M.~Jaggi.
\newblock Revisiting {F}rank-{W}olfe: Projection-free sparse convex
  optimization.
\newblock In {\em ICML (1)}, pages 427--435, 2013.

\bibitem{kurdyka1998gradients}
K.~Kurdyka.
\newblock On gradients of functions definable in o-minimal structures.
\newblock In {\em Annales de l'institut Fourier}, volume~48, pages 769--784.
  Chartres: L'Institut, 1950-, 1998.

\bibitem{Julien16}
S.~Lacoste-Julien.
\newblock Convergence rate of {F}rank-{W}olfe for non-convex objectives.
\newblock {\em arXiv preprint arXiv:1607.00345}, 2016.

\bibitem{Lanckreit09}
G.~R. Lanckriet and B.~K. Sriperumbudur.
\newblock On the convergence of the concave-convex procedure.
\newblock In {\em Advances in neural information processing systems}, pages
  1759--1767, 2009.

\bibitem{Jason-16}
J.~D. Lee, M.~Simchowitz, M.~I. Jordan, and B.~Recht.
\newblock Gradient descent only converges to minimizers.
\newblock In {\em Conference on Learning Theory}, pages 1246--1257, 2016.

\bibitem{li2016calculus}
G.~Li and T.~K. Pong.
\newblock Calculus of the exponent of {K}urdya-{L}ojasiewicz inequality and its
  applications to linear convergence of first-order methods.
\newblock {\em arXiv preprint arXiv:1602.02915}, 2016.

\bibitem{li2017convergence}
Y.~Li and Y.~Yuan.
\newblock Convergence analysis of two-layer neural networks with relu
  activation.
\newblock In {\em Advances in Neural Information Processing Systems}, pages
  597--607, 2017.

\bibitem{lipp2016variations}
T.~Lipp and S.~Boyd.
\newblock Variations and extension of the convex--concave procedure.
\newblock {\em Optimization and Engineering}, 17(2):263--287, 2016.

\bibitem{loh2013regularized}
P.-L. Loh and M.~J. Wainwright.
\newblock Regularized {M}-estimators with nonconvexity: {S}tatistical and
  algorithmic theory for local optima.
\newblock In {\em Advances in Neural Information Processing Systems}, pages
  476--484, 2013.

\bibitem{Lojasiewicz1963propriete}
S.~Lojasiewicz.
\newblock Une propri{\'e}t{\'e} topologique des sous-ensembles analytiques
  r{\'e}els.
\newblock {\em Les {\'e}quations aux d{\'e}riv{\'e}es partielles}, 117:87--89,
  1963.

\bibitem{nesterov2006cubic}
Y.~Nesterov and B.~T. Polyak.
\newblock Cubic regularization of {N}ewton method and its global performance.
\newblock {\em Mathematical Programming}, 108(1):177--205, 2006.

\bibitem{ostrowski2016solution}
A.~M. Ostrowski.
\newblock {\em Solution of equations and systems of equations: Pure and applied
  mathematics: A Series of monographs and textbooks}, volume~9.
\newblock Elsevier, 2016.

\bibitem{panageas2016gradient}
I.~Panageas and G.~Piliouras.
\newblock Gradient descent only converges to minimizers: Non-isolated critical
  points and invariant regions.
\newblock {\em arXiv preprint arXiv:1605.00405}, 2016.

\bibitem{parikh2014proximal}
N.~Parikh, S.~Boyd, et~al.
\newblock Proximal algorithms.
\newblock {\em Foundations and Trends{\textregistered} in Optimization},
  1(3):127--239, 2014.

\bibitem{Pham-13}
T.~Pham~Dinh, H.~Ngai, and H.~Le~Thi.
\newblock Convergence analysis of the {D}{C} algorithm for {D}{C} programming
  with subanalytic data.
\newblock {\em preprint}, 2013.

\bibitem{rockafellar2009variational}
R.~T. Rockafellar and R.~J.-B. Wets.
\newblock {\em Variational analysis}, volume 317.
\newblock Springer Science \& Business Media, 2009.

\bibitem{Tuy95}
H.~Tuy.
\newblock {DC} optimization: theory, methods and algorithms.
\newblock In {\em Handbook of global optimization}, pages 149--216. Springer,
  1995.

\bibitem{wang2014efficient}
S.~Wang, A.~Schwing, and R.~Urtasun.
\newblock Efficient inference of continuous markov random fields with
  polynomial potentials.
\newblock In {\em Advances in neural information processing systems}, pages
  936--944, 2014.

\bibitem{wen2018proximal}
B.~Wen, X.~Chen, and T.~K. Pong.
\newblock A proximal difference-of-convex algorithm with extrapolation.
\newblock {\em Computational Optimization and Applications}, 69(2):297--324,
  2018.

\bibitem{xu2017globally}
Y.~Xu and W.~Yin.
\newblock A globally convergent algorithm for nonconvex optimization based on
  block coordinate update.
\newblock {\em Journal of Scientific Computing}, 72(2):700--734, 2017.

\bibitem{yan2017convergence}
B.~Yan, M.~Yin, and P.~Sarkar.
\newblock Convergence of gradient em on multi-component mixture of gaussians.
\newblock In {\em Advances in Neural Information Processing Systems}, pages
  6959--6969, 2017.

\bibitem{Yuille03}
A.~L. Yuille and A.~Rangarajan.
\newblock The concave-convex procedure.
\newblock {\em Neural computation}, 15(4):915--936, 2003.

\bibitem{zhang1999shape}
R.~Zhang, P.-S. Tsai, J.~E. Cryer, and M.~Shah.
\newblock Shape-from-shading: a survey.
\newblock {\em IEEE transactions on pattern analysis and machine intelligence},
  21(8):690--706, 1999.

\end{thebibliography}
\bibliographystyle{abbrv}


\end{document}